\documentclass{article}

\usepackage[main,final]{neurips_2026}

\usepackage[utf8]{inputenc}
\usepackage[T1]{fontenc}
\usepackage{url}
\usepackage{microtype}
\usepackage{graphicx}
\usepackage{subcaption}
\usepackage{booktabs} % for professional tables
\usepackage{makecell}
\usepackage{multirow}
\usepackage{colortbl}
\usepackage{xcolor}

\renewcommand{\answerYes}[1][]{[Yes]#1}
\renewcommand{\answerNo}[1][]{[No]#1}
\renewcommand{\answerNA}[1][]{[N/A]#1}
\definecolor{myyellow}{HTML}{FFCB64}
\definecolor{myteal}{HTML}{62D4C5}
\definecolor{myblue}{HTML}{3266E3}
\definecolor{mypink}{HTML}{FE839C}
\usepackage[most]{tcolorbox}

\usepackage{amssymb}
\usepackage{hyperref}
\hypersetup{
  pdftitle={Preference-Guided Adversarial Policy Optimization for Long-Tail Robust Driving},
  pdfauthor={Tong Nie, Yihong Tang, Junlin He, Yuewen Mei, Jie Sun, Lijun Sun, Wei Ma, Jian Sun}
}
\usepackage{algorithm}
\usepackage{algorithmic}

\usepackage{amsmath}
\usepackage{amssymb}
\usepackage{mathtools}
\usepackage{amsthm}
\usepackage{enumitem}
\usepackage{titletoc}
\usepackage{wrapfig}
\usepackage{needspace}
\usepackage[capitalize,noabbrev]{cleveref}

\theoremstyle{plain}
\newtheorem{theorem}{Theorem}[section]

\newtheorem{lemma}[theorem]{Lemma}

\theoremstyle{definition}

\theoremstyle{remark}
\newtheorem{remark}[theorem]{Remark}

\usepackage[textsize=tiny]{todonotes}

\title{Preference-Guided Adversarial Policy Optimization for Long-Tail Robust Driving}

\author{%
  \textbf{Tong Nie}$^{1,2}$\thanks{Equal contribution.}\quad
  \textbf{Yihong Tang}$^{3}$\footnotemark[1]\quad
  \textbf{Junlin He}$^{1}$\quad
  \textbf{Yuewen Mei}$^{2}$\\
  \textbf{Jie Sun}$^{2}$\quad
  \textbf{Lijun Sun}$^{3}$\quad
  \textbf{Wei Ma}$^{1}$\thanks{Corresponding authors: Wei Ma (\texttt{wei.w.ma@polyu.edu.hk}) and Jian Sun (\texttt{sunjian@tongji.edu.cn}).}\quad
  \textbf{Jian Sun}$^{2}$\footnotemark[2]\\[0.5em]
  $^{1}$The Hong Kong Polytechnic University, Hong Kong SAR, China\\
  $^{2}$Tongji University, Shanghai, China\\
  $^{3}$McGill University, Montreal, QC, Canada
}

\begin{document}

\maketitle

\begin{abstract}
Deploying autonomous driving systems requires robustness against long-tail scenarios that are rare but safety-critical. While adversarial training offers a promising solution, existing methods typically decouple scenario generation from policy optimization and rely on heuristic surrogates. This leads to objective misalignment and fails to capture the shifting failure modes of evolving policies. This paper presents \texttt{ADV-0}, a closed-loop min-max optimization framework that treats the interaction between driving policy (defender) and adversarial agent (attacker) as a zero-sum Markov game. 
By aligning the attacker's utility directly with the defender's objective, we reveal the optimal adversary distribution. To make this tractable, we cast dynamic adversary evolution as iterative preference learning, efficiently approximating this optimum and offering an algorithm-agnostic solution to the game.
Theoretically, our analysis characterizes the idealized regularized game and derives a lower bound that connects adversarial training performance with real-world long-tail generalization.
Experiments indicate that \texttt{ADV-0} effectively exposes diverse safety-critical failures and greatly enhances the generalizability of both learned policies and motion planners against unseen long-tail risks.
\end{abstract}

%%%%%%%%%%%%%%%%%%%%%%%%%%%%%%%%%%%%%%%%%%%%%%%%%%
\vspace{-18pt}
\section{Introduction}
\vspace{-10pt}

Deploying autonomous driving (AD) systems in the open world faces a crucial bottleneck: the inability to anticipate and handle rare but safety-critical long-tail scenarios. 
While massive naturalistic driving data provide a basis for model development, they are dominated by normal logs where high-risk events like aggressive cut-ins appear with negligible frequency \citep{liu2024curse,xu2025wod}. 
The growing literature has introduced sampling and generative methods to accelerate the discovery of rare events \citep{feng2023dense, ding2023survey}.
However, they are often confined to stress testing or performance validation, failing to actively target long-tail risks. 
Effectively leveraging these synthetic data to improve the generalizability of AD policies in the long tail remains an open question.

Closed-loop adversarial training offers an avenue to address this challenge by exposing the training policy to synthetic risks. 
This paradigm can naturally be formulated as a min-max bi-level optimization problem via a zero-sum Markov game, involving an adversary that generates challenges and a defender that optimizes the policy \citep{pinto2017robust}. 
Despite its theoretical elegance, direct application in AD and robotics has been hindered by nontrivial optimization issues. 
First, end-to-end solutions via gradient descent are often computationally intractable due to the non-differentiable nature of physical simulators and the difficulty of propagating gradients through long-horizon rollouts to provide learning signals. 
Second, the zero-sum interaction between two players is prone to instability and mode collapse \citep{zhang2020stability}, where the adversary converges to unrealistic attack patterns, limiting the scalability of this framework.

To bypass computational difficulties, existing methods typically decouple the min-max objective into separate sub-problems: 
% generating a static set of adversarial scenarios and then training the agent against this fixed distribution \cite{zhang2023cat,zhang2024chatscene,stoler2025seal}. 
generating scenarios via fixed priors or heuristics, then training the agent against this static distribution \citep{zhang2023cat,zhang2024chatscene,stoler2025seal}.
However, this decoupled paradigm introduces notable limitations: 
(1) \textit{Misaligned}.
It creates a misalignment between the goals of the two players. While the defender optimizes a comprehensive reward for safety, efficiency, and comfort, the attacker solely targets collisions, relying on heuristic surrogates such as collision probability.
First, this discrepancy renders the adversarial objective ill-defined, often resulting in an overly aggressive attacker that overwhelms the defender and destabilizes training \citep{zhang2020stability}. Second, the divergence in gradient directions prevents the attacker from identifying non-collision failures like off-road violations and providing meaningful learning signals. Thus, the defender can overfit to specific collision modes while remaining vulnerable to broader risks \citep{vinitsky2020robust}, without acquiring generalized robustness. 
(2) \textit{Nonstationary}.
Decoupled methods with fixed attack modes fail to uncover the nonstationary vulnerability frontier---the shifting boundary of scenarios where the current policy remains prone to failure. As the defender evolves, its failure modes shift further into the long-tail distribution, becoming increasingly rare under the initial prior. Fixed adversaries with static priors are insufficient to track these shifting weaknesses, causing fragile generalization to unseen risks. 
(3) \textit{Uncertified}.
% Heuristic priors are often followed with black-box training, which relies on empirical trial-and-error, and thus fails to provide theoretical safety guarantees. 
Training against such static priors or heuristics is an empirical trial-and-error process, lacking theoretical safety guarantees. 
This lack of certified performance bounds contrasts sharply with the rigorous safety demands of real-world deployment, where agents must generalize to unbounded unknown long-tail scenarios \citep{brunke2022safe}.

This work revisits the min-max formulation and introduces \texttt{ADV-0}, a closed-loop policy optimization framework that enables end-to-end training for generalizable adversarial learning. 
\texttt{ADV-0} solves the Markov game by directly aligning the adversarial utility with the training objective of the defender. To address tractability and stability issues, we propose an online iterative preference learning algorithm, casting adversarial evolution as preference optimization. This allows the attacker to continuously track the shifting vulnerability frontier of the improving defender.
By coupling the evolution, \texttt{ADV-0} tailors the distribution towards the long tail of the current players, forcing the defender to learn generalized robustness rather than overfitting to heuristics.
Importantly, \texttt{ADV-0} is algorithm-agnostic, applicable to both RL agents and motion planning models, providing a theoretically grounded pathway from adversarial generation to policy improvement.
Our contributions are threefold:
\begin{itemize}[leftmargin=*, itemsep=-2pt, topsep=-5pt] 
    \item \texttt{ADV-0} is a closed-loop training framework for long-tail problems of AD that couples adversarial generation and policy optimization in an end-to-end way. 
    \item We propose a preference-based solution to the zero-sum game, a stable and efficient realization supporting on-policy interaction and algorithm-agnostic evolution.
    \item Theoretically, we characterize the idealized game and establish a lower bound that links adversarial training to real-world long-tail performance. Empirically, \texttt{ADV-0} not only exposes diverse safety-critical events but also enhances the generalizability of policies against unseen long-tail risks. 
\end{itemize}

%%%%%%%%%%%%%%%%%%%%%%%%%%%%%%%%%%%%%%%%%%%%%%%%%%%%%%%%%%%%%%%
\vspace{-10pt}
\section{Preliminary and Problem Formulation}
\label{sec:formulation}

\vspace{-10pt}
\paragraph{Task description.}
We model the safe AD task as a Markov Decision Process (MDP), defined by $(\mathcal{S}, \mathcal{A}, \mathcal{P}, \mathcal{R}, \gamma, T)$. Here, $\mathcal{S}$ and $\mathcal{A}$ denote the state and action spaces.
The state $s_t \in \mathcal{S}$ includes raw sensor inputs, high-level commands, and kinematic status. The action $a_t \in \mathcal{A}$ represents continuous low-level control signals (e.g., steering, acceleration). 
% The environment dynamics $\mathcal{P}:\mathcal{S} \times \mathcal{A} \rightarrow \mathcal{S}$ describe the transition probabilities of the traffic scene. 
The environment dynamics $\mathcal{P} : \mathcal{S} \times \mathcal{A} \rightarrow \Delta(\mathcal{S})$ describe the transition probabilities of the traffic scene.
We initialize scenarios using real-world driving logs, where the ego vehicle is controlled by a policy $\pi_{\theta}$ parameterized by $\theta$. Background traffic participants are initially governed by naturalistic behavior priors, such as log-replay or traffic models like Intelligent Driver Model (IDM). 
The reward function $\mathcal{R}(s, a)$ is designed to balance task progress with safety: it encourages route completion and velocity tracking, while imposing heavy penalties for safety violations, such as collisions or off-road events. 
The goal of the ego agent is to learn an optimal policy $\pi_{\theta}^*$ that maximizes the expected cumulative return $J(\pi_\theta) = \mathbb{E}_{\tau \sim \pi_\theta, \mathcal{P}}[\sum_{t=0}^{T} \gamma^t \mathcal{R}(s_t, a_t)]$ within $T$. Unlike imitation learning which assumes a fixed data distribution, we focus on online RL to enable the agent to recover from adversarial perturbations in closed-loop interactions.

\vspace{-10pt}
\paragraph{Min-max formulation.}
Relying solely on naturalistic scenarios often fails to expose the ego agent to low-probability but high-risk events residing in the long tail of the distribution. 
To ensure the robustness of the policy against long-tail risks, adversarial training frames the problem as a robust optimization task via a \textbf{two-player zero-sum game} between the ego agent and an adversary. 
Here, the behaviors of background agents are governed by a parameterized adversarial policy $\psi \in \Psi$ altering transition dynamics from naturalistic $\mathcal{P}$ to adversarial $\mathcal{P}_{\psi}$. 
This yields the min-max objective:
% The robust policy optimization is thus cast as a min-max objective:
% \vspace{-5pt}
\begin{equation}
    \label{eq:min_max}
    \max_{\theta} \min_{\psi \in \Psi} \left[J(\pi_\theta, \mathcal{P}_{\psi}) =  \mathbb{E}_{\tau \sim (\pi_\theta, \mathcal{P}_{\psi})}  [\sum_{t=0}^T \gamma^t \mathcal{R}(s_t, a_t)] \right],
\end{equation}
where $\Psi$ represents the feasible set of adversarial configurations that remain physically plausible. The outer loop maximizes the ego's performance, while the inner loop seeks an adversary minimizing the current ego's reward. 
Directly optimizing the bi-level objective via gradient descent is often computationally intractable due to the non-differentiable nature of physical simulators and the difficulty of propagating gradients through long-horizon rollouts. 
Existing methods often decouple Eq.~\ref{eq:min_max} into separate problems:
(1) \textit{Adversarial generation}: generating a static set of hard scenarios via surrogate objectives $J_\text{adv}$ (e.g., collision probability);
% With a hypothetical policy $\pi_\theta$, the adversary optimizes a surrogate objective $J_\text{adv}$ (e.g., collision probability) to generate hard cases.
(2) \textit{Policy optimization}: then optimizing $\pi_\theta$ against this fixed distribution. 
% The generated scenarios are added to the training data, and $\pi_\theta$ is updated to maximize reward against the prescribed distribution. 
However, this decoupled paradigm introduces objective misalignment and fails to capture the non-stationary vulnerability frontier of the evolving policy.

%%%%%%%%%%%%%%%%%%%%%%%%%%%%%%%%%%%%%%%%%%%%%
\vspace{-5pt}
\section{The ADV-0 Framework}
\label{sec:method}

%%%%%%%%%%%%%%%%%%Framework%%%%%%%%%%%%%%%%%%%%%%%%%%%
% 1. \textbf{Sec.~3.1 (Framework):} We formulate the problem as a game and show that the \textbf{theoretical objective} of the inner loop is to identify an optimal adversarial distribution that follows a Gibbs distribution, yielding an energy-based formulation.

% 2. \textbf{Sec.~3.2 (Sampling):} We address the question of \textbf{``how to obtain data''}. Since direct sampling from the theoretical distribution is intractable, we use the current generator $\mathcal{G}_\psi$ as a proposal distribution and approximate the target Gibbs distribution via importance sampling (i.e., temperature-scaled sampling), further accelerated by a proxy reward.

% 3. \textbf{Sec.~3.3 (Learning):} We address the question of \textbf{``how to update the generator''}. To drive the generator $\mathcal{G}_\psi$ closer to the theoretical Gibbs distribution in the next iteration, we model the parameter update as a preference learning problem.
%%%%%%%%%%%%%%%%%%Framework%%%%%%%%%%%%%%%%%%%%%%%%%%%

\begin{wrapfigure}{l}{0.56\textwidth}
  \vspace{-8pt}
  \centering
  \includegraphics[width=\linewidth]{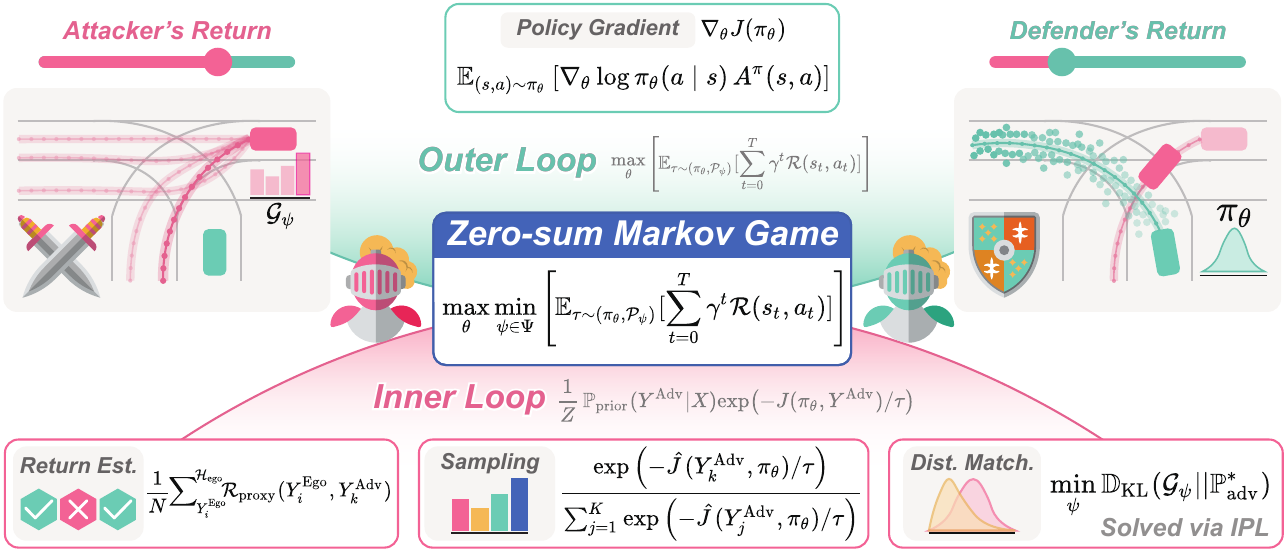}
  \vspace{-7pt}
  \caption{Illustration of \texttt{ADV-0}.  It alternates between an \textit{Inner Loop} where the adversary $\psi$ evolves via IPL to track the failure frontier, and an \textit{Outer Loop} where the ego $\pi_\theta$ optimizes policy gradients against the induced adversarial distribution.
  }
  \label{fig:method}
  \vspace{-12pt}
\end{wrapfigure}

% \vspace{-5pt}
We introduce \texttt{ADV-0}, a closed-loop adversarial training framework to solve the min-max optimization problem in Eq.~\ref{eq:min_max}. 
Our approach treats the interaction between the driving agent (\textit{defender}) and the traffic environment (\textit{attacker}) as a dynamic \textit{zero-sum game}: the defender minimizes the expected risk, while the attacker continuously explores the long-tail distribution to identify and exploit the ego's evolving weaknesses. 
Due to the nonstationarity of the driving environment and the rarity of critical events, relying on static datasets or heuristic adversarial priors is insufficient. 
Instead, we seek a Nash Equilibrium where the ego policy $\pi_{\theta}$ remains robust to the worst-case distributions generated by a continuously evolving adversary policy $\psi$. 
At this equilibrium, the performance of $\pi_{\theta}$ can theoretically be related to other distributions within the trust region.
This section details this end-to-end bi-level optimization scheme.

%%%%%%%%%%%%%%%%%%%%%%%%%%%%%%%%%%%%%%%%%%%%%%%
% \vspace{-10pt}
% \vspace{-10pt}
%%%%%%%%%%%%%%%%%%%%%%%%%%%%%%%%%%%%%%%%%%%%%%%

\vspace{-5pt}
\subsection{End-to-End Framework for Min-Max Optimization}
\label{subsec:framework}

\vspace{-5pt}
To enable tractability and stationarity, we propose an iterative end-to-end training pipeline inspired by Robust Adversarial Reinforcement Learning (RARL,~\citet{pinto2017robust}). The core of \texttt{ADV-0} lies in efficiently propagating gradients from the ego's objective ${J}$ to the adversary $\psi$. 
The training alternates between two phases: (1) an \textit{inner loop} updating $\psi$ to track the theoretical optimal attack distribution; and (2) an \textit{outer loop} optimizing policy $\pi_\theta$ against the current induced distribution (\cref{fig:method}).

%%%%%%%%%%%%%%%%%%%%%%%%%%%%%%%%%%%%%%%%%%%%%%%%%%%%%%%
\vspace{-8pt}
\paragraph{Inner loop.}
Let $X$ denote the context (e.g., map topology and initial traffic state) and $Y^\text{Adv}$ denote future adversarial trajectories. 
% We define the goal of the attacker as generating trajectories that have high prior likelihood while forcing the ego agent to fail. 
We define critical event $\mathcal{E}$ as the set of scenarios where the ego's performance drops below a safety threshold $\epsilon$: $\mathcal{E} := \{Y^{\text{Adv}} \mid J(\pi_\theta, Y^{\text{Adv}}) \leq \epsilon\}$. 
Then the attacker's goal is to find an optimal adversarial distribution $\mathbb{P}_\text{adv}$ maximizing likelihood of $\mathcal{E}$ while constrained within the trust region of the naturalistic prior $\mathbb{P}_{\text{prior}}$.
% This is equivalent to solving a constrained optimization problem subject to the ego's performance falling below a safety threshold $\epsilon$:
This is equivalent to a constrained optimization:
\begin{equation}
    \label{eq:constrained_opt}
    \begin{aligned}
        \max_{\mathbb{P}_\text{adv}} ~&\mathbb{E}_{X \sim \mathcal{D}} \left[ \mathbb{E}_{Y^\text{Adv} \sim \mathbb{P}_\text{adv}(\cdot|X)} [\log \mathbb{P}(\mathcal{E})] \right] ,
        &\text{s.t.}~\mathbb{D}_{\text{KL}}(\mathbb{P}_\text{adv} || \mathbb{P}_{\text{prior}}) \le \delta.
    \end{aligned}
\end{equation}
Directly solving Eq.~\ref{eq:constrained_opt} involves optimizing a hard indicator function $\mathbb{I}(Y^{\text{Adv}}\in\mathcal{E})$, suffering from vanishing gradients and event sparsity. To enable end-to-end gradient optimization, we relax the hard constraint into a soft energy function. We view the negative return $-J(\pi_\theta,Y^{\text{Adv}})$ as the unnormalized log-likelihood of the critical event. The objective is thus relaxed to maximizing the expected adversarial utility $\mathbb{P}(\mathcal{E} | Y^\text{Adv})\approx\mathbb{E}_{Y^{\text{Adv}} \sim \mathbb{P}_\text{adv}} \left[-J(\pi_\theta, Y^{\text{Adv}}) \right]$ within the trust region.
By applying Lagrange multipliers, the optima $\mathbb{P}_\text{adv}^*$ can be derived in closed form as the Gibbs distribution:
% \vspace{-5pt}
\begin{equation}\label{eq:factorization}
    \underbrace{\mathbb{P}_\text{adv}^*(Y^{\text{Adv}}|X)}_{\text{Posterior}} = \frac{1}{Z} \underbrace{\mathbb{P}_{\text{prior}}(Y^{\text{Adv}}|X)}_{\text{Traffic prior}} \underbrace{\exp \bigl( -J(\pi_\theta, Y^{\text{Adv}})/\tau \bigr)}_{\text{Generalized adversarial utility}},
\end{equation}
where $Z$ is the partition function and $\tau$ is the temperature. 
Eq.~\ref{eq:factorization} reveals that the theoretically optimal adversary re-weights the traffic prior based on the generalized adversarial utility (GAU).
Traffic prior ensures plausibility, and GAU is the likelihood that $Y^{\text{Adv}}$ causes the current policy $\pi_\theta$ to fail.
However, directly sampling from $\mathbb{P}_\text{adv}^*(\cdot|X)$ is intractable due to the unknown partition function. 
\texttt{ADV-0} solves for this via a two-step approximation: (1) \textbf{Sampling (Sec.~\ref{subsec:sampling})}: We approximate expectations over $\mathbb{P}_\text{adv}^*$ using importance sampling from the current prior;
and (2) \textbf{Learning (Sec.~\ref{subsec:ipl})}: We update the parameterized adversary $\psi$ to approximate the theoretical optimum $\mathbb{P}_\text{adv}^*$ via preference learning.

\vspace{-8pt}
\paragraph{Outer loop.} With the adversary $\psi_{{k+1}}$ fixed, the ego updates its policy to maximize the expected return under the induced adversarial distribution using standard RL methods:
\vspace{-2pt}
\begin{equation}
    \theta^{k+1} \leftarrow \underset{\theta^{k}}{\arg\max} \ J(\pi_\theta^{k}, \mathcal{P}_{\psi_{{k+1}}}).
\end{equation}
\vspace{-1pt}
Crucially, this general framework is agnostic to the ego's specific RL method. Since the adversary interacts with the ego solely through generated trajectories, the outer loop supports both on-policy and off-policy algorithms, by adjusting the synchronization schedule between the two. Moreover, \texttt{ADV-0} can be applied to $\pi_{\theta}$ either outputs continuous control signals (e.g., acceleration) or future trajectory plans (e.g., multi-modal trajectories with scores).
Training proceeds by fixing one player while updating the other. This coupled iteration ensures that the attacker dynamically tracks the defender's vulnerability frontier, while the defender learns to generalize against an increasingly sophisticated attacker. The complete training procedure is summarized in \cref{alg:adv0} in Appendix~\ref{app:implementation}.

\vspace{-5pt}
\subsection{Reward-guided Adversarial Sampling \& Alignment}
\label{subsec:sampling}
\vspace{-5pt}

To approximate the optimal distribution $\mathbb{P}_\text{adv}^*$ without computationally expensive MCMC, we adopt a generate-and-resample paradigm. 
Instead of generating trajectories from scratch, we sample from a pretrained multi-modal trajectory generator $\mathcal{G}_\psi$ that approximates $\mathbb{P}_{\text{prior}}(\cdot|X)$.
Given a context $X$, the current $\mathcal{G}_\psi$ produces $K$ candidates $\{Y^{\text{Adv}}_k\}_{k=1}^K \sim \mathcal{G}_\psi(\cdot|X)$ with prior probabilities. We then re-weight these candidates to approximate samples from $\mathbb{P}_\text{adv}^*$ using GAU.
% This ensures that the search space is confined to the feasible set $\Psi$ of physically plausible behaviors.
% The core challenge lies in re-weighting these candidates according to the adversarial utility $\mathbb{P}(\mathcal{E} | Y^\text{Adv}, \pi_\theta)$.

\vspace{-7pt}
\paragraph{Direct objective alignment.}
Prior works \citep{zhang2022adversarial,zhang2023cat} simplify the GAU via heuristic surrogates, e.g., collision probability. 
% They often use a collision probability estimator as the likelihood $\mathbb{P}(Y^\text{Adv} | \text{Collision}, X)$.
However, this introduces a misalignment: the defender optimizes comprehensive reward (safety, efficiency, comfort), while the attacker solely targets collisions. 
This discrepancy in gradient direction prevents the attacker from identifying non-collision failures (e.g., off-road) and allows the defender to overfit to specific collision modes while remaining vulnerable to other risks.
In contrast, \texttt{ADV-0} directly aligns the GAU with the defender's objective by \textit{setting the energy function as the negation of ego's cumulative return} (Eq.~\ref{eq:factorization}). 
% We define the generalized adversarial utility based on the direct minimization of the ego's cumulative return (Eq.~\ref{eq:factorization}). 
By targeting exactly what the ego optimizes, the attacker offers holistic supervision signals across the entire reward space, whether they are safety violations or efficiency drops.
% By targeting the exact goal the ego is trying to optimize, the attacker uncovers the most relevant failure modes, whether they be safety violations or efficiency drops, without relying on handcrafted heuristics.

\vspace{-7pt}
\paragraph{Efficient return estimation.}
Evaluating the exact return $\mathbb{E}[J(\pi_\theta, Y^\text{Adv}_k)]$ for all $K$ candidates via closed-loop simulation requires fully rolling out the policy, which is computationally prohibitive. To address this, we propose a \textit{Proxy Reward Evaluator} that estimates expected return via a lightweight function query.
Since the ego's response to $Y^\text{Adv}$ is stochastic during training, we treat $J$ as a random variable. We maintain a context-aware dynamic buffer of the ego's recent responses $\mathcal{H}_\text{ego}(X) = \{Y^\text{Ego}_i|X\}_{i=1}^N$ containing the $N$ most recent trajectories generated by $\pi_\theta$ in context $X$. For a new context with an empty buffer, we perform a single warm-up rollout using $\pi_{\theta}$ against the replay log to initialize the buffer. 
% This ensures that the proxy estimator calculates geometric interactions between trajectories sharing the same spatial topology.
We treat $\mathcal{H}_\text{ego}$ as an empirical approximation of current policy distribution. 
The expected return for a candidate $Y^\text{Adv}_k$ is estimated via Monte Carlo integration against the history:
\begin{equation}
    \label{eq:proxy_reward}
    \hat{J}(Y^\text{Adv}_k, \pi_\theta) \approx \frac{1}{N} \sum_{Y^\text{Ego}_i \in \mathcal{H}_\text{ego}} \mathcal{R}_{\text{proxy}}(Y^\text{Ego}_i, Y^\text{Adv}_k),
\end{equation}
where $\mathcal{R}_{\text{proxy}}$ is a vectorized function that computes geometric interactions (e.g., progress, collision overlap) between the adversary path and the cached ego paths without stepping the physics engine. 
% We empirically demonstrate that this rule-based proxy provides a sufficiently accurate gradient direction for adversarial sampling (see \cref{subsubsec:proxy_reward}).
% This rule-based proxy provides an efficient and unbiased estimate of the ego's performance against a specific attack $Y^\text{Adv}_k$ under the current distribution.
This rule-based proxy provides an efficient and sufficiently accurate gradient direction for adversarial sampling (see \cref{subsubsec:proxy_reward} for implementation).

\vspace{-7pt}
\paragraph{Temperature-scaled sampling.}
Finally, to select the adversarial trajectory for training, we implement the Gibbs distribution (Eq.~\ref{eq:factorization}) over the finite set of $K$ candidates. The probability of selecting candidate $k$ is given by the scaled softmax distribution over negative estimated returns:
\begin{equation}
    \label{eq:sampling}
    \mathbb{P}(Y^\text{Adv}_k) = \frac{\exp\left(-\hat{J}(Y^\text{Adv}_k, \pi_\theta) / \tau\right)}{\sum_{j=1}^K \exp\left(-\hat{J}(Y^\text{Adv}_j, \pi_\theta) / \tau\right)}.
\end{equation}
This serves as an \textit{importance sampling} step, re-weighting the proposal distribution $\mathcal{G}_\psi$ towards the theoretical optimum $\mathbb{P}_\text{adv}^*$ (see Appendix~\ref{subsec:convergence}). 
$\tau$ balances exploration and exploitation: 
% A limit of $\tau\to0$ selects the worst case, while a higher $\tau$ introduces stochasticity. 
$\tau\to0$ selects the worst case, while larger $\tau$ retains diversity from the prior.
It ensures that the defender is exposed to diverse challenging scenarios from the long tail, rather than collapsing into a single worst case.
% We justify that this sampling strategy corresponds to the closed-form optimality of the Gibbs adversary for the zero-sum game.

%%%%%%%%%%%%%%%%%%%%%%%%%%%%%%%%%%%%%%%%%%%%%%%%%%%%%%%%%%%%%%%%%%%%%%%%%
\vspace{-5pt}
\subsection{Iterative Preference Learning in the Long Tail}
\label{subsec:ipl}

\vspace{-5pt}
While the sampling strategy in \cref{subsec:sampling} identifies hard cases within the support of $\mathcal{G}_\psi$, this fixed proposal bounds its efficacy. As the defender $\pi_\theta$ improves, its weakness shifts into the long tail where the prior $\mathcal{G}_{\text{ref}}$ has negligible mass. Relying solely on static sampling becomes inefficient. 
% Therefore, to actively discover new failure modes, the parameters of $\psi$ should be updated to modulate the generative distribution.
To track the shifting frontier, we update $\psi$ to match the distribution of generator $\mathcal{G}_\psi$ and optimal target $\mathbb{P}_\text{adv}^*$.

\vspace{-7pt}
\paragraph{Implicit reward optimization via preferences.} 
Formally, our goal is to minimize the KL-divergence between $\mathcal{G}_\psi$ and $\mathbb{P}_\text{adv}^*$. Recall the definition in Eq.~\ref{eq:factorization}, we have:
\begin{equation}
\begin{aligned}
    &\min_\psi \mathbb{D}_{\text{KL}}(\mathcal{G}_\psi || \mathbb{P}_\text{adv}^*) = \min_\psi \mathbb{E}_{Y \sim \mathcal{G}_\psi} \left[ \log \mathcal{G}_\psi(Y|X)/\mathbb{P}_\text{adv}^*(Y|X) \right] \\
    &= \min_\psi \mathbb{E}_{Y \sim \mathcal{G}_\psi} \bigl[ \log \frac{\mathcal{G}_\psi(Y|X)}{\mathcal{G}_{\text{ref}}(Y|X)} + \frac{1}{\tau} J(\pi_\theta, Y) \bigr] + \text{const.} \\
    &\iff \max_\psi \mathbb{E}_{Y \sim \mathcal{G}_\psi} \bigl[-J(\pi_\theta, Y) \bigr] - \tau \mathbb{D}_{\text{KL}}(\mathcal{G}_\psi || \mathcal{G}_{\text{ref}}). 
\end{aligned}
\label{eq:rl_equivalence}
\end{equation}
Eq.~\ref{eq:rl_equivalence} reveals a standard RL objective: maximizing the expected adversarial reward subject to a KL-divergence constraint. However, directly solving it via policy gradient is notoriously unstable in this context due to the high variance of gradients. The action space of continuous trajectories is high-dimensional, and the zero-sum interaction often leads to mode collapse. 
Instead of explicit RL, we cast the problem as preference learning. Preference learning updates the generator of adversarial traffic trajectories $\mathcal{G}_\psi$ using candidate orderings induced by the current defender's cumulative return. Within each context, the candidate with lower estimated defender return is preferred, and these labels are refreshed as the defender evolves. Following \citet{rafailov2023direct}, the optimal policy for the KL-constrained objective satisfies a specific preference ordering, which is equivalent to optimizing an \textit{implicit reward}.
This allows us to update $\psi$ using a supervised loss on preference pairs, bypassing the need for an explicit value function or unstable reward maximization.

% \vspace{-5pt}
% \paragraph{Implicit reward optimization via preferences.}
% Recall the inner loop objective in Eq.~\ref{eq:constrained_opt}, which seeks to maximize the adversarial likelihood, a natural solution is to maximize the expected adversarial utility with a KL-divergence constraint to stay close to the naturalistic traffic prior $\mathcal{G}_{\text{ref}}$:
% \begin{equation}
%     \label{eq:kl_objective}
%     \max_{\mathcal{G}_\psi} \mathbb{E}_{Y^\text{Adv} \sim \mathcal{G}_\psi} \left[ r(Y^\text{Adv}, \pi_\theta) - \tau \mathbb{D}_{\text{KL}}(\mathcal{G}_\psi(\cdot|X) || \mathcal{G}_{\text{ref}}(\cdot|X)) \right],
% \end{equation}
% where $r(\cdot) = -J(\pi_\theta, Y^\text{Adv})$ represents the induced reward for the adversary, and $\tau$ controls the regularization.
% Directly solving Eq.~\ref{eq:kl_objective} via RL such as policy gradient is notoriously unstable in this context. The action space of trajectory generation is high-dimensional, and the zero-sum interaction between two players often leads to mode collapse, where the adversary converges to an unrealistic attack pattern. 

\vspace{-7pt}
\paragraph{Online iterative evolution.}
Standard preference learning methods are often offline with static preference datasets. Instead, \texttt{ADV-0} operates in a nonstationary game where the preference labels depend on evolving players. 
Therefore, we propose online \textit{Iterative Preference Learning} (IPL) in the inner loop. 
IPL generates preference data on-the-fly, conditioned on the current attacker and labels it using the current defender.
This process proceeds on-policy: 
\textbf{(1) Sampling}: For a given context $X$, a series of candidates are generated $\{Y^\text{Adv}_k\}_{k=1}^K$ from the \textit{current} attacker $\mathcal{G}_\psi$. 
\textbf{(2) Labeling}: Each candidate is evaluated using the proxy reward evaluator $\hat{J}(\cdot, \pi_\theta)$. Crucially, this evaluation uses the \textit{latest} history of the defender. 
\textbf{(3) Pairing}: A preference dataset $\mathcal{D}_{\text{pref}}$ is curated by pairing $(Y_w, Y_l)$ from the candidates, where $Y_w$ is preferred over $Y_l$ if $\hat{J}(Y_w, \pi_\theta) < \hat{J}(Y_l, \pi_\theta)$. To prevent trivial comparisons, we apply a reward margin $\delta$ and a spatial diversity filter $\xi$:
$\mathcal{D}_{\text{pref}} = \left\{ (Y_w, Y_l) \mid \hat{J}(Y_l) - \hat{J}(Y_w) > \delta \;\land\; \|Y_w - Y_l\|_2 > \xi \right\}$.
This reduces noise in the proxy, making preferences distinguishable between structurally different attacks.

\vspace{-5pt}
\paragraph{Objective.}
We update the adversarial policy $\mathcal{G}_\psi$ by minimizing the negative log-likelihood of the preferred trajectories. To handle high-variance heterogeneous traffic scenarios without high cost of a massive replay buffer, we employ a streaming gradient accumulation strategy to stabilize the training. We process a stream of scenarios sequentially, accumulating gradients over a mini-batch of $M$ scenarios before performing a parameter update. The loss function $\mathcal{L}_{\text{IPL}}(\mathcal{G}_\psi)$ for a mini-batch $\mathcal{B}$ of generated pairs is:
\refstepcounter{equation}\label{eq:ipl_loss}%
$
     \frac{-1}{|\mathcal{B}|} \sum_{(Y_w, Y_l) \in \mathcal{B}} \log \sigma \Bigl( \tau \bigl[\log \frac{\mathcal{G}_\psi(Y_w|X)}{\mathcal{G}_{\text{ref}}(Y_w|X)} -  \log \frac{\mathcal{G}_\psi(Y_l|X)}{\mathcal{G}_{\text{ref}}(Y_l|X)}\bigr] \Bigr).
     $~\textup{(\theequation)}.
This reduces per-scenario update variance while maintaining on-policy data generation. Here, the reference model $\mathcal{G}_{\text{ref}}$ remains frozen as the pretrained prior. 
This streaming evolution ensures that the adversary continuously adapts to the defender's evolving capabilities. By pushing the generator's distribution towards the theoretical Gibbs optimum $\mathbb{P}_\text{adv}^*$, the defender continuously trains against the most pertinent long-tail risks, effectively mitigating the distribution shift and forcing robust generalization.

%%%%%%%%%%%%%%%%%%%%%%%%%%%%%%%%%%%%%%%%%%%%%%%%%%%%%%%%%%%%%%%
\vspace{-5pt}
\subsection{Theoretical Analysis}\label{ssec:theory}

\vspace{-5pt}
We provide a theoretical analysis of the convergence properties of \texttt{ADV-0} and establish a generalization bound that characterizes the agent's performance in real-world long-tail distributions. Derivations and proofs are provided in \cref{sec:theory_appendix}.

\vspace{-5pt}
\paragraph{Convergence to Nash Equilibrium.}
The interaction between the defender and the attacker is modeled as a regularized zero-sum Markov game. 
Building on the finding that optimizing Eq.~\ref{eq:ipl_loss} recovers the optimal adversary solved for a Gibbs distribution, we prove that iterative updates constitute a contraction mapping on the value function space.

\begin{tcolorbox}[colback=myblue!8, colframe=myblue!8, boxrule=0pt, sharp corners=all, boxsep=0pt, left=5pt, right=5pt, top=5pt, bottom=5pt, before skip=2pt, after skip=2pt]
\begin{theorem}[Convergence to Nash Equilibrium]
Under exact inner optimization and policy improvement, the idealized regularized updates converge to a unique fixed point corresponding to the Nash Equilibrium $(\pi^*, \psi^*)$ of the game. It satisfies the saddle-point inequality $\mathcal{J}_{\tau}(\pi, \mathcal{G}_{\psi^*}) \leq \mathcal{J}_{\tau}(\pi^*, \mathcal{G}_{\psi^*}) \leq \mathcal{J}_{\tau}(\pi^*, \mathcal{G}_\psi)$ for all feasible policies, where $\mathcal{J}_{\tau}$ is the regularized objective.
\end{theorem}
\end{tcolorbox}

\textbf{Generalization to real-world long tail.} 
A core concern is whether robustness against a generated adversary $\mathcal{P}_\psi$ translates to safety in the real-world long-tail distribution $\mathcal{P}_{\text{real}}$.  
We model real dynamics as an unknown distribution lying within the trust region of the traffic prior. We derive a lower bound on expected return by measuring the discrepancy between two induced transition dynamics.

\begin{tcolorbox}[colback=myblue!8, colframe=myblue!8, boxrule=0pt, sharp corners=all, boxsep=0pt, left=5pt, right=5pt, top=5pt, bottom=5pt, before skip=2pt, after skip=2pt]
\begin{theorem}[Generalizability]\label{thm:bound_main}
Let $V_{\max}$ be the maximum of the value function, $\mathcal{P}_{\text{real}}$ be the real dynamics, and $\pi_\theta$ is trained under the adversarial dynamics $\mathcal{P}_\psi$ induced by $\mathcal{G}_\psi$. The performance of $\pi_\theta$ under $\mathcal{P}_{\text{real}}$ is bounded by:
\begin{equation*}
    J(\pi_\theta, \mathcal{P}_{\text{real}}) \ge J(\pi_\theta, \mathcal{P}_{\psi}) - \frac{\gamma V_{\max} \sqrt{2}}{1-\gamma} \sqrt{\mathbb{E} [D_{\text{KL}}(\mathcal{G}_\psi \| \mathcal{G}_{\text{ref}})]}.
\end{equation*}
\end{theorem}
\end{tcolorbox}
Theorem~\ref{thm:bound_main} implies that optimizing against the generated adversary improves a lower bound on the expected return in the real world.
The outer loop maximizes the robust return $J$, while the inner loop optimizes a KL-regularized adversary objective that controls deviation from the naturalistic prior.

%%%%%%%%%%%%%%%%%%%%%%%%%%%%%%%%%%%%%%%%%%%%%%%%%%%%%%%%%%%%%%%%%%

% \newpage
\vspace{-10pt}
\section{Experiments}
\vspace{-6pt}
We empirically evaluate \texttt{ADV-0} to answer three core questions: (1) Can \texttt{ADV-0} generate plausible yet long-tailed scenarios that effectively expose the vulnerabilities of driving agents? (2) Does the training process yield a robust policy that generalizes to diverse adversarial attacks? (3) Can the safety improvements observed in simulation transfer to real-world long-tailed events? 
All experiments are performed in MetaDrive simulator using scenarios converted from WOMD.

\vspace{-8pt}
\subsection{Generating Safety-Critical Scenarios}

\vspace{-5pt}
\paragraph{Main results.} We first evaluate \texttt{ADV-0} in generating safety-critical scenarios against various ego policies. 
As presented in \cref{tab:scenario_gen}, \texttt{ADV-0} consistently outperforms competing baselines in exposing system vulnerabilities, especially for reactive policies such as IDM and RL. 
A detailed ablation in \cref{tab:scenario_gen_adv0} further highlights two findings: (1) The proposed energy-based sampling strategy (Eq.~\ref{eq:sampling}) is effective compared to the standard logit-based scheme.
(2) IPL-based fine-tuning, which is agnostic to the underlying RL algorithm, further refines the adversary. By fine-tuning the generator against specific ego policies, the adversary learns to exploit specific weaknesses.

\vspace{-1pt}
\begin{figure}[!htbp]
  \centering 
  \begin{minipage}[t]{0.28\textwidth}
    \centering
    \includegraphics[width=\linewidth]{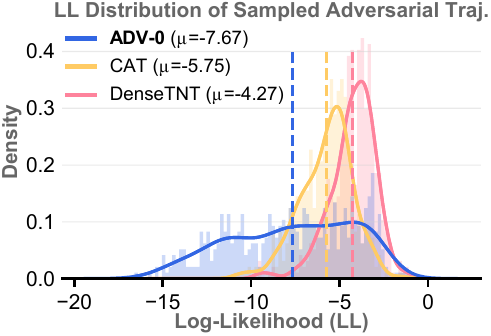}
    \vspace{-8pt}
    \caption{LL distribution of different adversarial generators.}
    \label{fig:ll}
  \end{minipage}
  \hfill
  \begin{minipage}[t]{0.28\textwidth}
    \centering
    \scalebox{1}[0.88]{\includegraphics[width=\linewidth]{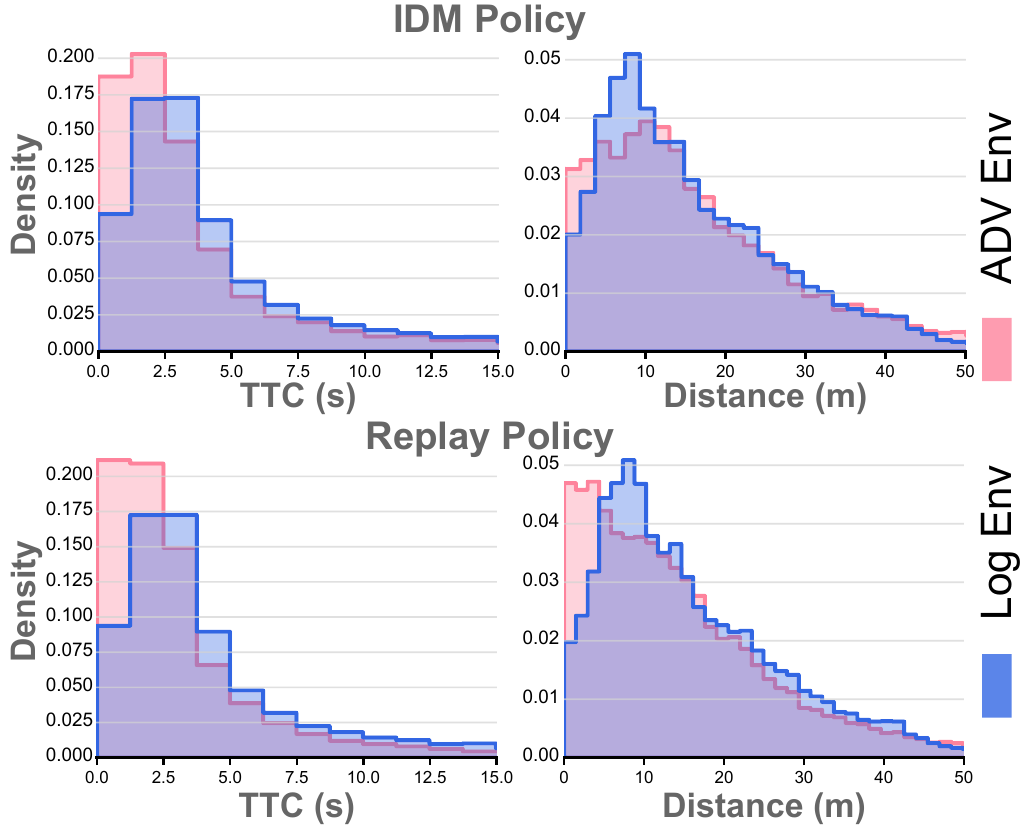}}
    \vspace{-8pt}
    \caption{Scene-level distributions of B2B distance and TTC.}
    \label{fig:dist_metrics}
  \end{minipage}
  \hfill
  \begin{minipage}[t]{0.38\textwidth}
    \centering
    \includegraphics[width=\linewidth]{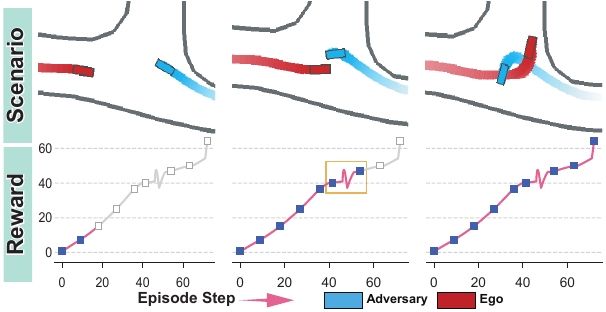}
    \vspace{-8pt}
    \caption{Examples of reward-reduced adversarial scenarios from \texttt{ADV-0}. Additional cases are shown in \cref{fig:case-reward-appendix}.}
    \label{fig:case-reward-main}
  \end{minipage}
\end{figure}
\vspace{-5pt}

\begin{table}[htbp]
\centering
\caption{\textbf{Safety-critical scenario generation}. \textit{CR} denotes the collision rate (\%) of the ego in generated scenarios, and \textit{ER} is the ego's cumulative return. Ego agents include Replay, IDM, and trained PPO. \texttt{ADV-0} entries are macro means across six fine-tuned generator variants; \cref{tab:scenario_gen_adv0} provides the corresponding dispersion statistics.}
\label{tab:scenario_gen}
\setlength{\tabcolsep}{4pt}
\renewcommand{\arraystretch}{1.1}
\resizebox{0.92\textwidth}{!}{%
\begin{tabular}{ll ccccccccc}
\toprule
\multirow{2}{*}{\textbf{Ego}} & \multirow{2}{*}{\textbf{Metric}} & \multicolumn{9}{c}{\textbf{Adversary}} \\
\cmidrule(l){3-11}
 & & Replay & Heuristic & CAT & KING & AdvTrajOpt & GOOSE & SAGE & SEAL & \cellcolor{myblue!20} \textbf{ADV-0 (ours)} \\
\midrule
\multirow{2}{*}{\textbf{Replay}} 
 & CR (\%) $\uparrow$ & - & 100.00 & 90.08 & 23.28 & 69.64 & 20.46 & 74.53 & 59.06 & \cellcolor{myblue!20}\textbf{91.10} \\
 & ER $\downarrow$    & - & 0.00   & 1.03  & 47.67 & 3.12  & 48.52 & 2.57  & 8.58  & \cellcolor{myblue!20}\textbf{0.99} \\
\midrule
\multirow{2}{*}{\textbf{IDM}} 
 & CR (\%) $\uparrow$ & 19.03 & 74.70 & 43.13 & 24.49 & 26.92 & 24.48 & 36.50 & 36.70 & \cellcolor{myblue!20}\textbf{45.83} \\
 & ER $\downarrow$    & 51.75 & 32.12 & 43.47 & 49.41 & 45.36 & 49.45 & 43.81 & 43.75 & \cellcolor{myblue!20}\textbf{40.03} \\
\midrule
\multirow{2}{*}{\textbf{RL Agent}} 
 & CR (\%) $\uparrow$ & 16.80 & 69.64 & 36.84 & 21.26 & 28.95 & 13.88 & 35.40 & 37.63 & \cellcolor{myblue!20}\textbf{40.68} \\
 & ER $\downarrow$    & 51.65 & 24.90 & 42.26 & 49.32 & 45.10 & 51.80 & 42.87 & 41.99 & \cellcolor{myblue!20}\textbf{39.13} \\
\bottomrule
\end{tabular}%
}
\end{table}
% \vspace{-5pt}

%%%%%%%%%%%%%%%%%%%%%%%%%%%%%%%%%%%%%%%%%%%%%%%%%%%%%%%%%%%%%%%%%%%%%%%%%%%%%%%

\vspace{-5pt}
\paragraph{Distribution of the long tail.} \texttt{ADV-0} can navigate the long-tail distribution of scenarios. 
\cref{fig:ll} illustrates the log-likelihood (LL) distribution of the sampled adversarial trajectories. 
Compared to CAT and pretrained prior, the distribution of \texttt{ADV-0} is shifted towards lower likelihood and exhibits a wider variance. This indicates that it uncovers rare but plausible, \textit{behavior-level} events that are typically ignored by standard priors. \cref{fig:dist_metrics} shows the \textit{scene-level} statistics. \texttt{ADV-0} produces notably lower Time-to-Collision (TTC) and closer Bumper-to-Bumper (B2B) distances compared to the naturalistic data. 
Crucially, this aggressiveness does not compromise physical plausibility. As shown in \cref{fig:real_pen}, \texttt{ADV-0} maintains a comparable realism penalty. 

%%%%%%%%%%%%%%%%%%%%%%%%%%%%%%%%%%%%%%%%%%%%%%%%%%%%%%%%%%%%%%%%%%%%%%%
% \vspace{-10pt}
\begin{table}[t]
\centering
\vspace{-10pt}
\caption{\textbf{Performance validation of learned policies}. Each environment reports the macro mean across six RL algorithms (GRPO, PPO, PPO-Lag, SAC, SAC-Lag, TD3) $\pm$ the mean of the standard deviations computed over five seeds for each algorithm. Avg. averages these statistics across the five environments. Full results are shown in Tables~\ref{tab:agent_grpo},~\ref{tab:agent_ppo},~\ref{tab:agent_ppo_lag},~\ref{tab:agent_sac},~\ref{tab:agent_saclag},~\ref{tab:agent_td3}.}
% \vspace{-5pt}
\label{tab:agent_average_all}
\setlength{\tabcolsep}{2pt}
\renewcommand{\arraystretch}{1.12}
\begin{small}
\resizebox{\textwidth}{!}{%
\begin{tabular}{@{}lcccccccc@{}}
\toprule
\textbf{Val. Env.} & \multicolumn{8}{c}{\textbf{ADV-0 (w/ IPL)}} \\
\cmidrule(l){2-9}
& \multicolumn{2}{c}{RC $\uparrow$} & \multicolumn{2}{c}{Crash $\downarrow$} & \multicolumn{2}{c}{Reward $\uparrow$} & \multicolumn{2}{c}{Cost $\downarrow$} \\
\midrule
Replay    & \multicolumn{2}{c}{0.742 $\pm$ 0.011} & \multicolumn{2}{c}{0.159 $\pm$ 0.019} & \multicolumn{2}{c}{49.56 $\pm$ 1.05} & \multicolumn{2}{c}{0.480 $\pm$ 0.013} \\
ADV-0     & \multicolumn{2}{c}{0.695 $\pm$ 0.011} & \multicolumn{2}{c}{0.289 $\pm$ 0.029} & \multicolumn{2}{c}{44.60 $\pm$ 0.83} & \multicolumn{2}{c}{0.598 $\pm$ 0.022} \\
CAT       & \multicolumn{2}{c}{0.704 $\pm$ 0.011} & \multicolumn{2}{c}{0.271 $\pm$ 0.025} & \multicolumn{2}{c}{45.68 $\pm$ 1.16} & \multicolumn{2}{c}{0.585 $\pm$ 0.025} \\
SAGE      & \multicolumn{2}{c}{0.699 $\pm$ 0.013} & \multicolumn{2}{c}{0.263 $\pm$ 0.025} & \multicolumn{2}{c}{45.19 $\pm$ 1.54} & \multicolumn{2}{c}{0.567 $\pm$ 0.028} \\
Heuristic & \multicolumn{2}{c}{0.710 $\pm$ 0.016} & \multicolumn{2}{c}{0.217 $\pm$ 0.021} & \multicolumn{2}{c}{45.41 $\pm$ 2.08} & \multicolumn{2}{c}{0.552 $\pm$ 0.035} \\
\cmidrule(l){1-9}
\cellcolor{myblue!30}\textbf{Avg.} & \multicolumn{2}{c}{\cellcolor{myblue!30}\textbf{0.710 $\pm$ 0.012}} & \multicolumn{2}{c}{\cellcolor{myblue!30}\textbf{0.240 $\pm$ 0.024}} & \multicolumn{2}{c}{\cellcolor{myblue!30}\textbf{46.09 $\pm$ 1.33}} & \multicolumn{2}{c}{\cellcolor{myblue!30}\textbf{0.556 $\pm$ 0.025}} \\
\midrule
\textbf{Val. Env.} & \multicolumn{4}{c}{\textbf{ADV-0 (w/o IPL)}} & \multicolumn{4}{c}{\textbf{CAT}} \\
\cmidrule(lr){2-5}\cmidrule(l){6-9}
& {RC} $\uparrow$ & {Crash} $\downarrow$ & {Reward} $\uparrow$ & {Cost} $\downarrow$ &
{RC} $\uparrow$ & {Crash} $\downarrow$ & {Reward} $\uparrow$ & {Cost} $\downarrow$\\
\midrule
Replay    & 0.719 $\pm$ 0.020 & 0.167 $\pm$ 0.024 & 46.41 $\pm$ 1.97 & 0.540 $\pm$ 0.023 &
0.720 $\pm$ 0.014 & 0.183 $\pm$ 0.026 & 46.52 $\pm$ 1.15 & 0.528 $\pm$ 0.015 \\
ADV-0     & 0.664 $\pm$ 0.018 & 0.317 $\pm$ 0.030 & 41.03 $\pm$ 1.87 & 0.657 $\pm$ 0.020 &
0.660 $\pm$ 0.018 & 0.332 $\pm$ 0.043 & 40.70 $\pm$ 0.97 & 0.660 $\pm$ 0.020 \\
CAT       & 0.677 $\pm$ 0.017 & 0.299 $\pm$ 0.035 & 42.30 $\pm$ 1.68 & 0.643 $\pm$ 0.030 &
0.667 $\pm$ 0.017 & 0.313 $\pm$ 0.032 & 41.30 $\pm$ 1.48 & 0.652 $\pm$ 0.018 \\
SAGE      & 0.672 $\pm$ 0.024 & 0.270 $\pm$ 0.028 & 42.17 $\pm$ 2.38 & 0.610 $\pm$ 0.038 &
0.660 $\pm$ 0.021 & 0.307 $\pm$ 0.025 & 41.41 $\pm$ 2.09 & 0.628 $\pm$ 0.017 \\
Heuristic & 0.679 $\pm$ 0.022 & 0.239 $\pm$ 0.041 & 41.78 $\pm$ 2.08 & 0.612 $\pm$ 0.040 &
0.676 $\pm$ 0.021 & 0.261 $\pm$ 0.034 & 41.72 $\pm$ 1.69 & 0.605 $\pm$ 0.023 \\
\cmidrule(l){1-9}
\cellcolor{myblue!15}\textbf{Avg.} & \cellcolor{myblue!15}\textbf{0.683 $\pm$ 0.020} & \cellcolor{myblue!15}\textbf{0.258 $\pm$ 0.031} & \cellcolor{myblue!15}\textbf{42.74 $\pm$ 2.00} & \cellcolor{myblue!15}\textbf{0.613 $\pm$ 0.030} &
\textbf{0.676 $\pm$ 0.018} & \textbf{0.279 $\pm$ 0.032} & \textbf{42.33 $\pm$ 1.47} & \textbf{0.614 $\pm$ 0.019} \\
\midrule
\textbf{Val. Env.} & \multicolumn{4}{c}{\textbf{Heuristic}} & \multicolumn{4}{c}{\textbf{Replay}} \\
\cmidrule(lr){2-5}\cmidrule(l){6-9}
& {RC} $\uparrow$ & {Crash} $\downarrow$ & {Reward} $\uparrow$ & {Cost} $\downarrow$ &
{RC} $\uparrow$ & {Crash} $\downarrow$ & {Reward} $\uparrow$ & {Cost} $\downarrow$\\
\midrule
Replay    & 0.682 $\pm$ 0.032 & 0.201 $\pm$ 0.018 & 42.75 $\pm$ 3.04 & 0.592 $\pm$ 0.037 &
0.692 $\pm$ 0.035 & 0.209 $\pm$ 0.030 & 42.93 $\pm$ 1.32 & 0.588 $\pm$ 0.025 \\
ADV-0     & 0.637 $\pm$ 0.047 & 0.331 $\pm$ 0.021 & 39.09 $\pm$ 2.12 & 0.678 $\pm$ 0.028 &
0.622 $\pm$ 0.038 & 0.374 $\pm$ 0.038 & 36.80 $\pm$ 2.48 & 0.700 $\pm$ 0.020 \\
CAT       & 0.650 $\pm$ 0.021 & 0.311 $\pm$ 0.020 & 40.11 $\pm$ 2.38 & 0.668 $\pm$ 0.033 &
0.642 $\pm$ 0.040 & 0.368 $\pm$ 0.037 & 38.26 $\pm$ 1.83 & 0.683 $\pm$ 0.030 \\
SAGE      & 0.642 $\pm$ 0.042 & 0.312 $\pm$ 0.019 & 39.53 $\pm$ 2.58 & 0.658 $\pm$ 0.025 &
0.625 $\pm$ 0.040 & 0.339 $\pm$ 0.040 & 37.68 $\pm$ 1.84 & 0.660 $\pm$ 0.032 \\
Heuristic & 0.641 $\pm$ 0.030 & 0.279 $\pm$ 0.016 & 38.81 $\pm$ 2.61 & 0.655 $\pm$ 0.028 &
0.638 $\pm$ 0.031 & 0.310 $\pm$ 0.034 & 37.80 $\pm$ 2.41 & 0.672 $\pm$ 0.015 \\
\cmidrule(l){1-9}
\textbf{Avg.} & \textbf{0.650 $\pm$ 0.034} & \textbf{0.287 $\pm$ 0.019} & \textbf{40.06 $\pm$ 2.54} & \textbf{0.650 $\pm$ 0.030} &
\textbf{0.644 $\pm$ 0.037} & \textbf{0.320 $\pm$ 0.036} & \textbf{38.70 $\pm$ 1.97} & \textbf{0.661 $\pm$ 0.024} \\
\bottomrule
\end{tabular}}
\end{small}
\end{table}
\vspace{-5pt}
%%%%%%%%%%%%%%%%%%%%%%%%%%%%%%%%%%%%%%%%%%%%%%%%%%%%%%%%%%%%%%%%%%%%%%%%%%%%%%%

% \vspace{-5pt}
\begin{table}[!htbp]
\centering
\caption{\textbf{Cross-validation of} \texttt{ADV-0}. Results are macro means across six RL algorithms $\pm$ the mean of their standard deviations over five seeds. Decrease and Improvement report percentage changes from adding IPL to the adversary and agent, respectively.}
\label{tab:cross_val_adv0}
\setlength{\tabcolsep}{4.5pt}
\renewcommand{\arraystretch}{1.1}
\begin{small} 
\resizebox{0.92\textwidth}{!}{%
\begin{tabular}{lcccccc}
\toprule
\multirow{2}{*}{\textbf{Agent/Adversary}} & \multicolumn{3}{c}{\textbf{Reward} ($\uparrow$)} & \multicolumn{3}{c}{\textbf{Cost} ($\downarrow$)} \\
\cmidrule(lr){2-4} \cmidrule(lr){5-7} 
 & \makecell{Adversary\\ (w/o IPL)} & \makecell{Adversary\\ (w/ IPL)} & Decrease 
 & \makecell{Adversary\\ (w/o IPL)} & \makecell{Adversary\\ (w/ IPL)} & Decrease \\
\midrule
Agent w/o IPL 
& $41.03 \pm 1.87$ & $39.01 \pm 0.97$ & \cellcolor{myyellow!40}$-4.92\%$ 
& $0.657 \pm 0.020$ & $0.685 \pm 0.016$ & \cellcolor{myyellow!40}$+4.26\%$ \\
Agent w/ IPL  
& $44.60 \pm 0.83$ & $43.47 \pm 1.40$ & \cellcolor{myyellow!20}$-2.53\%$ 
& $0.598 \pm 0.022$ & $0.615 \pm 0.025$ & \cellcolor{myyellow!20}$+2.84\%$ \\
\midrule
Improvement   
& \cellcolor{myteal!20}$+8.70\%$ & \cellcolor{myteal!40}$+11.43\%$ & -- 
& \cellcolor{myteal!20}$-8.98\%$ & \cellcolor{myteal!40}$-10.22\%$ & -- \\
\bottomrule
\end{tabular}%
}
\end{small}
\end{table}
% \vspace{-5pt}

% \vspace{-5pt}
\vspace{-5pt}

% \vspace{-15pt}
\paragraph{Beyond collision.} \texttt{ADV-0} directly targets the ego's return, allowing for the discovery of diverse failure modes beyond crashes. 
As shown in \cref{fig:case-reward-main} and \cref{fig:case-reward-appendix}, the adversary can force the ego into abnormal behaviors and non-collision failures, such as stalling at intersections or deviating from reference paths.
These behaviors result in a drastic drop in accumulated reward (e.g., lack of progress or discomfort penalties) even if a collision is avoided, which are critical performance risks often ignored by collision-centric attacks.

%%%%%%%%%%%%%%%%%%%%%%%%%%%%%%%%%%%%%%%%%%%%%%%%%%%%%%%%%%%%%%%%%%%%%%%

\vspace{-5pt}
\subsection{Learning Generalizable Driving Policies}
\label{subsec:exp_generalization}
\vspace{-5pt}

\paragraph{Generalization to unseen adversaries.} 
A core challenge in adversarial RL is overfitting to a specific adversary, which limits generalization to unseen risks. We conduct a cross-validation where agents trained via different methods are tested across a spectrum of environments. 
\cref{tab:agent_average_all} reports the performance averaged across six RL algorithms.
We observe that agents trained with \texttt{ADV-0} (w/ IPL) consistently achieve the best results across all metrics. 
While baselines often perform well against their own attacks, they degrade when applied to unseen adversarial distributions. In contrast, \texttt{ADV-0} maintains consistent generalizability. 
% For example, \texttt{ADV-0} agent outperforms the \texttt{CAT} agent even when evaluated within the \texttt{CAT} generated environment, and similarly outperforms the \texttt{SAGE} agent on \texttt{SAGE} scenarios. 
The comparison between \texttt{ADV-0} (w/o IPL) and CAT indicates the benefit of directly aligning the GAU with the ego's objective. By actively exploring the long tail via IPL, \texttt{ADV-0} enables generalized robustness to handle unseen risks.

% \vspace{-5pt}

\begin{wrapfigure}{l}{0.47\textwidth}
  \vspace{-8pt}
  \centering
  \includegraphics[width=\linewidth]{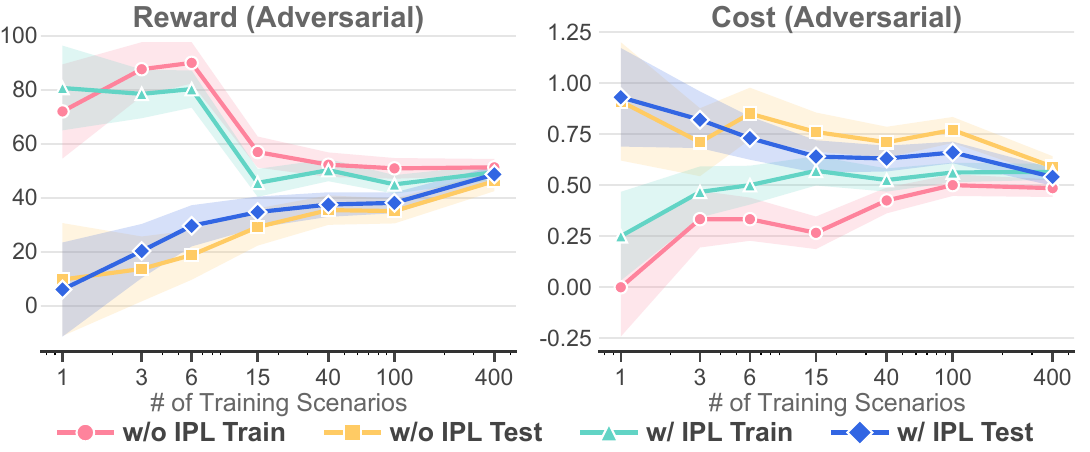}
  \vspace{-7pt}
  \caption{Generalization gap over different budgets.}
  \label{fig:train_test_gap}
  \vspace{-10pt}
\end{wrapfigure}

\vspace{-5pt}
\paragraph{Impact of IPL.} To isolate the gain provided by IPL, we compare the performance of agents and adversaries trained with and without IPL in \cref{tab:cross_val_adv0}. 
The inclusion of IPL in the adversary creates a more challenging environment, evidenced by the drop in agent rewards. However, the agent trained with full \texttt{ADV-0} becomes more robust when facing the stronger adversary. This confirms that the dynamic evolution of the adversary via IPL forces the defender to cover broader vulnerabilities.
Qualitative evidence of this dynamic evolution is shown in \cref{fig:ll_shift}. In the TD3 setting, \texttt{ADV-0} also achieves the highest route completion and return and the lowest crash rate and cost among the five methods compared in \cref{app:preference_baselines}.
Finally, we study the sample efficiency and overfitting risks regarding the training budget. As shown in \cref{fig:train_test_gap}, the gap between training and testing performance narrows with IPL, indicating that active discovery of diverse failures effectively mitigates the risk of overfitting to limited training data.
Detailed learning curves and results are provided in Figures~\ref{fig:td3_curves}-\ref{fig:grpo_curves} and Tables~\ref{tab:agent_grpo}-\ref{tab:agent_td3}.

%%%%%%%%%%%%%%%%%%%%%%%%%%%%%%%%%%%%%%%%%%%%%%%%%%%%%%%%%%%%%%%%%%%%%%%%%%%%%%%

%%%%%%%%%%%%%%%%%%%%%%%%%%%%%%%%%%%%%%%%%%%%%%%%%%%%%%%%%%%%%%%%%%%%%%%%%%%%%%%

\begin{table*}[!htbp]
\centering
\caption{\textbf{Robustness on mined real-world long-tailed sets.} Average agent performance on long-tail scenario categories filtered by: \textit{Critical TTC} ($\min \text{TTC} < 0.4s$), \textit{Critical PET} ($\text{PET} < 1.0s$), \textit{Hard Dynamics} (Acc. $< -4.0 m/s^2$ or $|\text{Jerk}| > 4.0 m/s^3$), and \textit{Rare Cluster} (topologically sparse trajectory clusters). 
% \textbf{Reactive Traffic} denotes whether background vehicles utilize IDM/MOBIL policies to interact with the agent ($\checkmark$) or strictly follow logged trajectories ($\times$). 
Metrics assess average \textit{Safety Margin} (higher indicates earlier risk detection), \textit{Stability \& Comfort} (lower Jerk indicates smoother control), \textit{Defensive Driving}, quantified by Near-Miss Rate (proximity without collision), RDP Violation (percentage of time requiring deceleration $> 6 m/s^2$ to avoid collision). Full results are shown in \cref{tab:long_tail}.}
% \vspace{-5pt}
\label{tab:long_tail_avg}
\setlength{\tabcolsep}{2.5pt}
\renewcommand{\arraystretch}{1.1}
\begin{small}
\resizebox{0.96\textwidth}{!}{%
\begin{tabular}{lccccccc}
\toprule
\multirow{2}{*}{\rotatebox{90}{}} & \multirow{2}[3]{*}{\makecell{\textbf{Reactive}\\\textbf{Traffic}}} & \multicolumn{2}{c}{\textbf{Safety Margin}} & \multicolumn{2}{c}{\textbf{Stability \& Comfort}} & \multicolumn{2}{c}{\textbf{Defensive Driving}} \\
\cmidrule(lr){3-4} \cmidrule(lr){5-6} \cmidrule(lr){7-8} 
 & & \makecell{{Avg Min-TTC} ($\uparrow$)} & \makecell{{Avg Min-PET} ($\uparrow$)} & \makecell{{Mean Abs Jerk} ($\downarrow$)} & \makecell{{95\% Jerk} ($\downarrow$)} & \makecell{{Near-Miss Rate} ($\downarrow$)} & \makecell{{RDP Violation Rate} ($\downarrow$)} \\
\midrule
\multirow{2}{*}{\makecell{\textbf{ADV-0}\\\textbf{(w/ IPL)}}} 
 & $\checkmark$ & $\cellcolor{myblue!20}{\boldsymbol{0.993 \pm 0.189}}$ & \cellcolor{myblue!20}${\boldsymbol{1.150 \pm 0.317}}$ & \cellcolor{myblue!20}${\boldsymbol{1.653 \pm 0.169}}$ & \cellcolor{myblue!20}${\boldsymbol{5.053 \pm 0.885}}$ & \cellcolor{myblue!20}${\boldsymbol{63.45\% \pm 4.34\%}}$ & \cellcolor{myblue!20}${\boldsymbol{33.70\% \pm 3.77\%}}$ \\
 & $\times$ & \cellcolor{myblue!20}${{1.527 \pm 0.182}}$ & \cellcolor{myblue!20}${\boldsymbol{1.103 \pm 0.313}}$ & \cellcolor{myblue!20}${\boldsymbol{1.617 \pm 0.204}}$ & \cellcolor{myblue!20}${\boldsymbol{5.177 \pm 0.953}}$ & \cellcolor{myblue!20}${\boldsymbol{60.97\% \pm 3.94\%}}$ & \cellcolor{myblue!20}${\boldsymbol{36.45\% \pm 2.85\%}}$ \\
\midrule
% \multirow{2}{*}{{ADV-0 (w/o IPL)}} 
%  & $\checkmark$ & ${0.806 \pm 0.184}$ & ${0.907 \pm 0.300}$ & ${2.074 \pm 0.205}$ & ${6.014 \pm 1.103}$ & ${79.26\% \pm 5.14\%}$ & ${48.06\% \pm 4.33\%}$ \\
%  & $\times$ & ${1.291 \pm 0.598}$ & ${0.873 \pm 0.333}$ & ${2.018 \pm 0.258}$ & ${6.144 \pm 1.090}$ & ${76.90\% \pm 4.90\%}$ & ${52.26\% \pm 3.46\%}$ \\
% \midrule
\multirow{2}{*}{{CAT}} 
 & $\checkmark$ & ${0.825 \pm 0.192}$ & ${1.017 \pm 0.317}$ & ${1.875 \pm 0.229}$ & ${5.422 \pm 0.833}$ & ${83.04\% \pm 5.99\%}$ & ${44.80\% \pm 6.12\%}$ \\
 & $\times$ & ${1.415 \pm 0.462}$ & ${0.980 \pm 0.327}$ & ${1.979 \pm 0.193}$ & ${6.002 \pm 0.931}$ & ${75.36\% \pm 7.86\%}$ & ${51.20\% \pm 4.74\%}$ \\
\midrule
\multirow{2}{*}{{Heuristic}} 
 & $\checkmark$ & ${0.864 \pm 0.200}$ & $\boldsymbol{1.150 \pm 0.467}$ & ${2.129 \pm 0.237}$ & ${6.542 \pm 0.942}$ & ${74.46\% \pm 5.98\%}$ & ${48.79\% \pm 3.65\%}$ \\
 & $\times$ & $\boldsymbol{1.601 \pm 0.538}$ & ${1.093 \pm 0.450}$ & ${2.199 \pm 0.218}$ & ${6.778 \pm 1.066}$ & ${68.82\% \pm 9.05\%}$ & ${52.10\% \pm 1.77\%}$ \\
\midrule
\multirow{2}{*}{{Replay}} 
 & $\checkmark$ & ${0.587 \pm 0.185}$ & ${0.683 \pm 0.217}$ & ${2.779 \pm 0.431}$ & ${8.265 \pm 1.718}$ & ${89.53\% \pm 4.30\%}$ & ${69.73\% \pm 5.27\%}$ \\
 & $\times$ & ${0.978 \pm 0.462}$ & ${0.730 \pm 0.233}$ & ${2.720 \pm 0.383}$ & ${8.104 \pm 1.537}$ & ${85.13\% \pm 4.58\%}$ & ${68.50\% \pm 4.68\%}$ \\
\bottomrule
\end{tabular}%
}
\end{small}
\end{table*}

\Needspace{20\baselineskip}
\subsection{Improving Policy Robustness in the Long Tail}
\label{subsec:long_tail_exp}

\vspace{-5pt}

\paragraph{Main results.}

\begin{wrapfigure}{l}{0.43\textwidth}
  \vspace{-8pt}
  \centering
  \includegraphics[width=\linewidth]{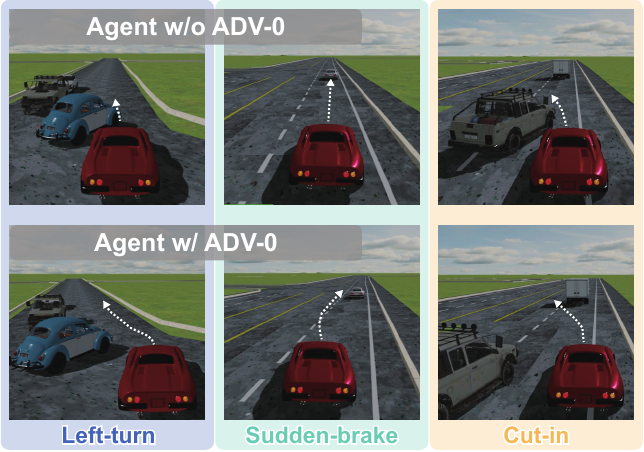}
  \caption{Visualization of improved safe driving ability after being trained with \texttt{ADV-0}. More examples are shown in \cref{fig:case-appendix}.}
  \label{fig:case-main}
  \vspace{-3pt}
\end{wrapfigure}

While \texttt{ADV-0} has demonstrated robustness against generated adversaries, it is crucial to verify whether this performance generalizes to real-world, naturally occurring long-tail events. To this end, we curate unbiased held-out sets mined from real-world WOMD logs, including only high-risk scenarios categorized by extreme safety criticality (e.g., critical TTC, PET) and semantic rarity (e.g., rare behaviors). 
As shown in \cref{tab:long_tail_avg} (detailed in \cref{tab:long_tail}), agents trained via \texttt{ADV-0} demonstrate superior zero-shot robustness compared to baselines.
It achieves the highest safety margins and stability scores, significantly reducing the rates of near misses and RDP violations, indicating that the agent learns to anticipate risks before they become critical, rather than merely reacting to emergencies. 
Examples in Figures~\ref{fig:case-main} and \ref{fig:case-appendix} further visualize these defensive behaviors: \texttt{ADV-0} agent proactively yields to aggressive cut-ins and navigates sudden occlusions where baseline agents fail.

%%%%%%%%%%%%%%%%%%%%%%%%%%%%%%%%%%%%%%%%%%%%%%%%%%%%%%%%%%%%%%%%%%%%%%%%%%%%%%%

\vspace{-5pt}
\subsection{Algorithmic Analysis}

\vspace{-5pt}
\paragraph{Applications to motion planners.} To show the generality of \texttt{ADV-0}, we extend our evaluation beyond RL agents to two kinds of SOTA learning-based trajectory planners: PlanTF \citep{cheng2024rethinking} (multimodal scoring) and SMART \citep{wu2024smart} (autoregressive generation). 
% Unlike end-to-end RL policies that output immediate control actions, they generate future plans, which are then executed via a low-level controller. 
As shown in Tables~\ref{tab:planner_results_avg} and \ref{tab:planner_results_all}, adversarial fine-tuning via \texttt{ADV-0} yields consistent improvements for both architectures. 
We further analyze the internal behavior of the planners using the trajectory-level breakdown in \cref{tab:planner_breakdown}. The results indicate that fine-tuned models learn to prioritize safety constraints significantly more than pretrained priors. This safety improvement comes with a trade-off in efficiency. 
Interestingly, we observe that the performance of these planners remains slightly lower than RL agents discussed previously. 
We attribute this to two factors: (1) the covariate shift in behavior cloning models~\citep{karkus2025beyond}, where they struggle to recover when the adversary forces it into out-of-distribution states; and (2) the latency introduced by the re-planning horizon. Unlike end-to-end RL policies that output immediate control actions, they generate a future trajectory executed by a controller. This delayed reaction limits the ability to react instantaneously to aggressive attacks.

%%%%%%%%%%%%%%%%%%%%%%%%%%%%%%%%%%%%%%%%%%%%%%%%%%%%%%%%%%%%%%%%%%%%%%%%%%%%%%%

% \vspace{-5pt}
\begin{table}[htbp]
\centering
\vspace{-8pt}
\caption{\textbf{Application to learning-based planners}.
Performance comparisons of two SOTA trajectory planning models before and after fine-tuning using \texttt{ADV-0} (GRPO). See \cref{tab:planner_results_all} for details.}
\label{tab:planner_results_avg}
\setlength{\tabcolsep}{2.25pt}
\begin{small}
\renewcommand{\arraystretch}{1.1}
\resizebox{0.92\textwidth}{!}{%
\begin{tabular}{lcccccc}
\toprule
\multirow{2}{*}{\textbf{Metric}} & \multicolumn{3}{c}{\textbf{PlanTF}} & \multicolumn{3}{c}{\textbf{SMART}} \\
\cmidrule(lr){2-4} \cmidrule(lr){5-7}
& Base & \cellcolor{myblue!20}\textbf{+ ADV-0} & \textbf{Rel. Change} & Base & \cellcolor{myblue!20}\textbf{+ ADV-0} &\textbf{ Rel. Change} \\ 
\midrule
RC $\uparrow$ 
& 0.628 $\pm$ 0.025 & \cellcolor{myblue!20}0.674 $\pm$ 0.016 & \cellcolor{myteal!15}\textbf{+7.46\%} 
& 0.587 $\pm$ 0.029 & \cellcolor{myblue!20}0.631 $\pm$ 0.015 & \cellcolor{myteal!15}\textbf{+7.57\%} \\
Crash $\downarrow$ 
& 0.357 $\pm$ 0.031 & \cellcolor{myblue!20}0.263 $\pm$ 0.025 & \cellcolor{myteal!30}\textbf{-26.23\%} 
& 0.396 $\pm$ 0.034 & \cellcolor{myblue!20}0.305 $\pm$ 0.023 & \cellcolor{myteal!30}\textbf{-22.88\%} \\
Reward $\uparrow$ 
& 35.85 $\pm$ 2.51 & \cellcolor{myblue!20}41.98 $\pm$ 1.62 & \cellcolor{myteal!20}\textbf{+17.11\%} 
& 32.66 $\pm$ 2.85 & \cellcolor{myblue!20}37.85 $\pm$ 1.70 & \cellcolor{myteal!20}\textbf{+15.86\%} \\
Cost $\downarrow$ 
& 1.04 $\pm$ 0.04 & \cellcolor{myblue!20}0.77 $\pm$ 0.02 & \cellcolor{myteal!25}\textbf{-25.68\%} 
& 1.15 $\pm$ 0.05 & \cellcolor{myblue!20}0.92 $\pm$ 0.02 & \cellcolor{myteal!25}\textbf{-20.31\%} \\
\bottomrule
\end{tabular}
}
\end{small}
\end{table}

\vspace{-5pt}
\paragraph{Impact of temperature parameter.}
We study the sensitivity of the sampling temperature $\tau$ in Eq.~\ref{eq:sampling}, which modulates the trade-off between adversarial exploitation and exploration. \cref{fig:ablation_temperature} illustrates an obvious trade-off: (1) At extremely low temperatures ($\tau \to 0$), the sampling degenerates into a deterministic hard mode. This leads to suboptimal performance, likely because an overly aggressive attacker overwhelms the defender early in training, without learning generalized robustness.
(2) Conversely, high values ($\tau=5.0$) reduce the adversarial signal, degrading \texttt{ADV-0} to domain randomization. The attacker fails to consistently expose weaknesses. The results indicate that a moderate value ($\tau=0.1$) achieves the best balance.
It introduces sufficient stochasticity to cover diverse long-tail distributions while maintaining enough focus to prioritize high-risk regions, effectively forming an automatic regularization.

%%%%%%%%%%%%%%%%%%%%%%%%%%%%%%%%%%%%%%%%%%%%%%%%%%%%%%%%%%%%%%%%%%%%%%%%%%%%%%%

\vspace{-5pt}
\paragraph{Ablation on return estimator.}
The inner loop relies on the quality of the return estimator $\hat{J}$ used to label preferences. We compare our rule-based proxy against three baselines: \textit{GTReward} (oracle simulation), \textit{Experience} (retrieval from history), and \textit{RewardModel} (learnable neural network). 
\Cref{fig:reward_model} reports the IPL training curves. 
Surprisingly, the rule-based proxy achieves low and stable preference loss comparable to the oracle, outperforming other estimators. 
\cref{fig:proxy_correlation} measures a strong Spearman correlation of $\rho = 0.77$ between the proxy estimates and oracle returns.
This suggests that the geometric proxy provides a high-fidelity and low-variance ranking signal, which is sufficient for adversarial sampling. 
Future work may explore using the Q-network from an Actor-Critic architecture for value estimation to handle more complex interactions. To test candidate selection against a reactive ego policy, we additionally evaluate 12,800 oracle rollouts across five seeds. The proxy achieves 0.80 pairwise accuracy and 0.84 recall of the oracle choice among the top four candidates; its selected attacks yield an ego return of 29.9, compared with 27.6 for oracle selection (\cref{app:reactive_ego}).

%%%%%%%%%%%%%%%%%%%%%%%%%%%%%%%%%%%%%%%%%%%%%%%%%%%%%%%%%%%%%%%%%%%%%%%%%%%%%%%
% \vspace{-5pt}
% \vspace{-5pt}

\begin{figure}[htbp]
  \centering
  \begin{minipage}[t]{0.85\textwidth}
    \centering
    \includegraphics[width=\linewidth]{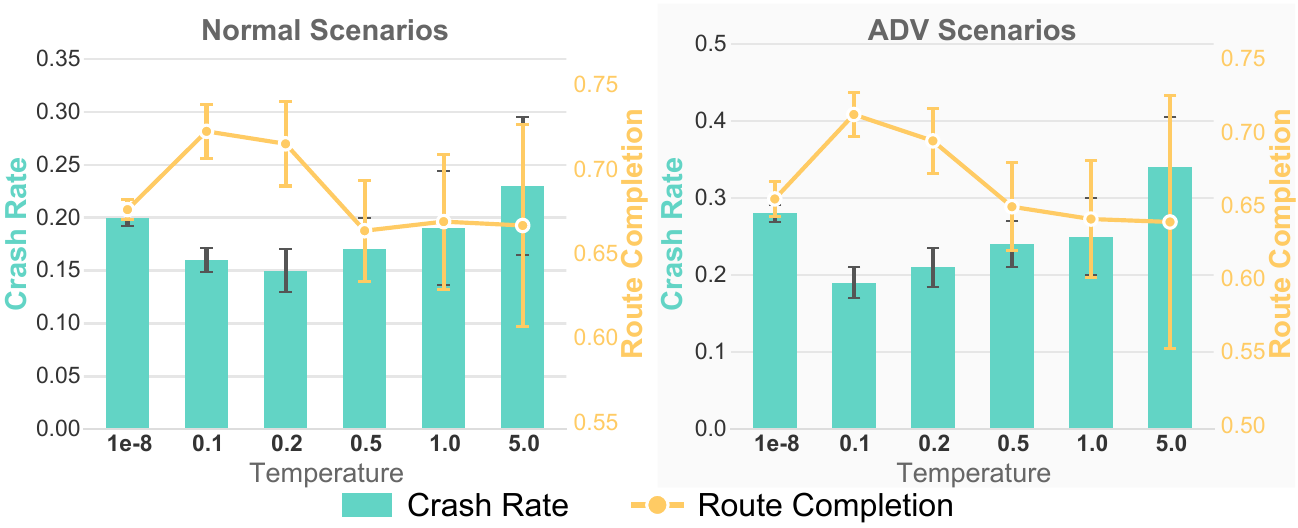}
    \caption{Impacts of the temperature parameter in sampling.}
    \label{fig:ablation_temperature}
  \end{minipage}

  \vspace{5pt}
  \begin{minipage}[t]{0.52\textwidth}
    \centering
    \includegraphics[width=\linewidth]{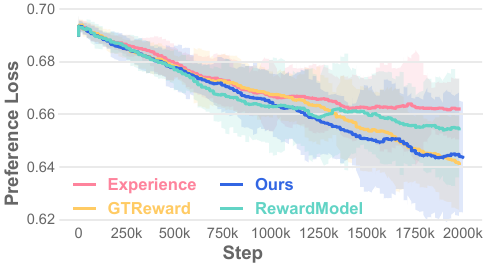}
    \caption{Impacts of different reward calculator schemes.}
    \label{fig:reward_model}
  \end{minipage}
  \hspace{0.035\textwidth}
  \begin{minipage}[t]{0.38\textwidth}
    \centering
    \includegraphics[width=\linewidth]{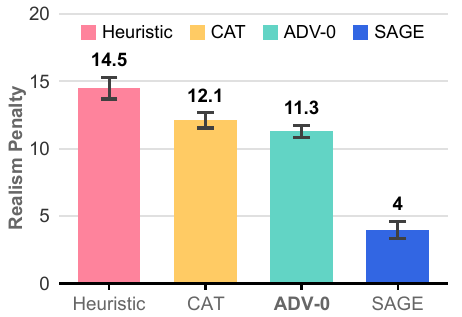}
    \caption{Realism penalty comparison.}
    \label{fig:real_pen}
  \end{minipage}
\end{figure}

\vspace{-8pt}
\section{Conclusion}
\vspace{-5pt}
This paper presents \texttt{ADV-0}, a closed-loop min-max policy optimization framework designed to enhance the robustness of AD models against long-tail risks. 
By formulating the problem as a zero-sum game and solving it via an alternating end-to-end training pipeline, we bridge the gap between adversarial generation and robust policy learning. 
Our theoretical analysis characterizes the idealized regularized game and relates the practical training objective to a lower bound on real-world long-tail performance. 
Empirical results suggest that \texttt{ADV-0} not only generates effective safety-critical scenarios but also improves the generalizability of both RL agents and motion planners against diverse long-tail risks.
Despite the promising results, several limitations exist. (1) The reliance on high-fidelity simulators for online RL training limits scalability.
(2) Extending \texttt{ADV-0} to vision-based sensor inputs may require differentiable neural rendering, which remains a non-trivial challenge. 
% (3) While the proxy reward estimator is efficient, it may not fully capture complex, long-horizon causal relationships in scenarios where early subtle actions lead to delayed failures. 
Future work could explore offline RL techniques \citep{karkus2025beyond} to improve training efficiency and vision-based adversarial generation.

% % Acknowledgements should only appear in the accepted version.
% \section*{Acknowledgements}

% \textbf{Do not} include acknowledgements in the initial version of the paper
% submitted for blind review.

% If a paper is accepted, the final camera-ready version can (and usually should)
% include acknowledgements.  Such acknowledgements should be placed at the end of
% the section, in an unnumbered section that does not count towards the paper
% page limit. Typically, this will include thanks to reviewers who gave useful
% comments, to colleagues who contributed to the ideas, and to funding agencies
% and corporate sponsors that provided financial support.

\begingroup
\section*{Acknowledgments and Disclosures}
This work was supported by the National Natural Science Foundation of China (524B2164, 52125208, and 52232015) and the Otto Poon Charitable Foundation Smart Cities Research Institute, The Hong Kong Polytechnic University (CD06).
The authors declare no competing interests.

\par\endgroup

\clearpage
\bibliographystyle{plainnat}
\bibliography{reference}

@article{karkus2025beyond,
  title={Beyond Behavior Cloning in Autonomous Driving: a Survey of Closed-Loop Training Techniques},
  author={Karkus, Peter and Igl, Maximilian and Chen, Yuxiao and Chitta, Kashyap and Packer, Jef and Douillard, Bertrand and Tian, Ran and Naumann, Alexander and Garcia-Cobo, Guillermo and Tan, Shuhan and others},
  journal={Authorea Preprints},
  year={2025},
  publisher={Authorea}
}

@article{li2022metadrive,
  title={Metadrive: Composing diverse driving scenarios for generalizable reinforcement learning},
  author={Li, Quanyi and Peng, Zhenghao and Feng, Lan and Zhang, Qihang and Xue, Zhenghai and Zhou, Bolei},
  journal={IEEE transactions on pattern analysis and machine intelligence},
  volume={45},
  number={3},
  pages={3461--3475},
  year={2022},
  publisher={IEEE}
}

@inproceedings{gu2021densetnt,
  title={Densetnt: End-to-end trajectory prediction from dense goal sets},
  author={Gu, Junru and Sun, Chen and Zhao, Hang},
  booktitle={Proceedings of the IEEE/CVF international conference on computer vision},
  pages={15303--15312},
  year={2021}
}

@inproceedings{pinto2017robust,
  title={Robust adversarial reinforcement learning},
  author={Pinto, Lerrel and Davidson, James and Sukthankar, Rahul and Gupta, Abhinav},
  booktitle={International conference on machine learning},
  pages={2817--2826},
  year={2017},
  organization={PMLR}
}

@inproceedings{zhang2023cat,
  title={Cat: Closed-loop adversarial training for safe end-to-end driving},
  author={Zhang, Linrui and Peng, Zhenghao and Li, Quanyi and Zhou, Bolei},
  booktitle={Conference on Robot Learning},
  pages={2357--2372},
  year={2023},
  organization={PMLR}
}

@inproceedings{pan2019risk,
  title={Risk averse robust adversarial reinforcement learning},
  author={Pan, Xinlei and Seita, Daniel and Gao, Yang and Canny, John},
  booktitle={2019 International Conference on Robotics and Automation (ICRA)},
  pages={8522--8528},
  year={2019},
  organization={IEEE}
}

@article{nie2025steerable,
  title={Steerable adversarial scenario generation through test-time preference alignment},
  author={Nie, Tong and Mei, Yuewen and Tang, Yihong and He, Junlin and Sun, Jie and Shi, Haotian and Ma, Wei and Sun, Jian},
  journal={arXiv preprint arXiv:2509.20102},
  year={2025}
}

@article{mei2025llm,
  title={Llm-attacker: Enhancing closed-loop adversarial scenario generation for autonomous driving with large language models},
  author={Mei, Yuewen and Nie, Tong and Sun, Jian and Tian, Ye},
  journal={arXiv preprint arXiv:2501.15850},
  year={2025}
}

@article{zhang2020robust,
  title={Robust deep reinforcement learning against adversarial perturbations on state observations},
  author={Zhang, Huan and Chen, Hongge and Xiao, Chaowei and Li, Bo and Liu, Mingyan and Boning, Duane and Hsieh, Cho-Jui},
  journal={Advances in neural information processing systems},
  volume={33},
  pages={21024--21037},
  year={2020}
}

@article{zhang2021robust,
  title={Robust reinforcement learning on state observations with learned optimal adversary},
  author={Zhang, Huan and Chen, Hongge and Boning, Duane and Hsieh, Cho-Jui},
  journal={arXiv preprint arXiv:2101.08452},
  year={2021}
}

@article{vinitsky2020robust,
  title={Robust reinforcement learning using adversarial populations},
  author={Vinitsky, Eugene and Du, Yuqing and Parvate, Kanaad and Jang, Kathy and Abbeel, Pieter and Bayen, Alexandre},
  journal={arXiv preprint arXiv:2008.01825},
  year={2020}
}

@inproceedings{tessler2019action,
  title={Action robust reinforcement learning and applications in continuous control},
  author={Tessler, Chen and Efroni, Yonathan and Mannor, Shie},
  booktitle={International Conference on Machine Learning},
  pages={6215--6224},
  year={2019},
  organization={PMLR}
}

@article{kamalaruban2020robust,
  title={Robust reinforcement learning via adversarial training with langevin dynamics},
  author={Kamalaruban, Parameswaran and Huang, Yu-Ting and Hsieh, Ya-Ping and Rolland, Paul and Shi, Cheng and Cevher, Volkan},
  journal={Advances in Neural Information Processing Systems},
  volume={33},
  pages={8127--8138},
  year={2020}
}

@article{zhang2020stability,
  title={On the stability and convergence of robust adversarial reinforcement learning: A case study on linear quadratic systems},
  author={Zhang, Kaiqing and Hu, Bin and Basar, Tamer},
  journal={Advances in Neural Information Processing Systems},
  volume={33},
  pages={22056--22068},
  year={2020}
}

@article{deng2021deep,
  title={Deep learning-based autonomous driving systems: A survey of attacks and defenses},
  author={Deng, Yao and Zhang, Tiehua and Lou, Guannan and Zheng, Xi and Jin, Jiong and Han, Qing-Long},
  journal={IEEE Transactions on Industrial Informatics},
  volume={17},
  number={12},
  pages={7897--7912},
  year={2021},
  publisher={IEEE}
}

@article{tu2021exploring,
  title={Exploring adversarial robustness of multi-sensor perception systems in self driving},
  author={Tu, James and Li, Huichen and Yan, Xinchen and Ren, Mengye and Chen, Yun and Liang, Ming and Bitar, Eilyan and Yumer, Ersin and Urtasun, Raquel},
  journal={arXiv preprint arXiv:2101.06784},
  year={2021}
}

@article{wang2025generative,
  title={Generative ai for autonomous driving: Frontiers and opportunities},
  author={Wang, Yuping and Xing, Shuo and Can, Cui and Li, Renjie and Hua, Hongyuan and Tian, Kexin and Mo, Zhaobin and Gao, Xiangbo and Wu, Keshu and Zhou, Sulong and others},
  journal={arXiv preprint arXiv:2505.08854},
  year={2025}
}

@article{mei2024bayesian,
  title={Bayesian fault injection safety testing for highly automated vehicles with uncertainty},
  author={Mei, Yuewen and Nie, Tong and Sun, Jian and Tian, Ye},
  journal={IEEE Transactions on Intelligent Vehicles},
  year={2024},
  publisher={IEEE}
}

@article{wachi2019failure,
  title={Failure-scenario maker for rule-based agent using multi-agent adversarial reinforcement learning and its application to autonomous driving},
  author={Wachi, Akifumi},
  journal={arXiv preprint arXiv:1903.10654},
  year={2019}
}

@inproceedings{ma2018improved,
  title={Improved robustness and safety for autonomous vehicle control with adversarial reinforcement learning},
  author={Ma, Xiaobai and Driggs-Campbell, Katherine and Kochenderfer, Mykel J},
  booktitle={2018 IEEE Intelligent Vehicles Symposium (IV)},
  pages={1665--1671},
  year={2018},
  organization={IEEE}
}

@inproceedings{boloor2019simple,
  title={Simple physical adversarial examples against end-to-end autonomous driving models},
  author={Boloor, Adith and He, Xin and Gill, Christopher and Vorobeychik, Yevgeniy and Zhang, Xuan},
  booktitle={2019 IEEE International Conference on Embedded Software and Systems (ICESS)},
  pages={1--7},
  year={2019},
  organization={IEEE}
}

@article{amirkhani2023survey,
  title={A survey on adversarial attacks and defenses for object detection and their applications in autonomous vehicles},
  author={Amirkhani, Abdollah and Karimi, Mohammad Parsa and Banitalebi-Dehkordi, Amin},
  journal={The Visual Computer},
  volume={39},
  number={11},
  pages={5293--5307},
  year={2023},
  publisher={Springer}
}

@inproceedings{tu2020physically,
  title={Physically realizable adversarial examples for lidar object detection},
  author={Tu, James and Ren, Mengye and Manivasagam, Sivabalan and Liang, Ming and Yang, Bin and Du, Richard and Cheng, Frank and Urtasun, Raquel},
  booktitle={Proceedings of the IEEE/CVF conference on computer vision and pattern recognition},
  pages={13716--13725},
  year={2020}
}

@inproceedings{tuncali2018simulation,
  title={Simulation-based adversarial test generation for autonomous vehicles with machine learning components},
  author={Tuncali, Cumhur Erkan and Fainekos, Georgios and Ito, Hisahiro and Kapinski, James},
  booktitle={2018 IEEE intelligent vehicles symposium (IV)},
  pages={1555--1562},
  year={2018},
  organization={IEEE}
}

@article{he2022robust,
  title={Robust lane change decision making for autonomous vehicles: An observation adversarial reinforcement learning approach},
  author={He, Xiangkun and Yang, Haohan and Hu, Zhongxu and Lv, Chen},
  journal={IEEE Transactions on Intelligent Vehicles},
  volume={8},
  number={1},
  pages={184--193},
  year={2022},
  publisher={IEEE}
}

@article{anzalone2022end,
  title={An end-to-end curriculum learning approach for autonomous driving scenarios},
  author={Anzalone, Luca and Barra, Paola and Barra, Silvio and Castiglione, Aniello and Nappi, Michele},
  journal={IEEE Transactions on Intelligent Transportation Systems},
  volume={23},
  number={10},
  pages={19817--19826},
  year={2022},
  publisher={IEEE}
}

@article{tang2025e3ad,
  title={E3AD: An emotion-aware vision-language-action model for human-centric end-to-end autonomous driving},
  author={Tang, Yihong and Liao, Haicheng and Nie, Tong and He, Junlin and Qu, Ao and Chen, Kehua and Ma, Wei and Li, Zhenning and Sun, Lijun and Xu, Chengzhong},
  journal={arXiv preprint arXiv:2512.04733},
  year={2025}
}

@article{kiran2021deep,
  title={Deep reinforcement learning for autonomous driving: A survey},
  author={Kiran, B Ravi and Sobh, Ibrahim and Talpaert, Victor and Mannion, Patrick and Al Sallab, Ahmad A and Yogamani, Senthil and P{\'e}rez, Patrick},
  journal={IEEE transactions on intelligent transportation systems},
  volume={23},
  number={6},
  pages={4909--4926},
  year={2021},
  publisher={IEEE}
}

@inproceedings{isele2018navigating,
  title={Navigating occluded intersections with autonomous vehicles using deep reinforcement learning},
  author={Isele, David and Rahimi, Reza and Cosgun, Akansel and Subramanian, Kaushik and Fujimura, Kikuo},
  booktitle={2018 IEEE international conference on robotics and automation (ICRA)},
  pages={2034--2039},
  year={2018},
  organization={IEEE}
}

@inproceedings{saxena2020driving,
  title={Driving in dense traffic with model-free reinforcement learning},
  author={Saxena, Dhruv Mauria and Bae, Sangjae and Nakhaei, Alireza and Fujimura, Kikuo and Likhachev, Maxim},
  booktitle={2020 IEEE International Conference on Robotics and Automation (ICRA)},
  pages={5385--5392},
  year={2020},
  organization={IEEE}
}

@inproceedings{chen2019model,
  title={Model-free deep reinforcement learning for urban autonomous driving},
  author={Chen, Jianyu and Yuan, Bodi and Tomizuka, Masayoshi},
  booktitle={2019 IEEE intelligent transportation systems conference (ITSC)},
  pages={2765--2771},
  year={2019},
  organization={IEEE}
}

@inproceedings{toromanoff2020end,
  title={End-to-end model-free reinforcement learning for urban driving using implicit affordances},
  author={Toromanoff, Marin and Wirbel, Emilie and Moutarde, Fabien},
  booktitle={Proceedings of the IEEE/CVF conference on computer vision and pattern recognition},
  pages={7153--7162},
  year={2020}
}

@article{schulman2017proximal,
  title={Proximal policy optimization algorithms},
  author={Schulman, John and Wolski, Filip and Dhariwal, Prafulla and Radford, Alec and Klimov, Oleg},
  journal={arXiv preprint arXiv:1707.06347},
  year={2017}
}

@inproceedings{haarnoja2018soft,
  title={Soft actor-critic: Off-policy maximum entropy deep reinforcement learning with a stochastic actor},
  author={Haarnoja, Tuomas and Zhou, Aurick and Abbeel, Pieter and Levine, Sergey},
  booktitle={International conference on machine learning},
  pages={1861--1870},
  year={2018},
  organization={Pmlr}
}

@article{chen2024end,
  title={End-to-end autonomous driving: Challenges and frontiers},
  author={Chen, Li and Wu, Penghao and Chitta, Kashyap and Jaeger, Bernhard and Geiger, Andreas and Li, Hongyang},
  journal={IEEE Transactions on Pattern Analysis and Machine Intelligence},
  year={2024},
  publisher={IEEE}
}

@article{xing2024autotrust,
  title={Autotrust: Benchmarking trustworthiness in large vision language models for autonomous driving},
  author={Xing, Shuo and Hua, Hongyuan and Gao, Xiangbo and Zhu, Shenzhe and Li, Renjie and Tian, Kexin and Li, Xiaopeng and Huang, Heng and Yang, Tianbao and Wang, Zhangyang and others},
  journal={arXiv preprint arXiv:2412.15206},
  year={2024}
}

@article{rafailov2023direct,
  title={Direct preference optimization: Your language model is secretly a reward model},
  author={Rafailov, Rafael and Sharma, Archit and Mitchell, Eric and Manning, Christopher D and Ermon, Stefano and Finn, Chelsea},
  journal={Advances in neural information processing systems},
  volume={36},
  pages={53728--53741},
  year={2023}
}

@article{guo2025deepseek,
  title={Deepseek-r1: Incentivizing reasoning capability in llms via reinforcement learning},
  author={Guo, Daya and Yang, Dejian and Zhang, Haowei and Song, Junxiao and Zhang, Ruoyu and Xu, Runxin and Zhu, Qihao and Ma, Shirong and Wang, Peiyi and Bi, Xiao and others},
  journal={arXiv preprint arXiv:2501.12948},
  year={2025}
}

@article{jiang2025alphadrive,
  title={Alphadrive: Unleashing the power of vlms in autonomous driving via reinforcement learning and reasoning},
  author={Jiang, Bo and Chen, Shaoyu and Zhang, Qian and Liu, Wenyu and Wang, Xinggang},
  journal={arXiv preprint arXiv:2503.07608},
  year={2025}
}

@article{li2025drive,
  title={Drive-R1: Bridging Reasoning and Planning in VLMs for Autonomous Driving with Reinforcement Learning},
  author={Li, Yue and Tian, Meng and Zhu, Dechang and Zhu, Jiangtong and Lin, Zhenyu and Xiong, Zhiwei and Zhao, Xinhai},
  journal={arXiv preprint arXiv:2506.18234},
  year={2025}
}

@article{knox2023reward,
  title={Reward (mis) design for autonomous driving},
  author={Knox, W Bradley and Allievi, Alessandro and Banzhaf, Holger and Schmitt, Felix and Stone, Peter},
  journal={Artificial Intelligence},
  volume={316},
  pages={103829},
  year={2023},
  publisher={Elsevier}
}

@article{li2025finetuning,
  title={Finetuning generative trajectory model with reinforcement learning from human feedback},
  author={Li, Derun and Ren, Jianwei and Wang, Yue and Wen, Xin and Li, Pengxiang and Xu, Leimeng and Zhan, Kun and Xia, Zhongpu and Jia, Peng and Lang, Xianpeng and others},
  journal={arXiv preprint arXiv:2503.10434},
  year={2025}
}

@article{ouyang2022training,
  title={Training language models to follow instructions with human feedback},
  author={Ouyang, Long and Wu, Jeffrey and Jiang, Xu and Almeida, Diogo and Wainwright, Carroll and Mishkin, Pamela and Zhang, Chong and Agarwal, Sandhini and Slama, Katarina and Ray, Alex and others},
  journal={Advances in neural information processing systems},
  volume={35},
  pages={27730--27744},
  year={2022}
}

@article{tian2024tokenize,
  title={Tokenize the world into object-level knowledge to address long-tail events in autonomous driving},
  author={Tian, Ran and Li, Boyi and Weng, Xinshuo and Chen, Yuxiao and Schmerling, Edward and Wang, Yue and Ivanovic, Boris and Pavone, Marco},
  journal={arXiv preprint arXiv:2407.00959},
  year={2024}
}

@article{fang2025corevla,
  title={CoReVLA: A dual-stage end-to-end autonomous driving framework for long-tail scenarios via collect-and-refine},
  author={Fang, Shiyu and Cui, Yiming and Liang, Haoyang and Lv, Chen and Hang, Peng and Sun, Jian},
  journal={arXiv preprint arXiv:2509.15968},
  year={2025}
}

@article{xu2025wod,
  title={Wod-e2e: Waymo open dataset for end-to-end driving in challenging long-tail scenarios},
  author={Xu, Runsheng and Lin, Hubert and Jeon, Wonseok and Feng, Hao and Zou, Yuliang and Sun, Liting and Gorman, John and Tolstaya, Ekaterina and Tang, Sarah and White, Brandyn and others},
  journal={arXiv preprint arXiv:2510.26125},
  year={2025}
}

@article{wang2025alpamayo,
  title={Alpamayo-r1: Bridging reasoning and action prediction for generalizable autonomous driving in the long tail},
  author={Wang, Yan and Luo, Wenjie and Bai, Junjie and Cao, Yulong and Che, Tong and Chen, Ke and Chen, Yuxiao and Diamond, Jenna and Ding, Yifan and Ding, Wenhao and others},
  journal={arXiv preprint arXiv:2511.00088},
  year={2025}
}

@article{feng2021intelligent,
  title={Intelligent driving intelligence test for autonomous vehicles with naturalistic and adversarial environment},
  author={Feng, Shuo and Yan, Xintao and Sun, Haowei and Feng, Yiheng and Liu, Henry X},
  journal={Nature communications},
  volume={12},
  number={1},
  pages={748},
  year={2021},
  publisher={Nature Publishing Group UK London}
}

@article{feng2023dense,
  title={Dense reinforcement learning for safety validation of autonomous vehicles},
  author={Feng, Shuo and Sun, Haowei and Yan, Xintao and Zhu, Haojie and Zou, Zhengxia and Shen, Shengyin and Liu, Henry X},
  journal={Nature},
  volume={615},
  number={7953},
  pages={620--627},
  year={2023},
  publisher={Nature Publishing Group UK London}
}

@article{liu2024curse,
  title={Curse of rarity for autonomous vehicles},
  author={Liu, Henry X and Feng, Shuo},
  journal={nature communications},
  volume={15},
  number={1},
  pages={4808},
  year={2024},
  publisher={Nature Publishing Group UK London}
}

@inproceedings{wang2021advsim,
  title={Advsim: Generating safety-critical scenarios for self-driving vehicles},
  author={Wang, Jingkang and Pun, Ava and Tu, James and Manivasagam, Sivabalan and Sadat, Abbas and Casas, Sergio and Ren, Mengye and Urtasun, Raquel},
  booktitle={Proceedings of the IEEE/CVF Conference on Computer Vision and Pattern Recognition},
  pages={9909--9918},
  year={2021}
}

@inproceedings{kuutti2020training,
  title={Training adversarial agents to exploit weaknesses in deep control policies},
  author={Kuutti, Sampo and Fallah, Saber and Bowden, Richard},
  booktitle={2020 IEEE International Conference on Robotics and Automation (ICRA)},
  pages={108--114},
  year={2020},
  organization={IEEE}
}

@inproceedings{ding2020learning,
  title={Learning to collide: An adaptive safety-critical scenarios generating method},
  author={Ding, Wenhao and Chen, Baiming and Xu, Minjun and Zhao, Ding},
  booktitle={2020 IEEE/RSJ International Conference on Intelligent Robots and Systems (IROS)},
  pages={2243--2250},
  year={2020},
  organization={IEEE}
}

@inproceedings{xu2025diffscene,
  title={Diffscene: Diffusion-based safety-critical scenario generation for autonomous vehicles},
  author={Xu, Chejian and Petiushko, Aleksandr and Zhao, Ding and Li, Bo},
  booktitle={Proceedings of the AAAI Conference on Artificial Intelligence},
  volume={39},
  pages={8797--8805},
  year={2025}
}

@article{stoler2025seal,
  title={Seal: Towards safe autonomous driving via skill-enabled adversary learning for closed-loop scenario generation},
  author={Stoler, Benjamin and Navarro, Ingrid and Francis, Jonathan and Oh, Jean},
  journal={IEEE Robotics and Automation Letters},
  volume={10},
  number={9},
  pages={9320--9327},
  year={2025},
  publisher={IEEE}
}

@inproceedings{hanselmann2022king,
  title={King: Generating safety-critical driving scenarios for robust imitation via kinematics gradients},
  author={Hanselmann, Niklas and Renz, Katrin and Chitta, Kashyap and Bhattacharyya, Apratim and Geiger, Andreas},
  booktitle={European Conference on Computer Vision},
  pages={335--352},
  year={2022},
  organization={Springer}
}

@inproceedings{zhang2022adversarial,
  title={On adversarial robustness of trajectory prediction for autonomous vehicles},
  author={Zhang, Qingzhao and Hu, Shengtuo and Sun, Jiachen and Chen, Qi Alfred and Mao, Z Morley},
  booktitle={Proceedings of the IEEE/CVF Conference on Computer Vision and Pattern Recognition},
  pages={15159--15168},
  year={2022}
}

@article{ding2023survey,
  title={A survey on safety-critical driving scenario generation—a methodological perspective},
  author={Ding, Wenhao and Xu, Chejian and Arief, Mansur and Lin, Haohong and Li, Bo and Zhao, Ding},
  journal={IEEE Transactions on Intelligent Transportation Systems},
  volume={24},
  number={7},
  pages={6971--6988},
  year={2023},
  publisher={IEEE}
}

@inproceedings{zhang2024chatscene,
  title={Chatscene: Knowledge-enabled safety-critical scenario generation for autonomous vehicles},
  author={Zhang, Jiawei and Xu, Chejian and Li, Bo},
  booktitle={Proceedings of the IEEE/CVF Conference on Computer Vision and Pattern Recognition},
  pages={15459--15469},
  year={2024}
}

@inproceedings{ransiek2024goose,
  title={Goose: Goal-conditioned reinforcement learning for safety-critical scenario generation},
  author={Ransiek, Joshua and Plaum, Johannes and Langner, Jacob and Sax, Eric},
  booktitle={2024 IEEE 27th International Conference on Intelligent Transportation Systems (ITSC)},
  pages={2651--2658},
  year={2024},
  organization={IEEE}
}

@inproceedings{zhou2022long,
  title={Long-tail prediction uncertainty aware trajectory planning for self-driving vehicles},
  author={Zhou, Weitao and Cao, Zhong and Xu, Yunkang and Deng, Nanshan and Liu, Xiaoyu and Jiang, Kun and Yang, Diange},
  booktitle={2022 IEEE 25th International Conference on Intelligent Transportation Systems (ITSC)},
  pages={1275--1282},
  year={2022},
  organization={IEEE}
}

@article{liu2025adv,
  title={Adv-BMT: Bidirectional Motion Transformer for Safety-Critical Traffic Scenario Generation},
  author={Liu, Yuxin and Peng, Zhenghao and Cui, Xuanhao and Zhou, Bolei},
  journal={arXiv preprint arXiv:2506.09485},
  year={2025}
}

@article{chen2025rift,
  title={RIFT: Closed-Loop RL Fine-Tuning for Realistic and Controllable Traffic Simulation},
  author={Chen, Keyu and Sun, Wenchao and Cheng, Hao and Zheng, Sifa},
  journal={arXiv preprint arXiv:2505.03344},
  year={2025}
}

@article{wu2024smart,
  title={Smart: Scalable multi-agent real-time motion generation via next-token prediction},
  author={Wu, Wei and Feng, Xiaoxin and Gao, Ziyan and Kan, Yuheng},
  journal={Advances in Neural Information Processing Systems},
  volume={37},
  pages={114048--114071},
  year={2024}
}

@inproceedings{cheng2024rethinking,
  title={Rethinking imitation-based planners for autonomous driving},
  author={Cheng, Jie and Chen, Yingbing and Mei, Xiaodong and Yang, Bowen and Li, Bo and Liu, Ming},
  booktitle={2024 IEEE International Conference on Robotics and Automation (ICRA)},
  pages={14123--14130},
  year={2024},
  organization={IEEE}
}

@inproceedings{scherrer2014approximate,
  title={Approximate policy iteration schemes: A comparison},
  author={Scherrer, Bruno},
  booktitle={International Conference on Machine Learning},
  pages={1314--1322},
  year={2014},
  organization={PMLR}
}

@article{luo2018algorithmic,
  title={Algorithmic framework for model-based deep reinforcement learning with theoretical guarantees},
  author={Luo, Yuping and Xu, Huazhe and Li, Yuanzhi and Tian, Yuandong and Darrell, Trevor and Ma, Tengyu},
  journal={arXiv preprint arXiv:1807.03858},
  year={2018}
}

@inproceedings{achiam2017constrained,
  title={Constrained policy optimization},
  author={Achiam, Joshua and Held, David and Tamar, Aviv and Abbeel, Pieter},
  booktitle={International conference on machine learning},
  pages={22--31},
  year={2017},
  organization={PMLR}
}

@article{brunke2022safe,
  title={Safe learning in robotics: From learning-based control to safe reinforcement learning},
  author={Brunke, Lukas and Greeff, Melissa and Hall, Adam W and Yuan, Zhaocong and Zhou, Siqi and Panerati, Jacopo and Schoellig, Angela P},
  journal={Annual Review of Control, Robotics, and Autonomous Systems},
  volume={5},
  number={1},
  pages={411--444},
  year={2022},
  publisher={Annual Reviews}
}

@article{li2025simulating,
  title={Simulating the Unseen: Crash Prediction Must Learn from What Did Not Happen},
  author={Li, Zihao and Cao, Xinyuan and Gao, Xiangbo and Tian, Kexin and Wu, Keshu and Anis, Mohammad and Zhang, Hao and Long, Keke and Jiang, Jiwan and Li, Xiaopeng and others},
  journal={arXiv preprint arXiv:2505.21743},
  year={2025}
}

@article{dennis2020emergent,
  title={Emergent complexity and zero-shot transfer via unsupervised environment design},
  author={Dennis, Michael and Jaques, Natasha and Vinitsky, Eugene and Bayen, Alexandre and Russell, Stuart and Critch, Andrew and Levine, Sergey},
  journal={Advances in neural information processing systems},
  volume={33},
  pages={13049--13061},
  year={2020}
}

@article{chung2024adversarial,
  title={Adversarial environment design via regret-guided diffusion models},
  author={Chung, Hojun and Lee, Junseo and Kim, Minsoo and Kim, Dohyeong and Oh, Songhwai},
  journal={Advances in Neural Information Processing Systems},
  volume={37},
  pages={63715--63746},
  year={2024}
}

@inproceedings{taranovic2023adversarial,
  title = {Adversarial Imitation Learning with Preferences},
  author = {Taranovic, Aleksandar and Kupcsik, Andras Gabor and Freymuth, Niklas and Neumann, Gerhard},
  booktitle = {The Eleventh International Conference on Learning Representations},
  year = {2023},
  url = {https://openreview.net/forum?id=bhfp5GlDtGe}
}

%%%%%%%%%%%%%%%%%%%%%%%%%%%%%%%%%%%%%%%%%%%%%%%%%%%%%%%%%%%%%%%%%%%%%%%%%%%%%%%
%%%%%%%%%%%%%%%%%%%%%%%%%%%%%%%%%%%%%%%%%%%%%%%%%%%%%%%%%%%%%%%%%%%%%%%%%%%%%%%
% APPENDIX
%%%%%%%%%%%%%%%%%%%%%%%%%%%%%%%%%%%%%%%%%%%%%%%%%%%%%%%%%%%%%%%%%%%%%%%%%%%%%%%
%%%%%%%%%%%%%%%%%%%%%%%%%%%%%%%%%%%%%%%%%%%%%%%%%%%%%%%%%%%%%%%%%%%%%%%%%%%%%%%
\newpage
\appendix
\onecolumn
\startcontents[appendix]
\section*{Appendix}

\begingroup
\setcounter{tocdepth}{3}
\small
\linespread{0.88}\selectfont
\setlength{\parskip}{3pt}
\printcontents[appendix]{}{1}{\section*{Appendix Contents}\vskip-10pt\hrulefill\vskip-15pt}
\endgroup

\bigskip

The appendix provides rigorous theoretical foundations, supplementary experimental results, and detailed implementation specifications that support the main text. We organize the contents as follows: \cref{sec:related work} discusses related work and positions our work in the literature.
\cref{sec:theory_appendix} presents the complete derivations and proofs for the theoretical discussion provided in \cref{ssec:theory}. 
We include additional qualitative visualizations and extended quantitative results in \cref{appendix:results} to complement the main paper.
Finally, \cref{appendix:setups} details the experimental setups, including the dataset, environment, baselines, implementations, and hyperparameters for reproducibility.

%%%%%%%%%%%%%%%%%%%%%%%%%%%%%%%%%%%%%%%%%%%%%%%%%%%%%%%%%%%%%%%
\section{Related Work}\label{sec:related work}
This section reviews the literature relevant to our approach, with a particular focus on autonomous driving (AD). We position our work at the intersection of three interleaving pathways: RL, long-tailed scenario handling, and adversarial learning.

% \vspace{-10pt}
\paragraph{Reinforcement learning.}
RL has been widely studied in AD to enable closed-loop decision-making and address the covariate shift inherent in supervised imitation \citep{kiran2021deep,chen2024end,karkus2025beyond}. Traditional approaches have largely focused on motion planning and continuous control using vectorized state representations \citep{isele2018navigating, saxena2020driving}, or vision-based end-to-end driving within high-fidelity simulators \citep{chen2019model, toromanoff2020end}. These methods typically employ actor-critic algorithms, such as PPO and SAC \citep{schulman2017proximal, haarnoja2018soft}, to maximize cumulative returns based on handcrafted reward functions. However, reward specification and value estimation in complex driving scenarios remain notoriously difficult \citep{knox2023reward, chen2024end}. 
To address the difficulties, recent research has shifted toward alignment techniques emerging from large language models (LLMs). This includes learning from human preferences or feedback to recover reward functions from demonstrations \citep{ouyang2022training}. More notably, critic-free methods, such as Direct Preference Optimization (DPO) \citep{rafailov2023direct} and Group Relative Policy Optimization (GRPO) \citep{guo2025deepseek}, are gaining increasing attention. These methods optimize policies directly against preference data or outcomes across rollouts without training unstable value functions, showing promising results in training end-to-end driving autonomy \citep{jiang2025alphadrive, li2025finetuning, li2025drive}.

% \vspace{-10pt}
\paragraph{Long-tailed scenario.} 
Handling long-tailed events remains a longstanding challenge for AD deployment and system trustworthiness \citep{feng2021intelligent,liu2024curse,xing2024autotrust,chen2024end,wang2025generative}. To mitigate the scarcity of such data in naturalistic driving and provide high-value testing samples, significant effort has been devoted to safety-critical scenario generation \citep{ding2023survey,li2025simulating}. 
Representative approaches range from rule-based \citep{tuncali2018simulation,zhang2024chatscene,mei2024bayesian}, optimization-based \citep{wang2021advsim,hanselmann2022king,zhang2022adversarial,zhang2023cat,nie2025steerable,mei2025llm}, to learning-based methods \citep{ding2020learning,kuutti2020training,feng2023dense,xu2025diffscene,liu2025adv}. Despite their success in identifying failures, effectively integrating these adversarial generation pipelines into closed-loop training remains an open question: the primary goal of existing methods is often to stress-test the system rather than to improve it. 
Conversely, a parallel line of research endeavors to enhance the robustness of decision-making in rare events by designing specialized architectures or leveraging the reasoning capabilities of pretrained LLMs/VLMs \citep{zhou2022long,tian2024tokenize,fang2025corevla,xu2025wod,wang2025alpamayo,tang2025e3ad}. However, adversarial scenario generation and policy improvement are seldom unified in a holistic framework. Consequently, the generalizability of these methods to unseen, open-world long-tailed scenarios remains under-investigated.

%%%%%%%%%%%%%%%%%%%%%%%%%%%%%%%%%%%%%

% \vspace{-10pt}
\paragraph{Adversarial learning for AD.}
Adversarial training offers a principled framework for improvement-targeted generation. In the context of robotics and autonomy, adversarial RL has been extensively discussed for control tasks in constrained settings \citep{pinto2017robust,pan2019risk,tessler2019action,zhang2020robust,zhang2020stability,vinitsky2020robust,kamalaruban2020robust,zhang2021robust}. However, they typically prioritize theoretical analysis within simplified simulation environments with controlled noise, which differs significantly from the complexity of real-world driving.
Within the AD domain, adversarial methods have mainly targeted the robustness of perception and detection modules against observation perturbations \citep{boloor2019simple,tu2020physically,tu2021exploring,deng2021deep,he2022robust,amirkhani2023survey}. In contrast, adversarial training for decision-making, particularly regarding long-tailed scenarios, remains underexplored \citep{ma2018improved,wachi2019failure,anzalone2022end,zhang2023cat,zhang2024chatscene}.
Even among the few works addressing this, the generation and training phases are often decoupled. Crucially, they are typically confined to specific policy types or tested against handcrafted adversarial scenarios, which poses significant challenges to their generalizability across diverse and evolving corner cases.

\paragraph{Environment design and regret-guided generators.}
\texttt{ADV-0} is also related to unsupervised environment design and regret-guided environment generation, where the environment distribution is adapted as the learner improves \citep{tessler2019action,dennis2020emergent,chung2024adversarial}.
The connection is conceptual: both lines seek informative training environments rather than passive evaluation sets.
However, \texttt{ADV-0} differs in its object of optimization and regularization.
Instead of designing generic task parameters or using regret as the adversarial signal, \texttt{ADV-0} optimizes a distribution over naturalistic traffic trajectories conditioned on real driving logs, regularizes this distribution toward a frozen traffic prior $\mathcal{G}_{\text{ref}}$, and uses the defender's closed-loop return as the adversarial utility.
This distinction is important in autonomous driving, where generated scenarios must remain physically plausible while continuously tracking the failure modes of the current policy.

% \clearpage
\section{Theoretical Analysis}\label{sec:theory_appendix}

\subsection{Convergence Analysis}
\label{subsec:convergence}

In this section, we provide a theoretical analysis of the convergence behavior of the idealized \texttt{ADV-0} game. 
We formulate the interaction between the ego agent $\pi_\theta$ and the adversary $\mathcal{G}_\psi$ as a regularized two-player Zero-Sum Markov Game (ZSMG). Our analysis proceeds in two steps: 
(1) We first prove that the inner loop optimization via IPL is mathematically equivalent to solving for the soft optimal adversarial distribution subject to a KL-divergence constraint (Lemma~\ref{lemma:inner_optimality}).
(2) We then elaborate that the alternating soft updates of the defender and the attacker constitute a contraction mapping in the idealized setting, yielding convergence to the Nash Equilibrium of the game (Theorem~\ref{thm:convergence}).

\subsubsection{Preliminaries: ADV-0 as a regularized zero-sum Markov game}
Formally, we define the discounted ZSMG between the ego and the adversary by the tuple $(\mathcal{S}, \mathcal{A}, \mathcal{Y}, \Pi, \Psi, \mathcal{P}, \mathcal{R}, \gamma)$, where:
\begin{itemize}
    \item The \textit{defender} (ego agent) chooses a policy $\pi_\theta\in\Pi: \mathcal{S} \to \Delta(\mathcal{A})$ to maximize its expected return.
    \item The \textit{attacker} (adversary) chooses a generative policy $\mathcal{G}_\psi\in\Psi: \mathcal{X} \to \Delta(\mathcal{Y})$ to produce adversarial trajectories $Y^\text{Adv} \in \mathcal{Y}$, which in turn perturbs the transition dynamics $\mathcal{P}_\psi$. The attacker's goal is to minimize the defender's return.
\end{itemize}
Let the value function $V^{\pi, \psi}(s)$ represent the expected return of the ego agent under the dynamics induced by adversary. The robust optimization objective (Eq.~\ref{eq:min_max}) with the KL-constraint (Eq.~\ref{eq:constrained_opt}, to maintain naturalistic priors) is formulated as finding the saddle point of the regularized value function:
\begin{equation}
    \label{eq:game_objective}
    \max_{\pi} \min_{\mathcal{G}_\psi} \mathcal{J}(\pi, \mathcal{G}_\psi) := \mathbb{E}_{X}\left[ \mathbb{E}_{\tau \sim \pi, \mathcal{P}_\psi} \left[ \sum_{t=0}^T \gamma^t \mathcal{R}(s_t, a_t) \right] + \tau \mathbb{D}_{\text{KL}}(\mathcal{G}_\psi(\cdot|X) || \mathcal{G}_{\text{ref}}(\cdot|X)) \right],
\end{equation}
where $\mathcal{P}_\psi$ denotes the transition dynamics modulated by the adversary's trajectory $Y \sim \mathcal{G}_\psi(\cdot|X)$, $\tau > 0$ controls the regularization strength. Note that the adversary seeks to minimize the ego's return subject to staying close to the prior $\mathcal{G}_{\text{ref}}$.

The ultimate goal is to find converged policies $(\pi^*, \psi^*)$ satisfy the saddle-point inequality condition of a Nash Equilibrium:
\begin{equation}
    \mathcal{J}(\pi_\theta, \mathcal{G}_{\psi^*}) \leq \mathcal{J}(\pi^*, \mathcal{G}_{\psi^*}) \leq \mathcal{J}(\pi^*, \mathcal{G}_\psi), \quad \forall \pi_\theta \in \Pi, \mathcal{G}_\psi \in \Psi,
\end{equation}
implying that neither the defender nor the attacker can unilaterally improve their objective.

\subsubsection{Inner loop optimality via implicit reward}
We first analyze the inner loop of Algorithm~\ref{alg:adv0}, where the adversary's policy $\mathcal{G}_\psi$ is updated while the ego policy $\pi_\theta$ is held fixed. 
The core of the inner loop is the IPL objective. We will show that minimizing the IPL loss (Eq.~\ref{eq:ipl_loss}) is equivalent to finding the optimal adversarial policy that solves the KL-regularized reward maximization problem.

Consider the inner loop objective defined in Eqs.~\ref{eq:constrained_opt} and~\ref{eq:rl_equivalence}. For a fixed ego policy $\pi_\theta$, the adversary seeks to find an optimal policy $\psi^*$ that maximizes the expected risk (minimizes ego return) while remaining close to the reference prior $\mathcal{G}_{\text{ref}}$. Let the reward for the adversary be defined as $r(Y) = -J(\pi_\theta, Y)$. The objective is:
\begin{equation}
    \label{eq:theory_inner_obj}
    \max_{\psi} \mathcal{J}_{\text{inner}}(\mathcal{G}_\psi) = \mathbb{E}_{Y \sim \mathcal{G}_\psi(\cdot|X)} \left[ r(Y) - \tau \log \frac{\mathcal{G}_\psi(Y|X)}{\mathcal{G}_{\text{ref}}(Y|X)} \right].
\end{equation}

\begin{tcolorbox}[colback=myblue!8, colframe=myblue!8, boxrule=0pt, sharp corners=all, boxsep=0pt, left=5pt, right=5pt, top=5pt, bottom=5pt, before skip=2pt, after skip=2pt]
\begin{lemma}[Closed-form optimality of the Gibbs adversary]
\label{lemma:inner_optimality}
For a fixed defender $\pi_\theta$ and a reference adversary $\mathcal{G}_{\text{ref}}$, the global optimum $\mathcal{G}^*$ of the KL-constrained objective in Eq.~\ref{eq:theory_inner_obj} is given by the Gibbs distribution:
\begin{equation}
    \label{eq:optimal_psi}
    \mathcal{G}^*(Y|X) = \frac{1}{Z(X)} \mathcal{G}_{\text{ref}}(Y|X) \exp\left( -\frac{1}{\tau} J(\pi_\theta, Y) \right),
\end{equation}
where $Z(X)$ is the partition function. Furthermore, minimizing the IPL loss $\mathcal{L}_{\text{IPL}}$ (Eq.~\ref{eq:ipl_loss}) is equivalent to performing maximum likelihood estimation on this optimal policy $\mathcal{G}^*$.
\end{lemma}
\end{tcolorbox}

\begin{proof}
The proof consists of two parts. First, we derive the closed-form optimal adversarial policy. 
Let $r(Y) = -J(\pi_\theta, Y)$ be the reward for the adversary. Following the derivation in \citet{rafailov2023direct}, we express the objective using the Gibbs inequality. The objective can be rewritten as maximizing:
\begin{align}
    \mathcal{J}_{\text{inner}}(\mathcal{G}) &= \mathbb{E}_{Y \sim \mathcal{G}} \left[r(Y) - \tau \log \frac{\mathcal{G}(Y|X)}{\mathcal{G}_{\text{ref}}(Y|X)}\right] \nonumber \\
    &= \tau \mathbb{E}_{Y \sim \mathcal{G}} \left[ \log \exp(\frac{1}{\tau} r(Y)) + \log \mathcal{G}_{\text{ref}}(Y|X) - \log \mathcal{G}(Y|X) \right] \nonumber \\
    &= \tau \mathbb{E}_{Y \sim \mathcal{G}} \left[ \log \left( \mathcal{G}_{\text{ref}}(Y|X) \exp\left(\frac{r(Y)}{\tau}\right) \right) - \log \mathcal{G}(Y|X) \right].
\end{align}
Let $Z(X) = \int \mathcal{G}_{\text{ref}}(y|X) \exp(r(y)/\tau) dy$ be the partition function. We can introduce $\log Z(X)$ into the expectation:
\begin{align}
    \mathcal{J}_{\text{inner}}(\mathcal{G}) &= \tau \mathbb{E}_{Y \sim \mathcal{G}} \left[ \log \left( \frac{1}{Z(X)} \mathcal{G}_{\text{ref}}(Y|X) \exp\left(\frac{r(Y)}{\tau}\right) \right) - \log \mathcal{G}(Y|X) + \log Z(X) \right] \nonumber \\
    &= -\tau \mathbb{D}_{\text{KL}}(\mathcal{G}(\cdot|X) || \mathcal{G}^*(\cdot|X)) + \tau \log Z(X),
\end{align}
where $\mathcal{G}^*(Y|X) = \frac{1}{Z(X)} \mathcal{G}_{\text{ref}}(Y|X) \exp\left( \frac{r(Y)}{\tau} \right)$. Since $\mathbb{D}_{\text{KL}} \ge 0$, the objective is maximized when $\mathcal{G} = \mathcal{G}^*$, proving the first part of the lemma.
This identifies the population-level optimal adversarial distribution. However, explicitly evaluating $Z(X)$ and sampling exactly from this distribution are intractable in high-dimensional trajectory space.

We next show that minimizing the IPL loss $\mathcal{L}_{\text{IPL}}$ with respect to $\psi$ is consistent with maximizing the likelihood of the preference data generated by the optimal adversary $\mathcal{G}^*$, without the need to estimate the partition function. We show the equivalence following \citet{rafailov2023direct}. We can invert the optimal policy equation to rewrite the reward as:
\begin{equation}
    r(Y) = \tau \log \frac{\mathcal{G}^*(Y|X)}{\mathcal{G}_{\text{ref}}(Y|X)} + \tau \log Z(X).
\end{equation}
Under the Bradley-Terry preference model, the probability that trajectory $Y_w$ is preferred over $Y_l$ (i.e., $Y_w$ induces lower ego return) is given by $P(Y_w \succ Y_l) = \sigma(r(Y_w) - r(Y_l))$. Substituting the reparameterized reward into this model:
\begin{align}
    P(Y_w \succ Y_l) &= \sigma\left( \left( \tau \log \frac{\mathcal{G}^*(Y_w|X)}{\mathcal{G}_{\text{ref}}(Y_w|X)} + \tau \log Z \right) - \left( \tau \log \frac{\mathcal{G}^*(Y_l|X)}{\mathcal{G}_{\text{ref}}(Y_l|X)} + \tau \log Z \right) \right) \nonumber \\
    &= \sigma\left( \tau \log \frac{\mathcal{G}^*(Y_w|X)}{\mathcal{G}_{\text{ref}}(Y_w|X)} - \tau \log \frac{\mathcal{G}^*(Y_l|X)}{\mathcal{G}_{\text{ref}}(Y_l|X)} \right).
\end{align}
The partition function $Z(X)$ cancels out. The IPL loss (Eq.~\ref{eq:ipl_loss}) is exactly the negative log-likelihood of this probability with the parameterized policy $\mathcal{G}_\psi$ approximating $\mathcal{G}^*$. Thus, minimizing $\mathcal{L}_{\text{IPL}}$ is equivalent to fitting the optimal adversarial policy $\mathcal{G}^*$ consistent with the observed preferences.
This proves that the inner loop of \texttt{ADV-0} effectively solves the constrained optimization problem in Eq.~\ref{eq:constrained_opt}. 
\end{proof}

%%%%%%%%%%%%%%%%%%%%%%%%%%%%%%%%%%%%%%%%%%%%%%%%%%%%%%%%%%%%%%%
\subsubsection{Global convergence to Nash Equilibrium}
% \label{subsec:theory_convergence}

Having established in Lemma~\ref{lemma:inner_optimality} the population-level optimal adversarial distribution $\mathcal{G}_{\psi^*}$, we now analyze the convergence of the global alternating optimization. We show that the idealized \texttt{ADV-0} game can be viewed as optimizing a specific robust Bellman Operator, which yields convergence to a unique Nash Equilibrium under exact updates.

Before beginning the formal derivation, we first establish the following Lemma:

\begin{tcolorbox}[colback=myblue!8, colframe=myblue!8, boxrule=0pt, sharp corners=all, boxsep=0pt, left=5pt, right=5pt, top=5pt, bottom=5pt, before skip=2pt, after skip=2pt]
\begin{lemma}[Non-expansiveness of Soft-Min]
\label{lemma:lse_contraction}
Let $f_\tau(X) \triangleq -\tau \log \mathbb{E}_{Y} [\exp(-X(Y)/\tau)]$ be the Soft-Min operator over a random variable $Y$ with temperature $\tau > 0$. For any two bounded functions $X_1, X_2$, the following inequality holds:
\begin{equation}
    |f_\tau(X_1) - f_\tau(X_2)| \leq \max_Y |X_1(Y) - X_2(Y)|.
\end{equation}
\end{lemma}
\end{tcolorbox}

\begin{proof}
Let $\Delta = \max_Y |X_1(Y) - X_2(Y)|$. By definition, for all $Y$, the difference is bounded by:
\begin{equation}
    X_2(Y) - \Delta \leq X_1(Y) \leq X_2(Y) + \Delta.
\end{equation}
Multiplying by $-1/\tau$ and exponentiating yields:
\begin{equation}
    \exp\left(\frac{-X_2(Y) - \Delta}{\tau}\right) \leq \exp\left(\frac{-X_1(Y)}{\tau}\right) \leq \exp\left(\frac{-X_2(Y) + \Delta}{\tau}\right).
\end{equation}
Taking the expectation $\mathbb{E}_Y$ preserves the inequality. We can factor out the constant terms $\exp(\pm \Delta/\tau)$:
\begin{equation}
    e^{-\Delta/\tau} \mathbb{E}_Y [e^{-X_2(Y)/\tau}] \leq \mathbb{E}_Y [e^{-X_1(Y)/\tau}] \leq e^{\Delta/\tau} \mathbb{E}_Y [e^{-X_2(Y)/\tau}].
\end{equation}
Next, we apply the strictly decreasing function $-\tau \log(\cdot)$ to all sides. We have:
\begin{equation}
    -\tau \log \left( e^{\Delta/\tau} \mathbb{E} [e^{-X_2/\tau}] \right) \leq -\tau \log \mathbb{E} [e^{-X_1/\tau}] \leq -\tau \log \left( e^{-\Delta/\tau} \mathbb{E} [e^{-X_2/\tau}] \right).
\end{equation}
Using the definition of $f_\tau$ and expanding the terms:
\begin{equation}
    f_\tau(X_2) - \Delta \leq f_\tau(X_1) \leq f_\tau(X_2) + \Delta.
\end{equation}
Subtracting $f_\tau(X_2)$ from all sides, we obtain:
\begin{equation}
    -\Delta \leq f_\tau(X_1) - f_\tau(X_2) \leq \Delta.
\end{equation}
This is equivalent to $|f_\tau(X_1) - f_\tau(X_2)| \leq \Delta$.
\end{proof}

\begin{tcolorbox}[colback=myblue!8, colframe=myblue!8, boxrule=0pt, sharp corners=all, boxsep=0pt, left=5pt, right=5pt, top=5pt, bottom=5pt, before skip=2pt, after skip=2pt]
\begin{theorem}[Contraction and convergence to Nash Equilibrium]
\label{thm:convergence}
Let $\mathcal{V}$ be the space of bounded value functions equipped with the $L_\infty$-norm. The Soft-Robust Bellman Operator $\mathcal{T}_{\text{rob}}: \mathcal{V} \to \mathcal{V}$, defined as:
\begin{equation}
    (\mathcal{T}_{\text{rob}} V)(s) \triangleq \max_{a \in \mathcal{A}} \min_{\mathcal{G}_\psi \in \Psi} \left( \mathcal{R}(s, a) + \gamma \mathbb{E}_{\substack{Y \sim \mathcal{G}_\psi \\ s' \sim \mathcal{P}(\cdot|s,a,Y)}}[V(s')] + \tau \mathbb{D}_{\text{KL}}(\mathcal{G}_\psi || \mathcal{G}_{\text{ref}}) \right),
\end{equation}
is a $\gamma$-contraction mapping. Specifically, for any two value functions $V, U \in \mathcal{V}$, the following inequality holds:
\begin{equation}
    \| \mathcal{T}_{\text{rob}} V - \mathcal{T}_{\text{rob}} U \|_\infty \leq \gamma \| V - U \|_\infty.
\end{equation}
Consequently, the idealized iterative updates converge to a unique fixed point $V^*$. This fixed point corresponds to the value of the unique Nash Equilibrium $(\pi^*, \mathcal{G}_{\psi^*})$ of the regularized zero-sum game, satisfying the saddle-point inequality:
\begin{equation}
    \mathcal{J}_{\tau}(\pi, \mathcal{G}_{\psi^*}) \leq \mathcal{J}_{\tau}(\pi^*, \mathcal{G}_{\psi^*}) \leq \mathcal{J}_{\tau}(\pi^*, \mathcal{G}_\psi), \quad \forall \pi \in \Pi, \forall \mathcal{G}_\psi \in \Psi,
\end{equation}
where $\mathcal{J}_{\tau}(\pi, \mathcal{G}_\psi) \triangleq \mathbb{E}_{\pi, \mathcal{G}_\psi}[\sum \gamma^t \mathcal{R}(s_t, a_t)] + \tau \mathbb{E}_{\pi}[\sum \gamma^t \mathbb{D}_{\text{KL}}(\mathcal{G}_\psi || \mathcal{G}_{\text{ref}})]$ is the regularized cumulative objective.
\end{theorem}
\end{tcolorbox}

\begin{proof}
First, we define the soft-robust Bellman operator $\mathcal{T}_{\text{rob}}$ acting on the value function $V \in \mathbb{R}^{|\mathcal{S}|}$:
\begin{equation}
    \label{eq:robust_operator}
    (\mathcal{T}_{\text{rob}} V)(s) = \max_{a \in \mathcal{A}} \min_{\mathcal{G}_\psi} \mathbb{E}_{Y \sim \mathcal{G}_\psi} \left[ \mathcal{R}(s, a) + \gamma \mathbb{E}_{s' \sim \mathcal{P}(s, a, Y)}[V(s')] + \tau \log \frac{\mathcal{G}_\psi(Y|X)}{\mathcal{G}_{\text{ref}}(Y|X)} \right].
\end{equation}
This operator represents one step of optimal decision-making by the ego agent against a worst-case adversary that is regularized by the KL-divergence.
Note that the inner minimization in $\mathcal{T}_{\text{rob}}$ corresponds exactly to the dual of the maximization problem in Lemma~\ref{lemma:inner_optimality} (due to the zero-sum sign flip).

Next, let $V_1, V_2 \in \mathbb{R}^{|\mathcal{S}|}$ be two arbitrary bounded value functions. We aim to show $\|\mathcal{T}_{\text{rob}} V_1 - \mathcal{T}_{\text{rob}} V_2\|_\infty \leq \gamma \|V_1 - V_2\|_\infty$.
We simplify the inner minimization problem. Let $\mathcal{Q}_V(s, a, Y) = \mathcal{R}(s, a) + \gamma \mathbb{E}_{s' \sim \mathcal{P}(\cdot|s, a, Y)}[V(s')]$. Using the closed-form solution derived in Lemma~\ref{lemma:inner_optimality}, the inner minimization over $\mathcal{G}_\psi$ is equivalent to a \texttt{LogSumExp} (Soft-Min) function. We define the smoothed value $\Omega_V(s, a)$ as:
\begin{equation}
    \begin{aligned}
    \Omega_V(s, a) &\triangleq \min_{\mathcal{G}_\psi} \mathbb{E}_{Y \sim \mathcal{G}_\psi} \left[ \mathcal{Q}_V(s, a, Y) + \tau \log \frac{\mathcal{G}_\psi(Y|X)}{\mathcal{G}_{\text{ref}}(Y|X)} \right] \\
    &= -\tau \log \mathbb{E}_{Y \sim \mathcal{G}_{\text{ref}}} \left[ \exp\left( -\frac{\mathcal{Q}_V(s, a, Y)}{\tau} \right) \right].
    \end{aligned}
\end{equation}
Thus, the operator simplifies to $(\mathcal{T}_{\text{rob}} V)(s) = \max_{a} \Omega_V(s, a)$.
Consider the difference for any state $s$:
\begin{align}\label{eq:max_ineq}
    |(\mathcal{T}_{\text{rob}} V_1)(s) - (\mathcal{T}_{\text{rob}} V_2)(s)| &= |\max_a \Omega_{V_1}(s, a) - \max_a \Omega_{V_2}(s, a)| \nonumber \\
    &\leq \max_a |\Omega_{V_1}(s, a) - \Omega_{V_2}(s, a)|.
\end{align}
where Eq.~\ref{eq:max_ineq} follows from the non-expansiveness of the max operator (i.e., $|\max f - \max g| \le \max |f-g|$). 
Next, we apply Lemma~\ref{lemma:lse_contraction} to the soft-min term $\Omega_V$. By identifying $X(Y)$ with $Q_V(s, a, Y)$, we obtain:
% We utilize the non-expansiveness property of the \texttt{LogSumExp} function: $|\log \sum x_i - \log \sum y_i| \leq \max_i |\log x_i - \log y_i|$. Applying this to the expectation term:
\begin{align}
    |\Omega_{V_1}(s, a) - \Omega_{V_2}(s, a)| &= \left| -\tau \log \frac{\mathbb{E}_{Y} [\exp(-\mathcal{Q}_{V_1}(s, a, Y)/\tau)]}{\mathbb{E}_{Y} [\exp(-\mathcal{Q}_{V_2}(s, a, Y)/\tau)]} \right| \nonumber \\
    &\leq \max_Y \left| -\tau \left( -\frac{\mathcal{Q}_{V_1}(s, a, Y)}{\tau} + \frac{\mathcal{Q}_{V_2}(s, a, Y)}{\tau} \right) \right| \nonumber \\
    &= \max_Y |\mathcal{Q}_{V_1}(s, a, Y) - \mathcal{Q}_{V_2}(s, a, Y)|.
\end{align}
Substituting the definition of $Q_V$ and expanding the expectation:
\begin{align}
    &= \max_Y \left| \gamma \mathbb{E}_{s' \sim \mathcal{P}(\cdot|s,a,Y)} [V_1(s') - V_2(s')] \right| \nonumber \\
    &\leq \gamma \max_Y \mathbb{E}_{s'} [|V_1(s') - V_2(s')|] \quad \text{(Jensen's inequality)} \nonumber \\
    &\leq \gamma \|V_1 - V_2\|_\infty. \quad \text{(Definition of $L_\infty$-norm)}
\end{align}
The last two steps utilize the convexity of the absolute value function and bound the local error by the global supremum norm.
Since $\gamma \in (0, 1)$, $\mathcal{T}_{\text{rob}}$ is a $\gamma$-contraction mapping. By Banach's Fixed Point Theorem, there exists a unique value $V^*$ such that $\mathcal{T}_{\text{rob}} V^* = V^*$.

Finally, we connect this to the alternating updates in Algorithm~\ref{alg:adv0} under the idealized setting of exact inner optimization and exact policy improvement. 
The population-level procedure performs generalized policy iteration. The inner loop solves for the optimal soft adversary under the KL-regularized objective, effectively evaluating $\Omega_{V^{\pi}}(s,a)$. The outer loop performs policy improvement on the induced robust value function. 
\citet{tessler2019action} (Theorem 3) and \citet{scherrer2014approximate} demonstrate that soft policy iteration converges to the optimal value if the policy improvement step is a contraction. 
Since $\mathcal{T}_{\text{rob}}$ is a contraction, the idealized sequence of policies $(\pi_k, \psi_k)$ converges to the Nash Equilibrium $(\pi^*, \psi^*)$ where $\pi^*$ is optimal against the worst-case regularized adversary $\psi^*$. 
\end{proof}

%%%%%%%%%%%%%%%%%%%%%%%%%%%%%%%%%%%%%%%%%%%%%%%%%%%%%%%%%%%%%%%

\subsection{Generalization Bound and Safety Guarantees}
\label{subsec:bound}

In this section, we provide a theoretical justification of \texttt{ADV-0} in its generalizability and safety guarantees. 
We aim to answer a fundamental question: \textit{Does optimizing the policy against a generated adversarial distribution guarantee performance and safety in the real-world long-tail distribution?}
Different from the previous section, we now model the interaction between the ego policy $\pi_\theta$ and the adversarial environment as a problem of policy optimization under dynamical uncertainty. Our goal is to show that optimizing the policy under the adversarial dynamics $\mathcal{P}_\psi$ (subject to the KL-constraint in the inner loop) improves a lower bound on the performance in the target real-world long-tail distribution $\mathcal{P}_{\text{real}}$, leading to a \textbf{generalization bound} (Theorem \ref{thm:generalization}) and a \textbf{safety guarantee} (Theorem \ref{thm:safety}). Our analysis builds upon the \textit{Simulation Lemma} from \citet{luo2018algorithmic} and trust-region bounds from \citet{achiam2017constrained}.

\paragraph{Preliminaries.}
Simply let $M = (\mathcal{S}, \mathcal{A}, \mathcal{R}, \gamma)$ be the shared components of the MDP. We consider two transition dynamics:
\begin{itemize}
    \item $\mathcal{P}_{\text{real}}(s'|s,a)$: The true, unknown long-tail dynamics of the real world.
    \item $\mathcal{P}_{\psi}(s'|s,a)$: The adversarial dynamics induced by the generator $\mathcal{G}_\psi(\cdot|X)$.
\end{itemize}
In our context, the state transition is deterministic given the background traffic trajectories. Let $Y$ denote the joint trajectory of background agents. The transition function can be written as $s' = f(s, a, Y)$. Thus, the stochasticity in dynamics comes entirely from the distribution of $Y$.
Let $V^{\pi, \mathcal{P}}(s) = \mathbb{E}_{\tau \sim \pi, \mathcal{P}}[\sum_{t=0}^\infty \gamma^t \mathcal{R}(s_t, a_t) | s_0=s]$ be the value function. 
The expected return is $J(\pi, \mathcal{P}) = \mathbb{E}_{s_0 \sim \rho_0}[V^{\pi, \mathcal{P}}(s_0)]$.
Similarly, let $J_C(\pi, \mathcal{P})$ denote the expected cumulative safety cost, where $C(s,a) \in [0, C_{\max}]$ is a safety cost function (e.g., collision indicator).
We assume the reward is bounded by $R_{\max}$, implying the value function is bounded by $V_{\max} = \frac{R_{\max}}{1-\gamma}$.

% \textbf{Assumption 1 (Lipschitz Continuity).} We assume the value function $V^{\pi, \mathcal{P}_{\text{real}}}$ is $L$-Lipschitz continuous with respect to the state space metric induced by the dynamics. That is, for any states $s, s'$, $|V^{\pi, \mathcal{P}_{\text{real}}}(s) - V^{\pi, \mathcal{P}_{\text{real}}}(s')| \le L \|s - s'\|_2$.

\subsubsection{Discrepancy under shifted dynamics}

To analyze performance generalization, we first quantify the discrepancy in expected return resulting from the shift from real dynamics to adversarial dynamics. We invoke the Simulation Lemma \citep{achiam2017constrained} and adapt here to quantify the gap between the adversarial training environment and the real world.

\begin{tcolorbox}[colback=myblue!8, colframe=myblue!8, boxrule=0pt, sharp corners=all, boxsep=0pt, left=5pt, right=5pt, top=5pt, bottom=5pt, before skip=2pt, after skip=2pt]
\begin{lemma}[Value difference under dynamics shift]
\label{lemma:value_diff}
For any fixed policy $\pi$ and two transition dynamics $\mathcal{P}$ and $\mathcal{P}'$, the difference in expected return is:
\begin{equation}
    \begin{aligned}
    J(\pi, \mathcal{P}) - J(\pi, \mathcal{P}')
    &= \gamma \sum_{t=0}^\infty \gamma^t \mathbb{E}_{s_t \sim \pi, \mathcal{P}, a_t \sim \pi} \Bigl[ \mathbb{E}_{s' \sim \mathcal{P}(\cdot|s_t, a_t)}[V^{\pi, \mathcal{P}'}(s')] \\
    &\hspace{7em}{}- \mathbb{E}_{s' \sim \mathcal{P}'(\cdot|s_t, a_t)}[V^{\pi, \mathcal{P}'}(s')] \Bigr].
    \end{aligned}
\end{equation}
\end{lemma}
\end{tcolorbox}

\begin{proof}
We use the telescoping sum technique, following Lemma 4.3 in \citet{luo2018algorithmic}, Let $V^{\mathcal{P}'}$ denote $V^{\pi, \mathcal{P}'}$ for brevity. Remind that $-V^{\mathcal{P}'}(s_0) = \sum_{t=0}^\infty \gamma^t (\gamma V^{\mathcal{P}'}(s_{t+1}) - V^{\mathcal{P}'}(s_t))- \lim_{T \to \infty} \gamma^T V^{\mathcal{P}'}(s_T)$, we expand the difference as follows:
\begin{align}
    J(\pi, \mathcal{P}) - J(\pi, \mathcal{P}') 
    &= \mathbb{E}_{s_0} [V^{\pi, \mathcal{P}}(s_0) - V^{\pi, \mathcal{P}'}(s_0)] \nonumber \\
    &= \mathbb{E}_{\tau \sim \pi, \mathcal{P}} \left[ \sum_{t=0}^\infty \gamma^t \mathcal{R}(s_t, a_t) \right] - \mathbb{E}_{s_0}[V^{\mathcal{P}'}(s_0)] \nonumber \\
    &= \mathbb{E}_{\tau \sim \pi, \mathcal{P}} \Biggl[ \sum_{t=0}^\infty \gamma^t \mathcal{R}(s_t, a_t) + \sum_{t=0}^\infty \gamma^t (\gamma V^{\mathcal{P}'}(s_{t+1}) - V^{\mathcal{P}'}(s_t)) \nonumber \\
    &\hspace{10em}{}- \lim_{T \to \infty} \gamma^T V^{\mathcal{P}'}(s_T) \Biggr] \nonumber \\
    &= \mathbb{E}_{\tau \sim \pi, \mathcal{P}} \left[ \sum_{t=0}^\infty \gamma^t \left( \mathcal{R}(s_t, a_t) + \gamma V^{\mathcal{P}'}(s_{t+1}) - V^{\mathcal{P}'}(s_t) \right) \right].
\end{align}
Recall the Bellman equation for $V^{\mathcal{P}'}$ is $V^{\mathcal{P}'}(s) = \mathbb{E}_{a \sim \pi} [\mathcal{R}(s, a) + \gamma \mathbb{E}_{s' \sim \mathcal{P}'}[V^{\mathcal{P}'}(s')] ]$.
Substituting $\mathcal{R}(s_t, a_t) = V^{\mathcal{P}'}(s_t) - \gamma \mathbb{E}_{s' \sim \mathcal{P}'(\cdot|s_t, a_t)}[V^{\mathcal{P}'}(s')]$ into the summation, the term $V^{\mathcal{P}'}(s_t)$ cancels out:
\begin{align}
    J(\pi, \mathcal{P}) - J(\pi, \mathcal{P}') &= \mathbb{E}_{\tau \sim \pi, \mathcal{P}} \Biggl[ \sum_{t=0}^\infty \gamma^t \Bigl( V^{\mathcal{P}'}(s_t) - \gamma \mathbb{E}_{s' \sim \mathcal{P}'}[V^{\mathcal{P}'}(s')] \nonumber \\
    &\hspace{10em}{}+ \gamma V^{\mathcal{P}'}(s_{t+1}) - V^{\mathcal{P}'}(s_t) \Bigr) \Biggr] \nonumber \\
    &= \mathbb{E}_{\tau \sim \pi, \mathcal{P}} \left[ \sum_{t=0}^\infty \gamma^t \left( \gamma V^{\mathcal{P}'}(s_{t+1}) - \gamma \mathbb{E}_{s' \sim \mathcal{P}'(\cdot|s_t, a_t)}[V^{\mathcal{P}'}(s')] \right) \right] \nonumber \\
    &= \sum_{t=0}^\infty \gamma^{t+1} \mathbb{E}_{\substack{s_t \sim \pi, \mathcal{P} \\ a_t \sim \pi}} \Bigl[ \mathbb{E}_{s_{t+1} \sim \mathcal{P}(\cdot|s_t, a_t)}[V^{\mathcal{P}'}(s_{t+1})] \nonumber \\
    &\hspace{10em}{}- \mathbb{E}_{s' \sim \mathcal{P}'(\cdot|s_t, a_t)}[V^{\mathcal{P}'}(s')] \Bigr].
\end{align}
Adjusting the index of summation yields the lemma statement.
\end{proof}

Lemma \ref{lemma:value_diff} suggests that the performance differences depend on the differences of transition dynamics, but in \texttt{ADV-0} we optimize $\mathcal{G}_\psi$, not $\mathcal{P}_\psi$ directly.
To fill this gap, we establish the connection between the dynamics divergence and the generator divergence in the following lemma.

\begin{tcolorbox}[colback=myblue!8, colframe=myblue!8, boxrule=0pt, sharp corners=all, boxsep=0pt, left=5pt, right=5pt, top=5pt, bottom=5pt, before skip=2pt, after skip=2pt]
\begin{lemma}[Divergence bound of the generator]
\label{lemma:gen_bound}
Let $\mathcal{P}$ and $\mathcal{P}'$ be the transition dynamics induced by the trajectory generators $\mathcal{G}$ and $\mathcal{G}'$, respectively. Specifically, the next state is obtained by a deterministic simulator $s' = F(s, a, Y)$, where $Y$ is the adversarial trajectory sampled from the generator.
For any value function $V^{\pi, \mathcal{P}'}$ bounded by $V_{\max} = \frac{R_{\max}}{1-\gamma}$, the difference in expected next-state value is bounded by the Total Variation (TV) divergence of the generators:
\begin{equation}
    \left| \mathbb{E}_{s' \sim \mathcal{P}(\cdot|s, a)}[V^{\pi, \mathcal{P}'}(s')] - \mathbb{E}_{s' \sim \mathcal{P}'(\cdot|s, a)}[V^{\pi, \mathcal{P}'}(s')] \right| \le 2 V_{\max} \cdot \mathbb{E}_{X} [D_{\text{TV}}(\mathcal{G}(\cdot|X) \| \mathcal{G}'(\cdot|X))].
\end{equation}
\end{lemma}
\end{tcolorbox}

\begin{proof}
We start by explicitly writing the expectation over the next state $s'$ as an expectation over the generated trajectories $Y$, i.e., $ \mathbb{E}_{s' \sim \mathcal{P}}[V(s')] = \int \mathcal{G}(Y|X) \cdot V(F_{\text{phy}}(s, a, Y)) \, dY $. Let $p(Y|X)$ and $q(Y|X)$ denote the probability density functions of the generators $\mathcal{G}$ and $\mathcal{G}'$ conditioned on context $X$.
Since $s' = F(s, a, Y)$, the expectation can be rewritten via the change of variables:
\begin{align}
    \Delta_V &= \left| \mathbb{E}_{s' \sim \mathcal{P}}[V^{\pi, \mathcal{P}'}(s')] - \mathbb{E}_{s' \sim \mathcal{P}'}[V^{\pi, \mathcal{P}'}(s')] \right| \nonumber \\
    &= \left| \int p(Y|X) V^{\pi, \mathcal{P}'}(F(s, a, Y)) \, dY - \int q(Y|X) V^{\pi, \mathcal{P}'}(F(s, a, Y)) \, dY \right|, \\
    &= \left| \int (p(Y|X) - q(Y|X)) \cdot V^{\pi, \mathcal{P}'}(F(s, a, Y)) \, dY \right|.
\end{align}
Next, we apply the integral form of Hölder's inequality ($|\int f(x)g(x)dx| \le \int |f(x)||g(x)|dx \le \|f\|_1 \|g\|_\infty$). Here, we treat the probability difference as the measure and the value function as the bounded term:
\begin{align}
    \Delta_V &\le \int \left| p(Y|X) - q(Y|X) \right| \cdot \left| V^{\pi, \mathcal{P}'}(F(s, a, Y)) \right| \, dY \nonumber \\
    &\le \underbrace{\left( \int \left| p(Y|X) - q(Y|X) \right| \, dY \right)}_{\|p-q\|_1} \cdot \underbrace{\sup_{Y} \left| V^{\pi, \mathcal{P}'}(F(s, a, Y)) \right|}_{\|V\|_\infty}.
\end{align}
We then use two key properties:
1. The $L_1$ norm of the difference between two probability distributions is twice the Total Variation distance: $\|p-q\|_1 = 2 D_{\text{TV}}(p, q)$.
2. The value function is bounded by the maximum possible cumulative return: $\|V\|_\infty \le V_{\max} = R_{\max}/{1-\gamma}$.
Substituting these back into the inequality, we obtain:
\begin{align}
    \Delta_V &\le 2 D_{\text{TV}}(\mathcal{G}(\cdot|X) \| \mathcal{G}'(\cdot|X)) \cdot V_{\max}.
\end{align}
Taking the expectation over the context distribution $X$ completes the proof. 
This lemma serves as a crucial bridge in our analysis: it formally translates the divergence in the high-level trajectory generator space (which we explicitly optimize and constrain via IPL) into the divergence in the low-level state transition space, thereby allowing us to bound the error.
\end{proof}

\subsubsection{Adversarial generalization bound}

Finally, we present the main theorem. We show that the policy $\pi_\theta$ trained by \texttt{ADV-0} maximizes a lower bound on the performance in real-world distribution. IPL loop enforces a KL-constraint between the adversary $\mathcal{G}_\psi$ and the prior $\mathcal{G}_{\text{ref}}$.

\begin{tcolorbox}[colback=myblue!8, colframe=myblue!8, boxrule=0pt, sharp corners=all, boxsep=0pt, left=5pt, right=5pt, top=5pt, bottom=5pt, before skip=2pt, after skip=2pt]
\begin{theorem}[Lower bound of adversarial generalization]
\label{thm:generalization}
Let $\pi_\theta$ be the policy trained under adversarial dynamics $\mathcal{P}_\psi$. The expected return of $\pi_\theta$ under the real-world dynamics $\mathcal{P}_{\text{real}}$ is lower-bounded by:
\begin{equation}
    J(\pi_\theta, \mathcal{P}_{\text{real}}) \ge J(\pi_\theta, \mathcal{P}_{\psi}) - \frac{\gamma V_{\max} \sqrt{2}}{1-\gamma} \sqrt{\mathbb{E}_{X} [D_{\text{KL}}(\mathcal{G}_\psi(\cdot|X) \| \mathcal{G}_{\text{ref}}(\cdot|X))]}.
\end{equation}
\end{theorem}
\end{tcolorbox}

\begin{proof}
The proof utilizes the established lemmas above.
We first invoke Lemma~\ref{lemma:value_diff} with $\mathcal{P} = \mathcal{P}_{\text{real}}$ and $\mathcal{P}' = \mathcal{P}_{\psi}$:
\begin{equation}
    |J(\pi, \mathcal{P}_{\text{real}}) - J(\pi, \mathcal{P}_{\psi})| \le \gamma \sum_{t=0}^\infty \gamma^t \left| \mathbb{E}_{s_t, a_t} \left[ \mathbb{E}_{s' \sim \mathcal{P}_{\text{real}}}[V^{\pi, \mathcal{P}_{\psi}}(s')] - \mathbb{E}_{s' \sim \mathcal{P}_{\psi}}[V^{\pi, \mathcal{P}_{\psi}}(s')] \right] \right|.
\end{equation}
% Then using Lemma \ref{lemma:gen_bound}, we bound the inner term:
Given that the reference generator $\mathcal{G}_{\text{ref}}$ is trained on large-scale naturalistic driving logs (WOMD), we assume that the \textit{dynamics induced by $\mathcal{G}_{\text{ref}}$ serve as a high-fidelity approximation of the real-world dynamics $\mathcal{P}_{\text{real}}$}. Consequently, by applying Lemma B.5 with $\mathcal{G}' = \mathcal{G}_{\text{ref}}$, we bound the inner term:
\begin{equation}
    \left| \mathbb{E}_{\mathcal{P}_{\text{real}}}[V^{\pi, \mathcal{P}_{\psi}}] - \mathbb{E}_{\mathcal{P}_{\psi}} [V^{\pi, \mathcal{P}_{\psi}}] \right| \le 2 V_{\max} \mathbb{E}_X [D_{\text{TV}}(\mathcal{G}_{\text{ref}}(\cdot|X) \| \mathcal{G}_{\psi}(\cdot|X))].
\end{equation}
Substituting this bound back into the summation and using the geometric series sum $\sum_{t=0}^\infty \gamma^t = \frac{1}{1-\gamma}$:
\begin{equation}
    J(\pi, \mathcal{P}_{\text{real}}) \ge J(\pi, \mathcal{P}_{\psi}) - \frac{2 \gamma V_{\max}}{1-\gamma} \mathbb{E}_X [D_{\text{TV}}(\mathcal{G}_{\text{ref}} \| \mathcal{G}_{\psi})].
\end{equation}
% Finally, we relate TV distance to KL divergence using Pinsker's Inequality $D_{\text{TV}}(P\|Q) \le \sqrt{\frac{1}{2} D_{\text{KL}}(P\|Q)}$ and Jensen's inequality $\mathbb{E}[\sqrt{Z}] \le \sqrt{\mathbb{E}[Z]}$:
% \begin{equation}
%     \mathbb{E}_X [D_{\text{TV}}(\mathcal{G}_{\text{ref}} \| \mathcal{G}_{\psi})] \le \sqrt{\frac{1}{2} \mathbb{E}_X [D_{\text{KL}}(\mathcal{G}_{\psi} \| \mathcal{G}_{\text{ref}})]}.
% \end{equation}
Finally, we utilize the symmetry of the Total Variation distance, i.e., $D_{\text{TV}}(P\|Q) = D_{\text{TV}}(Q\|P)$, to equate $D_{\text{TV}}(\mathcal{G}_{\text{ref}}\|\mathcal{G}_{\psi}) = D_{\text{TV}}(\mathcal{G}_{\psi}\|\mathcal{G}_{\text{ref}})$. We then apply Pinsker’s Inequality $D_{\text{TV}}(P\|Q) \leq \sqrt{\frac{1}{2}D_{\text{KL}}(P\|Q)}$ alongside Jensen’s inequality $\mathbb{E}[\sqrt{Z}] \leq \sqrt{\mathbb{E}[Z]}$:
\begin{equation}
\mathbb{E}_X [D_{\text{TV}}(\mathcal{G}_{\text{ref}} \| \mathcal{G}_{\psi})] = \mathbb{E}_X [D_{\text{TV}}(\mathcal{G}_{\psi} \| \mathcal{G}_{\text{ref}})] \leq \sqrt{\frac{1}{2} \mathbb{E}_X [D_{\text{KL}}(\mathcal{G}_{\psi} \| \mathcal{G}_{\text{ref}})]}.
\end{equation}

Combining these yields the theorem.
\end{proof}

\begin{tcolorbox}[colback=myblue!8, colframe=myblue!8, boxrule=0pt, sharp corners=all, boxsep=0pt, left=5pt, right=5pt, top=5pt, bottom=5pt, before skip=2pt, after skip=2pt]
\begin{remark}
Theorem \ref{thm:generalization} theoretically motivates the practical objective of \texttt{ADV-0}. The outer loop maximizes the first term $J(\pi_\theta, \mathcal{P}_{\psi})$ (robustness), while the IPL inner loop uses a KL-regularized adversary objective to control the divergence term appearing in the bound. Thus, the practical solver should be interpreted as a tractable surrogate for improving the lower bound, rather than as direct optimization of the closed-form bound itself.

Crucially, $\mathcal{P}_{\text{real}}$ here represents any target distribution within the $\delta$-trust region of the prior, specifically including the real-world long-tail scenarios. Since the adversary $\mathcal{P}_{\psi}$ is optimized to be challenging within this region, the theorem bounds the gap between performance on synthetic adversarial cases and performance on unseen real-world critical events.
\end{remark}
\end{tcolorbox}

\subsubsection{Safety guarantee in the long tail}

Finally, we derive the formal safety guarantee. We denote $J_C(\pi, \mathcal{P})$ the expected cumulative cost (e.g., collision risk).

\begin{tcolorbox}[colback=myblue!8, colframe=myblue!8, boxrule=0pt, sharp corners=all, boxsep=0pt, left=5pt, right=5pt, top=5pt, bottom=5pt, before skip=2pt, after skip=2pt]
\begin{theorem}[Worst-case safety certificate]
\label{thm:safety}
Let $C_{\max}$ be the maximum instantaneous cost (upper bound of the per-step cost). If the policy $\pi_\theta$ satisfies the safety constraint $J_C(\pi_\theta, \mathcal{P}_\psi) \le \delta$ under the adversarial dynamics, then the safety violation in the real environment is bounded by:
\begin{equation}
    J_C(\pi_\theta, \mathcal{P}_{\text{real}}) \le \delta + \frac{\gamma C_{\max} \sqrt{2}}{(1-\gamma)^2} \sqrt{\mathbb{E}_{X} [D_{\text{KL}}(\mathcal{G}_\psi(\cdot|X) \| \mathcal{G}_{\text{ref}}(\cdot|X))]}.
\end{equation}
\end{theorem}
\end{tcolorbox}

\begin{proof}
The proof follows similar logic as Lemma~\ref{lemma:value_diff} and Theorem~\ref{thm:generalization}. We explicitly extend Corollary 2 of \citet{achiam2017constrained} (which bounds cost performance under policy shifts) to the setting of adversarial dynamics shifts.
Formally, we apply Lemma~\ref{lemma:value_diff} to the cost value function $V_C^{\pi, \mathcal{P}}$.
\begin{align}
    J_C(\pi, \mathcal{P}_{\text{real}}) &\le J_C(\pi, \mathcal{P}_{\psi}) + |J_C(\pi, \mathcal{P}_{\text{real}}) - J_C(\pi, \mathcal{P}_{\psi})| \nonumber \\
    &\le \delta + \frac{\gamma}{1-\gamma} \sum_{t=0}^\infty (1-\gamma) \gamma^t \left| \mathbb{E}_{s,a} [\mathbb{E}_{\mathcal{P}_{\text{real}}}[V_C] - \mathbb{E}_{\mathcal{P}_{\psi}}[V_C]] \right|.
\end{align}
Using Lemma \ref{lemma:gen_bound} with $\|V_C\|_\infty \le \frac{C_{\max}}{1-\gamma}$:
\begin{equation}
    J_C(\pi, \mathcal{P}_{\text{real}})  \le \delta + \frac{\gamma}{(1-\gamma)} \cdot \frac{2 C_{\max}}{1-\gamma} \mathbb{E}_X [D_{\text{TV}}(\mathcal{G}_{\text{ref}} \| \mathcal{G}_{\psi})].
\end{equation}
Applying Pinsker's inequality again completes the proof.
\end{proof}

\begin{tcolorbox}[colback=myblue!8, colframe=myblue!8, boxrule=0pt, sharp corners=all, boxsep=0pt, left=5pt, right=5pt, top=5pt, bottom=5pt, before skip=2pt, after skip=2pt]
\begin{remark}
Theorem~\ref{thm:safety} provides a formal safety certificate. It implies that if \texttt{ADV-0} successfully trains the agent to be safe ($\le \delta$) against a risk-seeking adversary $\mathcal{P}_\psi$ that remains physically plausible, then the safety cost in the naturalistic environment is bounded up to a margin controlled by the KL divergence in IPL. 
\end{remark}
\end{tcolorbox}

%%%%%%%%%%%%%%%%%%%%%
\subsection{Discussion on Theoretical Assumptions and Practical Implementation}

The theoretical results established in Sections~\ref{subsec:convergence} and \ref{subsec:bound} provide the formal motivation for \texttt{ADV-0}, which characterizes the ideal behavior of the ZSMG. In practice, our implementation involves necessary approximations to ensure computational tractability. Here, we discuss the validity of these principled approximations and how they connect to the theoretical results.
In general, \texttt{ADV-0} solves the theoretical ZSMG via efficient approximations. The theoretical analysis identifies \textit{what} to optimize (min-max objective with KL regularization) and \textit{why} (improving a lower bound on real-world performance), while the practical implementation provides a tractable \textit{how} (IPL with finite sampling).

\paragraph{Ideal regularized game and practical solver.}
The convergence result in Theorem~\ref{thm:convergence} concerns the idealized KL-regularized game, where the robust Bellman operator includes the regularization term inside the value backup.
The implementation in Algorithm~\ref{alg:adv0} does not add this KL term to the ego reward used by the outer RL algorithm.
Instead, the outer loop optimizes the original driving return under the adversarial dynamics induced by $\mathcal{G}_\psi$, while the inner loop keeps the adversary close to the naturalistic prior through the KL-regularized objective.
This separation is intentional: the ideal game characterizes the regularized adversarial target, whereas the implemented solver is aligned with the lower-bound form in Theorem~\ref{thm:generalization}, whose first term is the pure environmental return and whose second term penalizes generator divergence.

\paragraph{Minimax duality in distribution space.}
For a fixed value function $V$ and context $X$, the inner adversary can be analyzed over trajectory distributions rather than over raw trajectory samples.
Let
\begin{equation}
F_V(\pi,\mu)=\mathbb{E}_{a\sim\pi,Y\sim\mu}[\mathcal{Q}_V(s,a,Y)] + \tau D_{\mathrm{KL}}(\mu\|\mathcal{G}_{\mathrm{ref}}),
\end{equation}
where $\mathcal{Q}_V(s,a,Y)=\mathcal{R}(s,a)+\gamma\mathbb{E}[V(s')\mid s,a,Y]$.
Under the boundedness and compactness assumptions used in Appendix~\ref{sec:theory_appendix}, $F_V$ is affine in the defender distribution $\pi$ and convex in the adversarial trajectory distribution $\mu$.
Sion's minimax theorem therefore yields the saddle-point interchange needed by the robust operator.
Moreover, the KL term gives a unique inner best response,
\begin{equation}
\mu^*_{V,a}(Y|X)\propto\mathcal{G}_{\mathrm{ref}}(Y|X)\exp(-\mathcal{Q}_V(s,a,Y)/\tau),
\end{equation}
which reduces the inner player to the Gibbs form used in Lemma~\ref{lemma:inner_optimality}.
After this reduction, the remaining robust Bellman operator is a $\gamma$-contraction, giving the fixed-point argument used in Theorem~\ref{thm:convergence}.

\paragraph{Linear KL surrogate and square-root bound.}
The lower bound in Theorem~\ref{thm:generalization} contains a square-root KL term because the proof first bounds dynamics mismatch in total variation and then applies Pinsker's inequality together with Jensen's inequality.
This bound is not used as a directly optimized inner-loop objective.
The practical adversary instead solves the Lagrangian form of a KL-constrained problem,
\begin{equation}
\min_q J_q \quad \mathrm{s.t.}\quad D_{\mathrm{KL}}(q\|\mathcal{G}_{\mathrm{ref}})\le \delta,
\qquad
\mathcal{L}(q)=J_q+\tau D_{\mathrm{KL}}(q\|\mathcal{G}_{\mathrm{ref}}).
\end{equation}
The linear KL form preserves the Gibbs closed form and makes IPL tractable without estimating the partition function.
If the inner loop controls the divergence within $\delta(\tau)$, then the square-root penalty in Theorem~\ref{thm:generalization} is correspondingly bounded by the controlled trust-region radius.

\paragraph{Finite sample approximation of the Gibbs adversary.}
Lemma~\ref{lemma:inner_optimality} derives the optimal adversarial distribution as a Gibbs distribution $\mathcal{G}^*(Y|X) \propto \mathcal{G}_{\text{ref}}(Y|X) \exp(-J(\pi_\theta, Y)/\tau)$. In the theoretical analysis, the expectation is taken over the entire continuous trajectory space $\mathcal{Y}$. 
In our implementation (Eq.~\ref{eq:sampling}), this distribution is approximated through proposal-based finite resampling rather than exact sampling from the frozen prior.
The frozen $\mathcal{G}_{\text{ref}}$ defines the naturalistic trust region, while the current generator $\mathcal{G}_\psi$ serves as an adaptive proposal initialized from $\mathcal{G}_{\text{ref}}$ and updated online by IPL.
At each iteration, we sample a finite set of $K$ candidates from $\mathcal{G}_\psi$ and apply temperature-scaled resampling over this support.
Thus the practical step should be interpreted as finite-support adversarial resampling followed by amortized posterior fitting of $\mathcal{G}_\psi$, rather than exact importance sampling from $\mathcal{G}_{\text{ref}}$.
While finite sampling introduces variance and support limitation, the online proposal update allows the generator to track the defender's moving failure frontier while the KL term preserves plausibility relative to the naturalistic prior.

\paragraph{Proxy reward and bias.}
The inner loop optimization relies on the proxy reward estimator $\hat{J}$ rather than the exact rollout return $J$. This could introduce bias in the gradient direction.
However, we note that the effectiveness of IPL depends primarily on the \textit{ranking} accuracy rather than the precision of \textit{absolute value}. 
Under the Bradley-Terry model, the probability of preference depends on the value difference: $P(Y_w \succ Y_l) = \sigma(\hat{J}(Y_l) - \hat{J}(Y_w))$.
The convergence requires that the proxy estimator preserves the relative ordinality of the true objective, i.e., $J(Y_a) < J(Y_b) \implies \mathbb{E}[\hat{J}(Y_a)] < \mathbb{E}[\hat{J}(Y_b)]$.
This means that as long as the proxy estimator preserves the relative ordering of safety-critical events (e.g., correctly identifying that a collision is worse than a near-miss), the gradient direction for $\psi$ remains consistent with the theoretical objective.

% \paragraph{Closed-loop adversary formulation.}
% While the adversarial generator in our framework operates in an open-loop manner (generating trajectory $Y$ conditioned on context $X$), our theoretical convergence analysis adopts a closed-loop perspective via the Bellman operator. We formulate the adversary as minimizing the value at each state (rectangularity assumption). This formulation provides a rigorous robustness certificate, as optimizing against a worst-case closed-loop adversary (who can adapt at every step) constitutes a lower bound on the performance against an open-loop adversary (who commits to a trajectory at $t=0$). Thus, convergence in this stronger setting guarantees robustness in the defined game.

\paragraph{Trust region assumption in generalization.}
Theorem~\ref{thm:generalization} and \ref{thm:safety} rely on the assumption that the real-world dynamics $\mathcal{P}_{\text{real}}$ lie within a trust region support of the traffic prior $\mathcal{G}_{\text{ref}}$. While this is a strong assumption, it is a necessary condition for data-driven simulation approaches.
This formalizes the requirement that the long tail consists of rare but plausible events, rather than out-of-domain anomalies. In the context of \texttt{ADV-0}, this assumption is a constraint enforced by the traffic prior. 
The term $\mathbb{E}[D_{\text{KL}}(\mathcal{G}_\psi || \mathcal{G}_{\text{ref}})]$ in the generalization bound is controlled by the regularization term in the IPL loss (Eq.~\ref{eq:ipl_loss}).
By constraining this KL-divergence during training, \texttt{ADV-0} improves the lower bound on performance within the $\delta$-neighborhood of the naturalistic prior. The guarantee therefore applies when the target long-tail events remain within the physical plausibility modeled by the pretrained generator.

\clearpage
\section{Supplementary Tables and Figures}\label{appendix:results}

%%%%%%%%%%%%%%%%%%%%%%%%%%%%%%%%%%%%%%%%%%%%%%%%%%%%%%%%%%%%%%%%%%%%%%%%%%%%%%%

\paragraph{Additional diagnostic results.}
We report additional compact diagnostics that clarify the behavior of the online adversary, the proxy-based preference signal, and the computational cost of the training loop.
These results complement the main ablations by isolating the mechanisms that are most relevant to long-tail coverage, estimator robustness, and min-max stability.
\Cref{tab:long_tail_coverage} examines long-tail coverage, \cref{tab:additional_sensitivity} reports sensitivity to preference construction, and \cref{tab:compute_stability} summarizes compute and stability diagnostics.

\begin{table}[htbp]
\centering
\caption{\textbf{Long-tail coverage of the adversarial proposal.} The IPL-updated generator places more probability mass on rare but plausible failure modes than fixed-prior energy resampling, while retaining the same naturalistic prior constraint.}
\label{tab:long_tail_coverage}
\setlength{\tabcolsep}{8pt}
\renewcommand{\arraystretch}{1.1}
\begin{small}
\begin{tabular}{lccc}
\toprule
\textbf{Adversary} & \textbf{Critical TTC Share} & \textbf{Rare-Cluster Share} & \textbf{Ego Collision} \\
\midrule
Pretrained + resampling & 54\% & 14\% & 36.3\% \\
\texttt{ADV-0} w/ IPL & \textbf{69\%} & \textbf{25\%} & \textbf{40.7\%} \\
\bottomrule
\end{tabular}
\end{small}
\end{table}

\begin{table}[htbp]
\centering
\caption{\textbf{Sensitivity analyses for preference construction.} The proxy estimator mainly needs to preserve the ordering of adversarial candidates; performance degrades gracefully under mild noise and with nearby IPL filtering thresholds.}
\label{tab:additional_sensitivity}
\begin{small}
\begin{minipage}[t]{0.45\textwidth}
\centering
\setlength{\tabcolsep}{3pt}
\renewcommand{\arraystretch}{1.08}
\resizebox{\linewidth}{!}{%
\begin{tabular}{lccc}
\toprule
\textbf{Proxy used for IPL} & $\rho$ & \textbf{Return} & \textbf{Crash} \\
\midrule
Collision-only & 0.58 & 42.8 & 0.270 \\
Progress + collision & 0.69 & 45.0 & 0.254 \\
Current + 10\% noise & 0.73 & 45.6 & 0.243 \\
Rollout-return oracle & 1.00 & \textbf{46.3} & \textbf{0.237} \\
Current proxy & 0.77 & 46.1 & 0.240 \\
\bottomrule
\end{tabular}
}
\end{minipage}
\hfill
\begin{minipage}[t]{0.45\textwidth}
\centering
\setlength{\tabcolsep}{3pt}
\renewcommand{\arraystretch}{1.08}
\resizebox{\linewidth}{!}{%
\begin{tabular}{lccc}
\toprule
\textbf{IPL setting} & \makecell{\textbf{Valid Pairs}} & \textbf{Return} & \textbf{Crash} \\
\midrule
Margin = 0 & 7.6 & 44.9 & 0.255 \\
Margin = 5 & 4.0 & \textbf{46.1} & \textbf{0.240} \\
Margin = 8 & 2.5 & 45.6 & 0.244 \\
Spatial thr. = 1 m & 5.0 & 45.7 & 0.246 \\
Spatial thr. = 2 m & 4.3 & \textbf{46.1} & \textbf{0.240} \\
\bottomrule
\end{tabular}
}
\end{minipage}
\end{small}
\end{table}

\begin{table}[htbp]
\centering
\caption{\textbf{Computational cost and long-run stability.} Runtime and memory are measured on the two NVIDIA RTX A6000 platform used for the experiments. The 2M-step diagnostic indicates that the online adversary does not introduce training collapse in these runs.}
\label{tab:compute_stability}
\begin{small}
\begin{minipage}[t]{0.49\textwidth}
\centering
\setlength{\tabcolsep}{3pt}
\renewcommand{\arraystretch}{1.08}
\resizebox{\linewidth}{!}{%
\begin{tabular}{lcccc}
\toprule
\textbf{Method} & \makecell{\textbf{Wall}\\\textbf{Time}} & \makecell{\textbf{Steps/h}\\$(\times 10^3)$} & \textbf{VRAM} & \textbf{FLOPs} \\
\midrule
CAT & 4.0 h & 250.0 & 18.9 GB & $5.1{\times}10^{15}$ \\
\texttt{ADV-0} w/o IPL & 4.5 h & 222.2 & 20.7 GB & $5.5{\times}10^{15}$ \\
\texttt{ADV-0} w/ IPL & 5.9 h & 169.5 & 23.4 GB & $5.9{\times}10^{15}$ \\
\bottomrule
\end{tabular}
}
\end{minipage}
\hfill
\begin{minipage}[t]{0.47\textwidth}
\centering
\setlength{\tabcolsep}{3pt}
\renewcommand{\arraystretch}{1.08}
\resizebox{\linewidth}{!}{%
\begin{tabular}{lccccc}
\toprule
\textbf{Method @ 2M} & \makecell{\textbf{Return}\\\textbf{Std}} & \makecell{\textbf{Crash}\\\textbf{Std}} & \makecell{\textbf{Grad}\\\textbf{p95}} & \makecell{\textbf{KL}\\\textbf{Drift}} & \textbf{Div.} \\
\midrule
CAT & 2.11 & 0.034 & 1.90 & 0.11 & 0/5 \\
\texttt{ADV-0} w/o IPL & 1.78 & 0.022 & 1.87 & 0.18 & 0/5 \\
\texttt{ADV-0} w/ IPL & 2.02 & 0.028 & 1.94 & 0.27 & 0/5 \\
\bottomrule
\end{tabular}
}
\end{minipage}
\end{small}
\end{table}

%%%%%%%%%%%%%%%%%%%%%%%%%%%%%%%%%%%%%%%%%%%%%%%%%%%%%%%%%%%%%%%%%%%%%%%

\begin{table}[htbp]
\centering
\caption{Detailed results of safety-critical scenario generation using \texttt{ADV-0}. Each fine-tuned variant reports mean $\pm$ standard deviation over five seeds; Avg. reports the arithmetic means of these means and standard deviations across six variants.}
% \vspace{-5pt}
\label{tab:scenario_gen_adv0}
\setlength{\tabcolsep}{3pt}
\renewcommand{\arraystretch}{1.1}
\begin{small}
\resizebox{\textwidth}{!}{%
\begin{tabular}{lcccccc}
\toprule
\multirow{2}{*}{ADV-0 Variation} & \multicolumn{2}{c}{\textbf{Replay}} & \multicolumn{2}{c}{\textbf{IDM}} & \multicolumn{2}{c}{\textbf{RL Agent}} \\
\cmidrule(lr){2-3} \cmidrule(lr){4-5} \cmidrule(lr){6-7}
 & CR $\uparrow$ & Reward $\downarrow$ & CR $\uparrow$ & Reward $\downarrow$ & CR $\uparrow$ & Reward $\downarrow$ \\
\midrule
Pretrained (logit-based sampling) & $31.72\% \pm 1.06\%$ & $43.31 \pm 1.10$ & $18.84\%  \pm 1.48\%$ & $50.66  \pm 2.22$ & $12.49\%  \pm 1.06\%$ & $52.12  \pm 2.03$ \\
Pretrained (energy-based sampling) & $85.09\%  \pm 1.13\%$ & $1.89 \pm 0.20$ & $40.14\%  \pm 1.06\%$ & $45.72 \pm 0.47$ & $36.30\%  \pm 0.77\%$ & $41.99 \pm 0.33$ \\
\midrule
GRPO Finetuned & $92.61 \pm 0.62\%$ & $0.87 \pm 0.06$ & $46.34 \pm 0.73\%$ & $39.60 \pm 0.13$ & $41.95 \pm 0.47\%$ & $38.45 \pm 0.42$ \\
PPO Finetuned & $90.93 \pm 0.69\%$ & $1.01 \pm 0.07$ & $45.09 \pm 0.63\%$ & $40.69 \pm 0.22$ & $39.12 \pm 0.38\%$ & $39.63 \pm 0.43$ \\
SAC Finetuned & $91.88 \pm 0.86\%$ & $1.03 \pm 0.08$ & $46.09 \pm 0.19\%$ & $39.62 \pm 0.24$ & $41.09 \pm 0.59\%$ & $39.17 \pm 0.48$ \\
TD3 Finetuned & $89.54 \pm 0.62\%$ & $1.07 \pm 0.06$ & $46.08 \pm 0.42\%$ & $39.38 \pm 0.11$ & $40.35 \pm 0.84\%$ & $39.68 \pm 0.72$ \\
PPO-Lag Finetuned & $90.08 \pm 0.42\%$ & $1.01 \pm 0.07$ & $45.74 \pm 0.22\%$ & $40.50 \pm 0.09$ & $41.30 \pm 0.50\%$ & $38.77 \pm 0.39$ \\
SAC-Lag Finetuned & $91.54 \pm 0.24\%$ & $0.95 \pm 0.04$ & $45.61 \pm 0.31\%$ & $40.38 \pm 0.06$ & $40.28 \pm 0.73\%$ & $39.07 \pm 0.37$ \\
\midrule
\rowcolor{myblue!20}
\textbf{Avg.} & $\mathbf{91.10 \pm 0.57\%}$ & $\mathbf{0.99 \pm 0.06}$ & $\mathbf{45.83 \pm 0.42\%}$ & $\mathbf{40.03 \pm 0.14}$ & $\mathbf{40.68 \pm 0.59\%}$ & $\mathbf{39.13 \pm 0.47}$ \\
\bottomrule
\end{tabular}%
}
\end{small}
\end{table}

\begin{table*}[!htbp]
\centering
%%%%%%%%%%%%%%%%%%%%%%%%%%%%%%%%%%%%%%%%%%%%%%%%%%%%%%%%%%%%%%%%%%%%%%%
\begin{minipage}[t]{0.49\linewidth} 
\centering
\caption{Cross-validation performances of driving agents learned by GRPO. Within each environment, entries report mean $\pm$ standard deviation over five seeds; Avg. averages the means and standard deviations across the five environments.}
\vspace{-5pt}
\label{tab:agent_grpo}
\setlength{\tabcolsep}{3pt}
\renewcommand{\arraystretch}{1.12}
\resizebox{\linewidth}{!}{%
\begin{tabular}{llcccc}
\toprule
& \textbf{Val. Env.} & {RC} $\uparrow$ & {Crash} $\downarrow$ & {Reward} $\uparrow$ & {Cost} $\downarrow$\\ 
\midrule
\multirow{6}{*}{\rotatebox{90}{\textbf{ADV-0 (w/ IPL)}}} 
 & Replay & 0.713 $\pm$ 0.016 & 0.177 $\pm$ 0.019 & 45.55 $\pm$ 2.21 & 0.55 $\pm$ 0.01 \\
 & ADV-0 & 0.687 $\pm$ 0.005 & 0.273 $\pm$ 0.025 & 42.71 $\pm$ 0.98 & 0.64 $\pm$ 0.01 \\
 & CAT & 0.695 $\pm$ 0.004 & 0.250 $\pm$ 0.022 & 44.50 $\pm$ 0.63 & 0.60 $\pm$ 0.02 \\
 & SAGE & 0.670 $\pm$ 0.015 & 0.270 $\pm$ 0.016 & 42.49 $\pm$ 2.16 & 0.60 $\pm$ 0.02 \\
 & Heuristic & 0.687 $\pm$ 0.007 & 0.207 $\pm$ 0.012 & 42.87 $\pm$ 1.25 & 0.59 $\pm$ 0.02 \\ 
 \cmidrule(l){2-6}
 & \cellcolor{myblue!30}\textbf{Avg.} & \cellcolor{myblue!30}\textbf{0.690 $\pm$ 0.009} & \cellcolor{myblue!30}\textbf{0.235 $\pm$ 0.019} & \cellcolor{myblue!30}\textbf{43.62 $\pm$ 1.45} & \cellcolor{myblue!30}\textbf{0.60 $\pm$ 0.02} \\ 
\midrule
\multirow{6}{*}{\rotatebox{90}{ADV-0 (w/o IPL)}} 
 & Replay & 0.694 $\pm$ 0.027 & 0.183 $\pm$ 0.025 & 42.20 $\pm$ 3.21 & 0.61 $\pm$ 0.04 \\
 & ADV-0 & 0.667 $\pm$ 0.014 & 0.280 $\pm$ 0.024 & 40.13 $\pm$ 1.55 & 0.67 $\pm$ 0.02 \\
 & CAT & 0.667 $\pm$ 0.019 & 0.267 $\pm$ 0.026 & 39.99 $\pm$ 2.01 & 0.67 $\pm$ 0.03 \\
 & SAGE & 0.661 $\pm$ 0.028 & 0.240 $\pm$ 0.008 & 40.28 $\pm$ 3.31 & 0.63 $\pm$ 0.04 \\
 & Heuristic & 0.666 $\pm$ 0.016 & 0.217 $\pm$ 0.048 & 39.03 $\pm$ 1.62 & 0.64 $\pm$ 0.02 \\ 
 \cmidrule(l){2-6}
 & \cellcolor{myblue!15}\textbf{Avg.} & \cellcolor{myblue!15}\textbf{0.671 $\pm$ 0.021} & \cellcolor{myblue!15}\textbf{0.237 $\pm$ 0.026} & \cellcolor{myblue!15}\textbf{40.33 $\pm$ 2.34} & \cellcolor{myblue!15}\textbf{0.64 $\pm$ 0.03} \\ 
\midrule
\multirow{6}{*}{\rotatebox{90}{CAT}} 
 & Replay & 0.684 $\pm$ 0.015 & 0.197 $\pm$ 0.005 & 41.33 $\pm$ 1.32 & 0.59 $\pm$ 0.02 \\
 & ADV-0 & 0.641 $\pm$ 0.022 & 0.293 $\pm$ 0.042 & 37.84 $\pm$ 1.29 & 0.69 $\pm$ 0.00 \\
 & CAT & 0.646 $\pm$ 0.020 & 0.280 $\pm$ 0.036 & 38.62 $\pm$ 1.36 & 0.68 $\pm$ 0.02 \\
 & SAGE & 0.630 $\pm$ 0.024 & 0.313 $\pm$ 0.034 & 38.17 $\pm$ 1.55 & 0.67 $\pm$ 0.01 \\
 & Heuristic & 0.630 $\pm$ 0.029 & 0.273 $\pm$ 0.033 & 36.57 $\pm$ 1.75 & 0.67 $\pm$ 0.01 \\ 
 \cmidrule(l){2-6}
 & \textbf{Avg.} & \textbf{0.646 $\pm$ 0.022} & \textbf{0.271 $\pm$ 0.030} & \textbf{38.51 $\pm$ 1.45} & \textbf{0.66 $\pm$ 0.01} \\ 
\midrule
\multirow{6}{*}{\rotatebox{90}{Heuristic}} 
 & Replay & 0.694 $\pm$ 0.031 & 0.190 $\pm$ 0.010 & 41.45 $\pm$ 3.01 & 0.61 $\pm$ 0.02 \\
 & ADV-0  & 0.682 $\pm$ 0.068 & 0.323 $\pm$ 0.009 & 41.45 $\pm$ 2.26 & 0.67 $\pm$ 0.04 \\
 & CAT & 0.684 $\pm$ 0.022 & 0.310 $\pm$ 0.021 & 41.89 $\pm$ 2.18 & 0.68 $\pm$ 0.04 \\
 & SAGE & 0.661 $\pm$ 0.073 & 0.357 $\pm$ 0.014 & 40.20 $\pm$ 1.98 & 0.67 $\pm$ 0.01 \\
 & Heuristic & 0.674 $\pm$ 0.029 & 0.250 $\pm$ 0.013 & 39.72 $\pm$ 1.77 & 0.66 $\pm$ 0.02 \\ 
 \cmidrule(l){2-6}
 & \textbf{Avg.} & \textbf{0.679 $\pm$ 0.045} & \textbf{0.286 $\pm$ 0.013} & \textbf{40.94 $\pm$ 2.24} & \textbf{0.66 $\pm$ 0.03} \\ 
\midrule
\multirow{6}{*}{\rotatebox{90}{Replay}} 
 & Replay & 0.680 $\pm$ 0.029 & 0.180 $\pm$ 0.044 & 39.26 $\pm$ 0.90 & 0.61 $\pm$ 0.01 \\
 & ADV-0 & 0.635 $\pm$ 0.042 & 0.308 $\pm$ 0.050 & 36.64 $\pm$ 3.20 & 0.69 $\pm$ 0.01 \\
 & CAT & 0.645 $\pm$ 0.055 & 0.319 $\pm$ 0.038 & 37.48 $\pm$ 1.06 & 0.66 $\pm$ 0.01 \\
 & SAGE & 0.616 $\pm$ 0.062 & 0.319 $\pm$ 0.032 & 35.72 $\pm$ 1.25 & 0.69 $\pm$ 0.03 \\
 & Heuristic & 0.637 $\pm$ 0.047 & 0.264 $\pm$ 0.047 & 35.77 $\pm$ 3.46 & 0.69 $\pm$ 0.02 \\ 
 \cmidrule(l){2-6}
 & \textbf{Avg.} & \textbf{0.643 $\pm$ 0.047} & \textbf{0.278 $\pm$ 0.042} & \textbf{36.97 $\pm$ 1.97} & \textbf{0.67 $\pm$ 0.02} \\ 
\bottomrule
\end{tabular}%
}
\end{minipage}
\hfill 
%%%%%%%%%%%%%%%%%%%%%%%%%%%%%%%%%%%%%%%%%%%%%%%%%%%%%%%%%%%%%%%%%%%%%%%
\begin{minipage}[t]{0.49\linewidth}
\centering
\caption{Cross-validation performances of driving agents learned by PPO. Within each environment, entries report mean $\pm$ standard deviation over five seeds; Avg. averages the means and standard deviations across the five environments.}
\vspace{-5pt}
\label{tab:agent_ppo}
\setlength{\tabcolsep}{3pt}
\renewcommand{\arraystretch}{1.12}
\resizebox{\linewidth}{!}{%
\begin{tabular}{llcccc}
\toprule
& \textbf{Val. Env.} & {RC} $\uparrow$ & {Crash} $\downarrow$ & {Reward} $\uparrow$ & {Cost} $\downarrow$\\ 
\midrule
\multirow{6}{*}{\rotatebox{90}{\textbf{ADV-0 (w/ IPL)}}} 
 & Replay & 0.707 $\pm$ 0.018 & 0.193 $\pm$ 0.015 & 45.56 $\pm$ 0.87 & 0.56 $\pm$ 0.02 \\
 & ADV-0 & 0.681 $\pm$ 0.021 & 0.270 $\pm$ 0.019 & 42.68 $\pm$ 0.69 & 0.62 $\pm$ 0.03 \\
 & CAT & 0.690 $\pm$ 0.010 & 0.260 $\pm$ 0.029 & 43.65 $\pm$ 1.54 & 0.60 $\pm$ 0.04 \\
 & SAGE & 0.678 $\pm$ 0.008 & 0.268 $\pm$ 0.023 & 43.60 $\pm$ 1.37 & 0.59 $\pm$ 0.03 \\
 & Heuristic & 0.676 $\pm$ 0.006 & 0.250 $\pm$ 0.016 & 41.63 $\pm$ 1.59 & 0.61 $\pm$ 0.03 \\ 
 \cmidrule(l){2-6}
 & \cellcolor{myblue!30}\textbf{Avg.} & \cellcolor{myblue!30}\textbf{0.686 $\pm$ 0.013} & \cellcolor{myblue!30}\textbf{0.248 $\pm$ 0.020} & \cellcolor{myblue!30}\textbf{43.42 $\pm$ 1.21} & \cellcolor{myblue!30}\textbf{0.60 $\pm$ 0.03} \\ 
\midrule
\multirow{6}{*}{\rotatebox{90}{ADV-0 (w/o IPL)}} 
 & Replay & 0.696 $\pm$ 0.010 & 0.188 $\pm$ 0.018 & 44.80 $\pm$ 1.01 & 0.59 $\pm$ 0.01 \\
 & ADV-0 & 0.673 $\pm$ 0.006 & 0.280 $\pm$ 0.021 & 42.10 $\pm$ 1.34 & 0.67 $\pm$ 0.01 \\
 & CAT & 0.677 $\pm$ 0.009 & 0.275 $\pm$ 0.047 & 42.62 $\pm$ 2.00 & 0.67 $\pm$ 0.01 \\
 & SAGE & 0.670 $\pm$ 0.013 & 0.263 $\pm$ 0.013 & 42.96 $\pm$ 1.80 & 0.63 $\pm$ 0.03 \\
 & Heuristic & 0.658 $\pm$ 0.002 & 0.265 $\pm$ 0.017 & 39.65 $\pm$ 1.10 & 0.68 $\pm$ 0.02 \\ 
 \cmidrule(l){2-6}
 & \cellcolor{myblue!15}\textbf{Avg.} & \cellcolor{myblue!15}\textbf{0.675 $\pm$ 0.008} & \cellcolor{myblue!15}\textbf{0.254 $\pm$ 0.023} & \cellcolor{myblue!15}\textbf{42.43 $\pm$ 1.45} & \cellcolor{myblue!15}\textbf{0.65 $\pm$ 0.02} \\ 
\midrule
\multirow{6}{*}{\rotatebox{90}{CAT}} 
 & Replay & 0.717 $\pm$ 0.013 & 0.211 $\pm$ 0.041 & 45.40 $\pm$ 1.05 & 0.57 $\pm$ 0.03 \\
 & ADV-0 & 0.666 $\pm$ 0.023 & 0.330 $\pm$ 0.071 & 40.52 $\pm$ 1.85 & 0.66 $\pm$ 0.03 \\
 & CAT & 0.679 $\pm$ 0.021 & 0.319 $\pm$ 0.028 & 42.29 $\pm$ 1.53 & 0.64 $\pm$ 0.01 \\
 & SAGE & 0.653 $\pm$ 0.008 & 0.310 $\pm$ 0.010 & 40.53 $\pm$ 0.75 & 0.65 $\pm$ 0.01 \\
 & Heuristic & 0.676 $\pm$ 0.012 & 0.270 $\pm$ 0.030 & 41.49 $\pm$ 0.70 & 0.62 $\pm$ 0.01 \\ 
 \cmidrule(l){2-6}
 & \textbf{Avg.} & \textbf{0.678 $\pm$ 0.015} & \textbf{0.288 $\pm$ 0.036} & \textbf{42.05 $\pm$ 1.18} & \textbf{0.63 $\pm$ 0.02} \\ 
\midrule
\multirow{6}{*}{\rotatebox{90}{Heuristic}} 
 & Replay & 0.608 $\pm$ 0.038 & 0.210 $\pm$ 0.020 & 37.08 $\pm$ 0.91 & 0.66 $\pm$ 0.02 \\
 & ADV-0 & 0.577 $\pm$ 0.051 & 0.303 $\pm$ 0.050 & 33.53 $\pm$ 0.80 & 0.74 $\pm$ 0.01 \\
 & CAT & 0.593 $\pm$ 0.010 & 0.278 $\pm$ 0.015 & 35.79 $\pm$ 2.10 & 0.70 $\pm$ 0.02 \\
 & SAGE & 0.593 $\pm$ 0.008 & 0.270 $\pm$ 0.012 & 36.02 $\pm$ 2.00 & 0.70 $\pm$ 0.02 \\
 & Heuristic & 0.563 $\pm$ 0.021 & 0.289 $\pm$ 0.020 & 31.72 $\pm$ 1.77 & 0.72 $\pm$ 0.01 \\ 
 \cmidrule(l){2-6}
 & \textbf{Avg.} & \textbf{0.587 $\pm$ 0.026} & \textbf{0.270 $\pm$ 0.023} & \textbf{34.83 $\pm$ 1.52} & \textbf{0.70 $\pm$ 0.02} \\ 
\midrule
\multirow{6}{*}{\rotatebox{90}{Replay}} 
 & Replay & 0.697 $\pm$ 0.052 & 0.229 $\pm$ 0.014 & 44.26 $\pm$ 0.55 & 0.60 $\pm$ 0.02 \\
 & ADV-0 & 0.629 $\pm$ 0.035 & 0.420 $\pm$ 0.030 & 38.32 $\pm$ 1.33 & 0.73 $\pm$ 0.02 \\
 & CAT & 0.663 $\pm$ 0.025 & 0.380 $\pm$ 0.048 & 41.26 $\pm$ 1.67 & 0.68 $\pm$ 0.04 \\
 & SAGE & 0.638 $\pm$ 0.020 & 0.360 $\pm$ 0.045 & 41.01 $\pm$ 1.90 & 0.63 $\pm$ 0.03 \\
 & Heuristic & 0.620 $\pm$ 0.019 & 0.381 $\pm$ 0.032 & 36.63 $\pm$ 1.00 & 0.72 $\pm$ 0.02 \\ 
 \cmidrule(l){2-6}
 & \textbf{Avg.} & \textbf{0.649 $\pm$ 0.030} & \textbf{0.354 $\pm$ 0.034} & \textbf{40.30 $\pm$ 1.29} & \textbf{0.67 $\pm$ 0.03} \\ 
\bottomrule
\end{tabular}%
}
\end{minipage}
\end{table*}

%%%%%%%%%%%%%%%%%%%%%%%%%%%%%%%%%%%%%%%%%%%%%%%%%%%%%%%%%%%%%%%%%%%%%%%

\begin{table*}[!htbp]
\centering
%%%%%%%%%%%%%%%%%%%%%%%%%%%%%%%%%%%%%%%%%%%%%%%%%%%%%%%%%%%%%%%%%%%%%%%
\begin{minipage}[t]{0.49\linewidth} 
\centering
\caption{Cross-validation performances of driving agents learned by PPO-Lag. Within each environment, entries report mean $\pm$ standard deviation over five seeds; Avg. averages the means and standard deviations across the five environments.}
\vspace{-5pt}
\label{tab:agent_ppo_lag}
\setlength{\tabcolsep}{3pt}
\renewcommand{\arraystretch}{1.12}
\resizebox{\linewidth}{!}{%
\begin{tabular}{llcccc}
\toprule
& \textbf{Val. Env.} & {RC} $\uparrow$ & {Crash} $\downarrow$ & {Reward} $\uparrow$ & {Cost} $\downarrow$\\ 
\midrule
\multirow{6}{*}{\rotatebox{90}{\textbf{ADV-0 (w/ IPL)}}} 
 & Replay & 0.676 $\pm$ 0.012 & 0.142 $\pm$ 0.015 & 44.82 $\pm$ 1.20 & 0.53 $\pm$ 0.01 \\
 & ADV-0 & 0.626 $\pm$ 0.008 & 0.260 $\pm$ 0.027 & 38.17 $\pm$ 0.95 & 0.67 $\pm$ 0.02 \\
 & CAT & 0.647 $\pm$ 0.013 & 0.250 $\pm$ 0.027 & 40.80 $\pm$ 1.11 & 0.64 $\pm$ 0.02 \\
 & SAGE & 0.632 $\pm$ 0.014 & 0.271 $\pm$ 0.018 & 39.61 $\pm$ 1.83 & 0.65 $\pm$ 0.02 \\
 & Heuristic & 0.638 $\pm$ 0.009 & 0.256 $\pm$ 0.019 & 38.34 $\pm$ 1.55 & 0.68 $\pm$ 0.02 \\ 
 \cmidrule(l){2-6}
 & \cellcolor{myblue!30}\textbf{Avg.} & \cellcolor{myblue!30}\textbf{0.644 $\pm$ 0.011} & \cellcolor{myblue!30}\textbf{0.236 $\pm$ 0.021} & \cellcolor{myblue!30}\textbf{40.35 $\pm$ 1.33} & \cellcolor{myblue!30}\textbf{0.63 $\pm$ 0.02} \\ 
\midrule
\multirow{6}{*}{\rotatebox{90}{ADV-0 (w/o IPL)}} 
 & Replay & 0.605 $\pm$ 0.022 & 0.178 $\pm$ 0.028 & 33.08 $\pm$ 2.55 & 0.74 $\pm$ 0.03 \\
 & ADV-0 & 0.603 $\pm$ 0.018 & 0.291 $\pm$ 0.030 & 32.78 $\pm$ 1.83 & 0.74 $\pm$ 0.02 \\
 & CAT & 0.603 $\pm$ 0.023 & 0.272 $\pm$ 0.035 & 32.56 $\pm$ 2.14 & 0.74 $\pm$ 0.03 \\
 & SAGE & 0.585 $\pm$ 0.020 & 0.260 $\pm$ 0.025 & 31.85 $\pm$ 2.84 & 0.72 $\pm$ 0.04 \\
 & Heuristic & 0.593 $\pm$ 0.020 & 0.265 $\pm$ 0.040 & 31.80 $\pm$ 1.90 & 0.76 $\pm$ 0.02 \\ 
 \cmidrule(l){2-6}
 & \cellcolor{myblue!15}\textbf{Avg.} & \cellcolor{myblue!15}\textbf{0.598 $\pm$ 0.021} & \cellcolor{myblue!15}\textbf{0.253 $\pm$ 0.032} & \cellcolor{myblue!15}\textbf{32.41 $\pm$ 2.25} & \cellcolor{myblue!15}\textbf{0.74 $\pm$ 0.03} \\ 
\midrule
\multirow{6}{*}{\rotatebox{90}{CAT}} 
 & Replay & 0.625 $\pm$ 0.015 & 0.190 $\pm$ 0.005 & 36.51 $\pm$ 1.30 & 0.68 $\pm$ 0.02 \\
 & ADV-0  & 0.599 $\pm$ 0.022 & 0.305 $\pm$ 0.042 & 32.91 $\pm$ 1.25 & 0.75 $\pm$ 0.01 \\
 & CAT & 0.608 $\pm$ 0.020 & 0.290 $\pm$ 0.035 & 33.82 $\pm$ 1.35 & 0.73 $\pm$ 0.02 \\
 & SAGE & 0.590 $\pm$ 0.025 & 0.316 $\pm$ 0.035 & 33.19 $\pm$ 1.50 & 0.74 $\pm$ 0.01 \\
 & Heuristic & 0.604 $\pm$ 0.028 & 0.286 $\pm$ 0.033 & 33.09 $\pm$ 1.70 & 0.71 $\pm$ 0.01 \\ 
 \cmidrule(l){2-6}
 & \textbf{Avg.} & \textbf{0.605 $\pm$ 0.022} & \textbf{0.277 $\pm$ 0.030} & \textbf{33.90 $\pm$ 1.42} & \textbf{0.72 $\pm$ 0.01} \\ 
\midrule
\multirow{6}{*}{\rotatebox{90}{Heuristic}} 
 & Replay & 0.630 $\pm$ 0.030 & 0.202 $\pm$ 0.010 & 36.82 $\pm$ 3.00 & 0.69 $\pm$ 0.02 \\
 & ADV-0  & 0.590 $\pm$ 0.065 & 0.333 $\pm$ 0.010 & 33.52 $\pm$ 2.20 & 0.76 $\pm$ 0.04 \\
 & CAT & 0.606 $\pm$ 0.022 & 0.324 $\pm$ 0.020 & 34.27 $\pm$ 2.15 & 0.75 $\pm$ 0.04 \\
 & SAGE & 0.586 $\pm$ 0.070 & 0.354 $\pm$ 0.015 & 32.56 $\pm$ 2.00 & 0.76 $\pm$ 0.01 \\
 & Heuristic & 0.590 $\pm$ 0.029 & 0.290 $\pm$ 0.015 & 32.80 $\pm$ 1.80 & 0.74 $\pm$ 0.02 \\ 
 \cmidrule(l){2-6}
 & \textbf{Avg.} & \textbf{0.600 $\pm$ 0.043} & \textbf{0.301 $\pm$ 0.014} & \textbf{33.99 $\pm$ 2.23} & \textbf{0.74 $\pm$ 0.03} \\ 
\midrule
\multirow{6}{*}{\rotatebox{90}{Replay}} 
 & Replay & 0.620 $\pm$ 0.030 & 0.215 $\pm$ 0.045 & 34.53 $\pm$ 0.90 & 0.70 $\pm$ 0.01 \\
 & ADV-0  & 0.570 $\pm$ 0.040 & 0.364 $\pm$ 0.050 & 30.20 $\pm$ 3.20 & 0.78 $\pm$ 0.01 \\
 & CAT & 0.585 $\pm$ 0.055 & 0.358 $\pm$ 0.040 & 31.53 $\pm$ 1.10 & 0.76 $\pm$ 0.01 \\
 & SAGE & 0.565 $\pm$ 0.060 & 0.358 $\pm$ 0.035 & 29.80 $\pm$ 1.25 & 0.78 $\pm$ 0.03 \\
 & Heuristic & 0.581 $\pm$ 0.045 & 0.320 $\pm$ 0.045 & 30.50 $\pm$ 3.40 & 0.77 $\pm$ 0.02 \\ 
 \cmidrule(l){2-6}
 & \textbf{Avg.} & \textbf{0.584 $\pm$ 0.046} & \textbf{0.323 $\pm$ 0.043} & \textbf{31.31 $\pm$ 1.97} & \textbf{0.76 $\pm$ 0.02} \\ 
\bottomrule
\end{tabular}%
}
\end{minipage}
\hfill 
%%%%%%%%%%%%%%%%%%%%%%%%%%%%%%%%%%%%%%%%%%%%%%%%%%%%%%%%%%%%%%%%%%%%%%%
\begin{minipage}[t]{0.49\linewidth}
\centering
\caption{Cross-validation performances of driving agents learned by SAC. Within each environment, entries report mean $\pm$ standard deviation over five seeds; Avg. averages the means and standard deviations across the five environments.}
\vspace{-5pt}
\label{tab:agent_sac}
\setlength{\tabcolsep}{3pt}
\renewcommand{\arraystretch}{1.12}
\resizebox{\linewidth}{!}{%
\begin{tabular}{llcccc}
\toprule
& \textbf{Val. Env.} & {RC} $\uparrow$ & {Crash} $\downarrow$ & {Reward} $\uparrow$ & {Cost} $\downarrow$\\ 
\midrule
\multirow{6}{*}{\rotatebox{90}{\textbf{ADV-0 (w/ IPL)}}} 
 & Replay & 0.781 $\pm$ 0.011 & 0.130 $\pm$ 0.016 & 53.66 $\pm$ 1.02 & 0.41 $\pm$ 0.01 \\
 & ADV-0 & 0.736 $\pm$ 0.011 & 0.305 $\pm$ 0.022 & 49.70 $\pm$ 0.34 & 0.53 $\pm$ 0.01 \\
 & CAT & 0.745 $\pm$ 0.013 & 0.268 $\pm$ 0.024 & 50.48 $\pm$ 1.02 & 0.52 $\pm$ 0.03 \\
 & SAGE & 0.745 $\pm$ 0.016 & 0.250 $\pm$ 0.025 & 48.84 $\pm$ 1.79 & 0.51 $\pm$ 0.04 \\
 & Heuristic & 0.758 $\pm$ 0.031 & 0.170 $\pm$ 0.019 & 50.48 $\pm$ 2.79 & 0.45 $\pm$ 0.06 \\ 
 \cmidrule(l){2-6}
 & \cellcolor{myblue!30}\textbf{Avg.} & \cellcolor{myblue!30}\textbf{0.753 $\pm$ 0.016} & \cellcolor{myblue!30}\textbf{0.225 $\pm$ 0.021} & \cellcolor{myblue!30}\textbf{50.63 $\pm$ 1.39} & \cellcolor{myblue!30}\textbf{0.48 $\pm$ 0.03} \\ 
\midrule
\multirow{6}{*}{\rotatebox{90}{ADV-0 (w/o IPL)}} 
 & Replay & 0.784 $\pm$ 0.020 & 0.160 $\pm$ 0.022 & 54.40 $\pm$ 1.69 & 0.44 $\pm$ 0.03 \\
 & ADV-0 & 0.706 $\pm$ 0.020 & 0.307 $\pm$ 0.026 & 45.92 $\pm$ 2.59 & 0.58 $\pm$ 0.03 \\
 & CAT & 0.729 $\pm$ 0.011 & 0.287 $\pm$ 0.026 & 47.81 $\pm$ 1.48 & 0.58 $\pm$ 0.03 \\
 & SAGE & 0.743 $\pm$ 0.027 & 0.260 $\pm$ 0.024 & 49.35 $\pm$ 2.73 & 0.50 $\pm$ 0.05 \\
 & Heuristic & 0.743 $\pm$ 0.033 & 0.203 $\pm$ 0.034 & 48.98 $\pm$ 3.48 & 0.49 $\pm$ 0.07 \\ 
 \cmidrule(l){2-6}
 & \cellcolor{myblue!15}\textbf{Avg.} & \cellcolor{myblue!15}\textbf{0.741 $\pm$ 0.022} & \cellcolor{myblue!15}\textbf{0.243 $\pm$ 0.026} & \cellcolor{myblue!15}\textbf{49.29 $\pm$ 2.39} & \cellcolor{myblue!15}\textbf{0.52 $\pm$ 0.04} \\ 
\midrule
\multirow{6}{*}{\rotatebox{90}{CAT}} 
 & Replay & 0.775 $\pm$ 0.013 & 0.177 $\pm$ 0.041 & 54.06 $\pm$ 1.12 & 0.42 $\pm$ 0.00 \\
 & ADV-0 & 0.698 $\pm$ 0.010 & 0.353 $\pm$ 0.034 & 45.81 $\pm$ 0.57 & 0.60 $\pm$ 0.03 \\
 & CAT & 0.704 $\pm$ 0.018 & 0.327 $\pm$ 0.041 & 45.78 $\pm$ 1.55 & 0.60 $\pm$ 0.02 \\
 & SAGE & 0.705 $\pm$ 0.025 & 0.297 $\pm$ 0.029 & 46.65 $\pm$ 3.41 & 0.56 $\pm$ 0.03 \\
 & Heuristic & 0.727 $\pm$ 0.024 & 0.247 $\pm$ 0.038 & 47.48 $\pm$ 2.57 & 0.53 $\pm$ 0.04 \\ 
 \cmidrule(l){2-6}
 & \textbf{Avg.} & \textbf{0.722 $\pm$ 0.018} & \textbf{0.280 $\pm$ 0.037} & \textbf{47.96 $\pm$ 1.84} & \textbf{0.54 $\pm$ 0.02} \\ 
\midrule
\multirow{6}{*}{\rotatebox{90}{Heuristic}} 
 & Replay & 0.718 $\pm$ 0.030 & 0.200 $\pm$ 0.028 & 47.20 $\pm$ 4.13 & 0.54 $\pm$ 0.07 \\
 & ADV-0 & 0.665 $\pm$ 0.025 & 0.330 $\pm$ 0.025 & 42.10 $\pm$ 2.64 & 0.63 $\pm$ 0.02 \\
 & CAT & 0.672 $\pm$ 0.024 & 0.320 $\pm$ 0.022 & 42.50 $\pm$ 2.84 & 0.64 $\pm$ 0.03 \\
 & SAGE & 0.670 $\pm$ 0.030 & 0.295 $\pm$ 0.030 & 42.30 $\pm$ 3.77 & 0.60 $\pm$ 0.05 \\
 & Heuristic & 0.678 $\pm$ 0.028 & 0.275 $\pm$ 0.015 & 43.10 $\pm$ 4.29 & 0.61 $\pm$ 0.05 \\ 
 \cmidrule(l){2-6}
 & \textbf{Avg.} & \textbf{0.681 $\pm$ 0.027} & \textbf{0.284 $\pm$ 0.024} & \textbf{43.44 $\pm$ 3.53} & \textbf{0.60 $\pm$ 0.04} \\ 
\midrule
\multirow{6}{*}{\rotatebox{90}{Replay}} 
 & Replay & 0.710 $\pm$ 0.035 & 0.217 $\pm$ 0.035 & 45.51 $\pm$ 1.30 & 0.56 $\pm$ 0.02 \\
 & ADV-0 & 0.625 $\pm$ 0.045 & 0.360 $\pm$ 0.046 & 37.81 $\pm$ 2.50 & 0.67 $\pm$ 0.03 \\
 & CAT & 0.652 $\pm$ 0.035 & 0.353 $\pm$ 0.041 & 39.20 $\pm$ 2.00 & 0.66 $\pm$ 0.04 \\
 & SAGE & 0.635 $\pm$ 0.029 & 0.332 $\pm$ 0.040 & 38.54 $\pm$ 2.10 & 0.64 $\pm$ 0.03 \\
 & Heuristic & 0.638 $\pm$ 0.027 & 0.316 $\pm$ 0.035 & 39.09 $\pm$ 2.50 & 0.65 $\pm$ 0.01 \\ 
 \cmidrule(l){2-6}
 & \textbf{Avg.} & \textbf{0.652 $\pm$ 0.034} & \textbf{0.316 $\pm$ 0.039} & \textbf{40.03 $\pm$ 2.08} & \textbf{0.64 $\pm$ 0.03} \\ 
\bottomrule
\end{tabular}%
}
\end{minipage}
\end{table*}

%%%%%%%%%%%%%%%%%%%%%%%%%%%%%%%%%%%%%%%%%%%%%%%%%%%%%%%%%%%%%%%%%%%%%%%%%%%%%%%

\begin{figure}[!htbp]
  \begin{center}
    \centerline{\includegraphics[width=0.99\columnwidth]{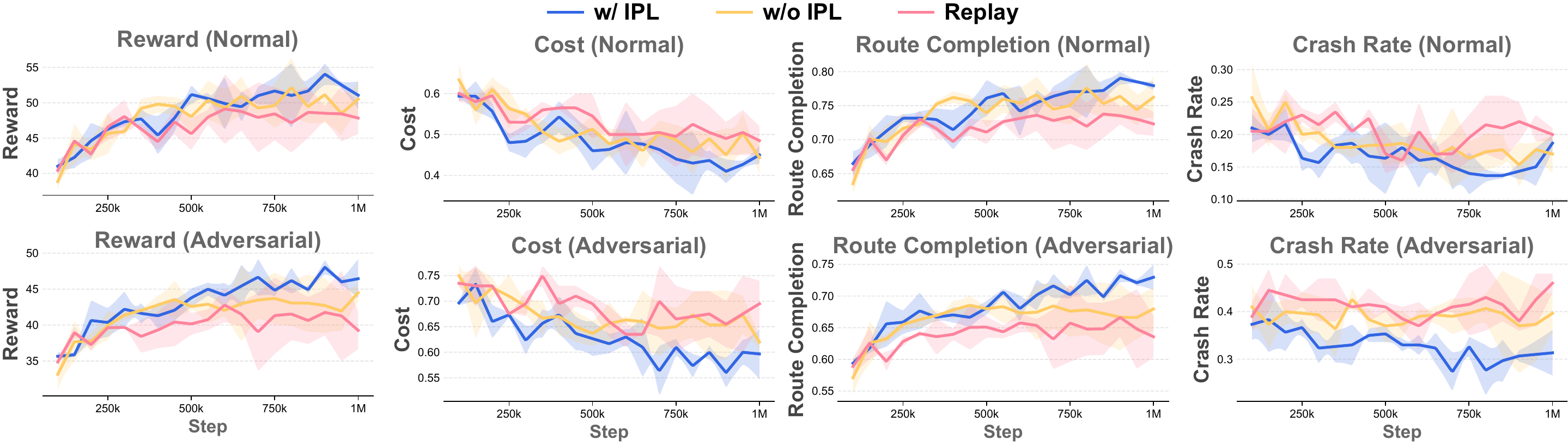}}
    \vspace{-5pt}
    \caption{Learning curves of TD3.}
    \label{fig:td3_curves}
  \end{center}
\end{figure}

%%%%%%%%%%%%%%%%%%%%%%%%%%%%%%%%%%%%%%%%%%%%%%%%%%%%%%%%%%%%%%%%%%%%%%%

\begin{table*}[!htbp]
\centering
%%%%%%%%%%%%%%%%%%%%%%%%%%%%%%%%%%%%%%%%%%%%%%%%%%%%%%%%%%%%%%%%%%%%%%%
\begin{minipage}[t]{0.49\linewidth} 
\centering
\caption{Cross-validation performances of driving agents learned by SAC-Lag. Within each environment, entries report mean $\pm$ standard deviation over five seeds; Avg. averages the means and standard deviations across the five environments.}
\vspace{-5pt}
\label{tab:agent_saclag}
\setlength{\tabcolsep}{3pt}
\renewcommand{\arraystretch}{1.12}
\resizebox{\linewidth}{!}{%
\begin{tabular}{llcccc}
\toprule
& \textbf{Val. Env.} & {RC} $\uparrow$ & {Crash} $\downarrow$ & {Reward} $\uparrow$ & {Cost} $\downarrow$\\ 
\midrule
\multirow{6}{*}{\rotatebox{90}{\textbf{ADV-0 (w/ IPL)}}} 
 & Replay & 0.787 $\pm$ 0.006 & 0.127 $\pm$ 0.017 & 54.62 $\pm$ 0.26 & 0.40 $\pm$ 0.02 \\
 & ADV-0 & 0.729 $\pm$ 0.002 & 0.303 $\pm$ 0.042 & 48.78 $\pm$ 0.35 & 0.54 $\pm$ 0.03 \\
 & CAT & 0.725 $\pm$ 0.011 & 0.297 $\pm$ 0.019 & 48.01 $\pm$ 1.10 & 0.57 $\pm$ 0.02 \\
 & SAGE & 0.736 $\pm$ 0.007 & 0.250 $\pm$ 0.024 & 48.88 $\pm$ 0.58 & 0.52 $\pm$ 0.01 \\
 & Heuristic & 0.742 $\pm$ 0.022 & 0.217 $\pm$ 0.045 & 49.25 $\pm$ 2.68 & 0.49 $\pm$ 0.06 \\ 
 \cmidrule(l){2-6}
 & \cellcolor{myblue!30}\textbf{Avg.} & \cellcolor{myblue!30}\textbf{0.744 $\pm$ 0.010} & \cellcolor{myblue!30}\textbf{0.239 $\pm$ 0.029} & \cellcolor{myblue!30}\textbf{49.91 $\pm$ 0.99} & \cellcolor{myblue!30}\textbf{0.50 $\pm$ 0.03} \\ 
\midrule
\multirow{6}{*}{\rotatebox{90}{ADV-0 (w/o IPL)}}
 & Replay & 0.776 $\pm$ 0.009 & 0.135 $\pm$ 0.025 & 53.75 $\pm$ 1.29 & 0.42 $\pm$ 0.02 \\
 & ADV-0 & 0.689 $\pm$ 0.004 & 0.320 $\pm$ 0.020 & 44.41 $\pm$ 0.98 & 0.61 $\pm$ 0.01 \\
 & CAT & 0.699 $\pm$ 0.002 & 0.320 $\pm$ 0.000 & 45.36 $\pm$ 0.43 & 0.61 $\pm$ 0.03 \\
 & SAGE & 0.707 $\pm$ 0.009 & 0.260 $\pm$ 0.030 & 46.63 $\pm$ 0.19 & 0.56 $\pm$ 0.02 \\
 & Heuristic & 0.730 $\pm$ 0.013 & 0.205 $\pm$ 0.035 & 47.97 $\pm$ 0.34 & 0.49 $\pm$ 0.03 \\ 
 \cmidrule(l){2-6}
 & \cellcolor{myblue!15}\textbf{Avg.} & \cellcolor{myblue!15}\textbf{0.720 $\pm$ 0.007} & \cellcolor{myblue!15}\textbf{0.248 $\pm$ 0.022} & \cellcolor{myblue!15}\textbf{47.62 $\pm$ 0.65} & \cellcolor{myblue!15}\textbf{0.54 $\pm$ 0.02} \\ 
\midrule
\multirow{6}{*}{\rotatebox{90}{CAT}} 
 & Replay & 0.765 $\pm$ 0.015 & 0.165 $\pm$ 0.036 & 52.80 $\pm$ 1.15 & 0.43 $\pm$ 0.01 \\
 & ADV-0 & 0.685 $\pm$ 0.012 & 0.345 $\pm$ 0.035 & 44.55 $\pm$ 0.61 & 0.61 $\pm$ 0.03 \\
 & CAT & 0.692 $\pm$ 0.018 & 0.320 $\pm$ 0.047 & 44.78 $\pm$ 1.50 & 0.61 $\pm$ 0.02 \\
 & SAGE & 0.695 $\pm$ 0.025 & 0.305 $\pm$ 0.033 & 45.87 $\pm$ 3.22 & 0.57 $\pm$ 0.03 \\
 & Heuristic & 0.712 $\pm$ 0.024 & 0.255 $\pm$ 0.042 & 46.50 $\pm$ 2.52 & 0.54 $\pm$ 0.04 \\ 
 \cmidrule(l){2-6}
 & \textbf{Avg.} & \textbf{0.710 $\pm$ 0.019} & \textbf{0.278 $\pm$ 0.039} & \textbf{46.90 $\pm$ 1.80} & \textbf{0.55 $\pm$ 0.03} \\ 
\midrule
\multirow{6}{*}{\rotatebox{90}{Heuristic}} 
 & Replay & 0.731 $\pm$ 0.030 & 0.193 $\pm$ 0.024 & 48.36 $\pm$ 4.17 & 0.53 $\pm$ 0.07 \\
 & ADV-0 & 0.677 $\pm$ 0.022 & 0.320 $\pm$ 0.024 & 43.47 $\pm$ 2.58 & 0.62 $\pm$ 0.02 \\
 & CAT & 0.686 $\pm$ 0.024 & 0.313 $\pm$ 0.021 & 43.81 $\pm$ 2.86 & 0.63 $\pm$ 0.03 \\
 & SAGE & 0.687 $\pm$ 0.027 & 0.283 $\pm$ 0.031 & 43.61 $\pm$ 3.71 & 0.59 $\pm$ 0.05 \\
 & Heuristic & 0.691 $\pm$ 0.028 & 0.267 $\pm$ 0.012 & 44.19 $\pm$ 4.19 & 0.60 $\pm$ 0.05 \\ 
 \cmidrule(l){2-6}
 & \textbf{Avg.} & \textbf{0.694 $\pm$ 0.026} & \textbf{0.275 $\pm$ 0.022} & \textbf{44.69 $\pm$ 3.50} & \textbf{0.59 $\pm$ 0.04} \\ 
\midrule
\multirow{6}{*}{\rotatebox{90}{Replay}} 
 & Replay & 0.718 $\pm$ 0.035 & 0.205 $\pm$ 0.034 & 46.25 $\pm$ 1.25 & 0.55 $\pm$ 0.02 \\
 & ADV-0 & 0.635 $\pm$ 0.047 & 0.355 $\pm$ 0.045 & 38.80 $\pm$ 2.55 & 0.66 $\pm$ 0.03 \\
 & CAT & 0.665 $\pm$ 0.033 & 0.345 $\pm$ 0.040 & 40.15 $\pm$ 2.00 & 0.65 $\pm$ 0.04 \\
 & SAGE & 0.645 $\pm$ 0.031 & 0.325 $\pm$ 0.040 & 39.50 $\pm$ 2.11 & 0.63 $\pm$ 0.03 \\
 & Heuristic & 0.650 $\pm$ 0.029 & 0.305 $\pm$ 0.036 & 39.80 $\pm$ 2.54 & 0.64 $\pm$ 0.01 \\ 
 \cmidrule(l){2-6}
 & \textbf{Avg.} & \textbf{0.663 $\pm$ 0.035} & \textbf{0.307 $\pm$ 0.039} & \textbf{40.90 $\pm$ 2.09} & \textbf{0.63 $\pm$ 0.03} \\ 
\bottomrule
\end{tabular}%
}
\end{minipage}
\hfill 
%%%%%%%%%%%%%%%%%%%%%%%%%%%%%%%%%%%%%%%%%%%%%%%%%%%%%%%%%%%%%%%%%%%%%%%
\begin{minipage}[t]{0.49\linewidth}
\centering
\caption{Cross-validation performances of driving agents learned by TD3. Within each environment, entries report mean $\pm$ standard deviation over five seeds; Avg. averages the means and standard deviations across the five environments.}
\vspace{-5pt}
\label{tab:agent_td3}
\setlength{\tabcolsep}{3pt}
\renewcommand{\arraystretch}{1.12}
\resizebox{\linewidth}{!}{%
\begin{tabular}{llcccc}
\toprule
& \textbf{Val. Env.} & {RC} $\uparrow$ & {Crash} $\downarrow$ & {Reward} $\uparrow$ & {Cost} $\downarrow$\\ 
\midrule
\multirow{6}{*}{\rotatebox{90}{\textbf{ADV-0 (w/ IPL)}}} 
 & Replay & 0.787 $\pm$ 0.002 & 0.187 $\pm$ 0.031 & 53.15 $\pm$ 0.76 & 0.43 $\pm$ 0.01 \\
 & ADV-0 & 0.711 $\pm$ 0.016 & 0.323 $\pm$ 0.038 & 45.54 $\pm$ 1.70 & 0.59 $\pm$ 0.03 \\
 & CAT & 0.724 $\pm$ 0.016 & 0.300 $\pm$ 0.029 & 46.67 $\pm$ 1.55 & 0.58 $\pm$ 0.02 \\
 & SAGE & 0.734 $\pm$ 0.016 & 0.267 $\pm$ 0.046 & 47.69 $\pm$ 1.52 & 0.53 $\pm$ 0.05 \\
 & Heuristic & 0.757 $\pm$ 0.021 & 0.200 $\pm$ 0.014 & 49.87 $\pm$ 2.64 & 0.49 $\pm$ 0.02 \\ 
 \cmidrule(l){2-6}
 & \cellcolor{myblue!30}\textbf{Avg.} & \cellcolor{myblue!30}\textbf{0.743 $\pm$ 0.014} & \cellcolor{myblue!30}\textbf{0.255 $\pm$ 0.032} & \cellcolor{myblue!30}\textbf{48.58 $\pm$ 1.63} & \cellcolor{myblue!30}\textbf{0.52 $\pm$ 0.03} \\ 
\midrule
\multirow{6}{*}{\rotatebox{90}{ADV-0 (w/o IPL)}} 
 & Replay & 0.761 $\pm$ 0.031 & 0.160 $\pm$ 0.024 & 50.21 $\pm$ 2.09 & 0.44 $\pm$ 0.01 \\
 & ADV-0 & 0.644 $\pm$ 0.046 & 0.423 $\pm$ 0.057 & 40.86 $\pm$ 2.90 & 0.67 $\pm$ 0.03 \\
 & CAT & 0.689 $\pm$ 0.036 & 0.373 $\pm$ 0.077 & 45.46 $\pm$ 2.04 & 0.59 $\pm$ 0.05 \\
 & SAGE & 0.669 $\pm$ 0.049 & 0.337 $\pm$ 0.068 & 41.95 $\pm$ 3.39 & 0.62 $\pm$ 0.05 \\
 & Heuristic & 0.685 $\pm$ 0.046 & 0.277 $\pm$ 0.069 & 43.25 $\pm$ 4.07 & 0.61 $\pm$ 0.08 \\ 
 \cmidrule(l){2-6}
 & \cellcolor{myblue!15}\textbf{Avg.} & \cellcolor{myblue!15}\textbf{0.690 $\pm$ 0.042} & \cellcolor{myblue!15}\textbf{0.314 $\pm$ 0.059} & \cellcolor{myblue!15}\textbf{44.35 $\pm$ 2.90} & \cellcolor{myblue!15}\textbf{0.59 $\pm$ 0.04} \\ 
\midrule
\multirow{6}{*}{\rotatebox{90}{CAT}} 
 & Replay & 0.755 $\pm$ 0.013 & 0.160 $\pm$ 0.029 & 49.03 $\pm$ 0.97 & 0.48 $\pm$ 0.01 \\
 & ADV-0 & 0.668 $\pm$ 0.017 & 0.367 $\pm$ 0.033 & 42.60 $\pm$ 0.23 & 0.65 $\pm$ 0.02 \\
 & CAT & 0.673 $\pm$ 0.006 & 0.343 $\pm$ 0.005 & 42.46 $\pm$ 1.60 & 0.65 $\pm$ 0.02 \\
 & SAGE & 0.687 $\pm$ 0.016 & 0.300 $\pm$ 0.008 & 43.99 $\pm$ 2.09 & 0.58 $\pm$ 0.01 \\
 & Heuristic & 0.706 $\pm$ 0.009 & 0.237 $\pm$ 0.029 & 45.21 $\pm$ 0.88 & 0.56 $\pm$ 0.03 \\ 
 \cmidrule(l){2-6}
 & \textbf{Avg.} & \textbf{0.698 $\pm$ 0.012} & \textbf{0.281 $\pm$ 0.021} & \textbf{44.66 $\pm$ 1.15} & \textbf{0.58 $\pm$ 0.02} \\ 
\midrule
\multirow{6}{*}{\rotatebox{90}{Heuristic}} 
 & Replay & 0.713 $\pm$ 0.030 & 0.210 $\pm$ 0.016 & 45.57 $\pm$ 3.01 & 0.52 $\pm$ 0.02 \\
 & ADV-0 & 0.632 $\pm$ 0.051 & 0.380 $\pm$ 0.010 & 40.47 $\pm$ 2.24 & 0.65 $\pm$ 0.04 \\
 & CAT & 0.661 $\pm$ 0.022 & 0.320 $\pm$ 0.020 & 42.40 $\pm$ 2.18 & 0.61 $\pm$ 0.04 \\
 & SAGE & 0.654 $\pm$ 0.045 & 0.315 $\pm$ 0.015 & 42.51 $\pm$ 1.99 & 0.63 $\pm$ 0.01 \\
 & Heuristic & 0.652 $\pm$ 0.047 & 0.305 $\pm$ 0.022 & 41.35 $\pm$ 1.85 & 0.60 $\pm$ 0.02 \\ 
 \cmidrule(l){2-6}
 & \textbf{Avg.} & \textbf{0.662 $\pm$ 0.039} & \textbf{0.306 $\pm$ 0.017} & \textbf{42.46 $\pm$ 2.25} & \textbf{0.60 $\pm$ 0.03} \\ 
\midrule
\multirow{6}{*}{\rotatebox{90}{Replay}} 
 & Replay & 0.727 $\pm$ 0.026 & 0.210 $\pm$ 0.010 & 47.77 $\pm$ 3.00 & 0.51 $\pm$ 0.07 \\
 & ADV-0 & 0.638 $\pm$ 0.019 & 0.435 $\pm$ 0.005 & 39.01 $\pm$ 2.08 & 0.67 $\pm$ 0.02 \\
 & CAT & 0.643 $\pm$ 0.038 & 0.455 $\pm$ 0.015 & 39.96 $\pm$ 3.15 & 0.69 $\pm$ 0.04 \\
 & SAGE & 0.654 $\pm$ 0.035 & 0.340 $\pm$ 0.050 & 41.54 $\pm$ 2.42 & 0.59 $\pm$ 0.04 \\
 & Heuristic & 0.700 $\pm$ 0.020 & 0.280 $\pm$ 0.010 & 45.02 $\pm$ 1.55 & 0.56 $\pm$ 0.01 \\ 
 \cmidrule(l){2-6}
 & \textbf{Avg.} & \textbf{0.672 $\pm$ 0.028} & \textbf{0.344 $\pm$ 0.018} & \textbf{42.66 $\pm$ 2.44} & \textbf{0.60 $\pm$ 0.04} \\ 
\bottomrule
\end{tabular}%
}
\end{minipage}
\end{table*}

\begin{figure}[!htbp]
  \begin{center}
    \centerline{\includegraphics[width=0.99\columnwidth]{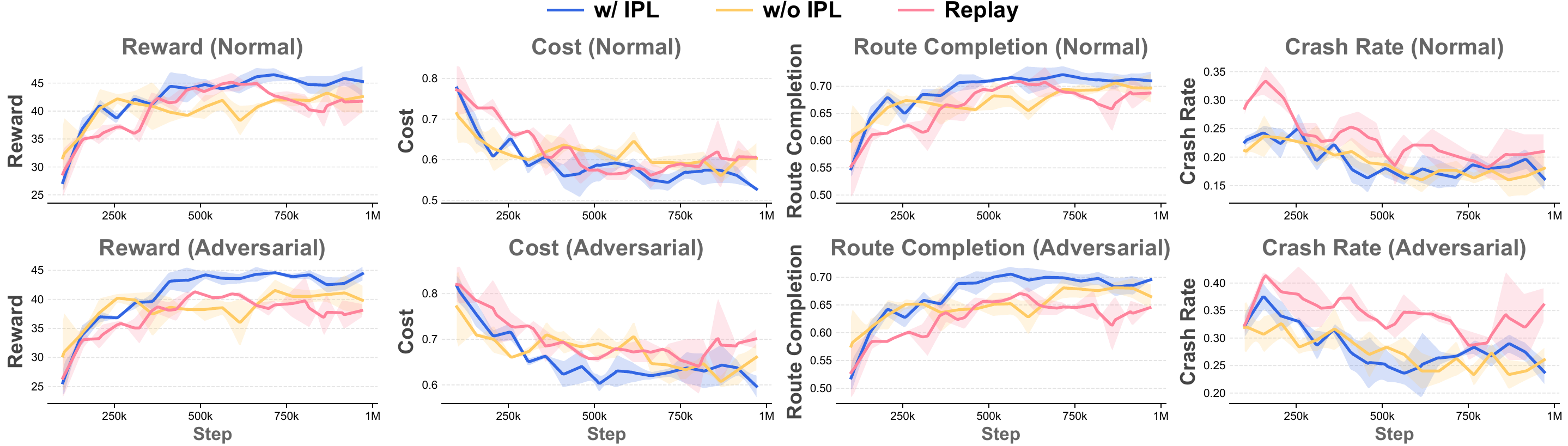}}
    % \vspace{-5pt}
    \caption{Learning curves of GRPO.}
    \label{fig:grpo_curves}
  \end{center}
\end{figure}

%%%%%%%%%%%%%%%%%%%%%%%%%%%%%%%%%%%%%%%%%%%%%%%%%%%%%%%%%%%%%%%%%%%%%%%%%%%%%%%

\vspace{-10pt}
\begin{table}[!htbp]
\centering
\caption{Performance comparison of learning-based planners before and after adversarial fine-tuning using ADV-0 (GRPO).}
% \vspace{-5pt}
\label{tab:planner_results_all}
\setlength{\tabcolsep}{10pt}
\begin{small}
\resizebox{0.8\textwidth}{!}{%
\begin{tabular}{l|l|cccc}
\toprule
\textbf{Model Phase} & \textbf{Val. Env.} & {RC $\uparrow$} & {Crash $\downarrow$} & {Reward $\uparrow$} & {Cost $\downarrow$} \\ 
\midrule
\multirow{6}{*}{{PlanTF (Pretrained)}} 
 & Replay & 0.727 $\pm$ 0.021 & 0.220 $\pm$ 0.027 & 41.53 $\pm$ 3.10 & 0.76 $\pm$ 0.03 \\
 & ADV-0 & 0.581 $\pm$ 0.030 & 0.420 $\pm$ 0.035 & 32.11 $\pm$ 2.54 & 1.24 $\pm$ 0.05 \\
 & CAT & 0.615 $\pm$ 0.025 & 0.385 $\pm$ 0.030 & 35.26 $\pm$ 2.16 & 1.09 $\pm$ 0.03 \\
 & SAGE & 0.595 $\pm$ 0.028 & 0.407 $\pm$ 0.032 & 33.80 $\pm$ 2.80 & 1.15 $\pm$ 0.04 \\
 & Heuristic & 0.620 $\pm$ 0.020 & 0.352 $\pm$ 0.032 & 36.53 $\pm$ 1.97 & 0.94 $\pm$ 0.03 \\ 
 \cmidrule(l){2-6}
 & \textbf{Avg.} & \textbf{0.628 $\pm$ 0.025} & \textbf{0.357 $\pm$ 0.031} & \textbf{35.85 $\pm$ 2.51} & \textbf{1.04 $\pm$ 0.04}   \\ 
\midrule
\multirow{6}{*}{{PlanTF (Fine-tuned)}} 
 & Replay & 0.738 $\pm$ 0.015 & 0.199 $\pm$ 0.018 & 46.20 $\pm$ 1.88 & 0.62 $\pm$ 0.02 \\
 & ADV-0 & 0.655 $\pm$ 0.012 & 0.293 $\pm$ 0.022 & 40.80 $\pm$ 1.20 & 0.85 $\pm$ 0.00 \\
 & CAT & 0.668 $\pm$ 0.010 & 0.273 $\pm$ 0.020 & 42.11 $\pm$ 1.16 & 0.78 $\pm$ 0.01 \\
 & SAGE & 0.640 $\pm$ 0.018 & 0.305 $\pm$ 0.025 & 39.50 $\pm$ 1.56 & 0.88 $\pm$ 0.03 \\
 & Heuristic & 0.671 $\pm$ 0.024 & 0.246 $\pm$ 0.038 & 41.28 $\pm$ 2.32 & 0.72 $\pm$ 0.05 \\ 
 \cmidrule(l){2-6}
 & \textbf{Avg.} & \textbf{0.674 $\pm$ 0.016} & \textbf{0.263 $\pm$ 0.025} & \textbf{41.98 $\pm$ 1.62} & \textbf{0.77 $\pm$ 0.02} \\ 
\midrule
\multicolumn{2}{c|}{\textbf{Average Relative Change}} & \cellcolor{myteal!15}\textbf{+7.46\%} & \cellcolor{myteal!30}\textbf{-26.23\%} & \cellcolor{myteal!20}\textbf{+17.11\%} & \cellcolor{myteal!25}\textbf{-25.68\%} \\
\cmidrule[0.9pt]{1-6}
\multirow{6}{*}{{SMART (Pretrained)}} 
 & Replay & 0.686 $\pm$ 0.020 & 0.255 $\pm$ 0.023 & 38.59 $\pm$ 3.51 & 0.83 $\pm$ 0.05 \\
 & ADV-0 & 0.540 $\pm$ 0.035 & 0.460 $\pm$ 0.049 & 29.80 $\pm$ 2.99 & 1.35 $\pm$ 0.04 \\
 & CAT & 0.565 $\pm$ 0.031 & 0.433 $\pm$ 0.033 & 31.50 $\pm$ 2.43 & 1.23 $\pm$ 0.05 \\
 & SAGE & 0.556 $\pm$ 0.032 & 0.448 $\pm$ 0.038 & 30.27 $\pm$ 3.10 & 1.30 $\pm$ 0.04 \\
 & Heuristic & 0.586 $\pm$ 0.029 & 0.384 $\pm$ 0.026 & 33.16 $\pm$ 2.20 & 1.05 $\pm$ 0.07 \\ 
 \cmidrule(l){2-6}
 & \textbf{Avg.}  & \textbf{0.587 $\pm$ 0.029} & \textbf{0.396 $\pm$ 0.034} & \textbf{32.66 $\pm$ 2.85} & \textbf{1.15 $\pm$ 0.05} \\  
\midrule
\multirow{6}{*}{{SMART (Fine-tuned)}} 
 & Replay & 0.705 $\pm$ 0.010 & 0.230 $\pm$ 0.026 & 42.13 $\pm$ 2.11 & 0.77 $\pm$ 0.01 \\
 & ADV-0 & 0.610 $\pm$ 0.018 & 0.340 $\pm$ 0.022 & 36.51 $\pm$ 1.56 & 0.97 $\pm$ 0.02 \\
 & CAT & 0.625 $\pm$ 0.012 & 0.313 $\pm$ 0.022 & 37.82 $\pm$ 1.40 & 0.95 $\pm$ 0.03 \\
 & SAGE & 0.595 $\pm$ 0.020 & 0.352 $\pm$ 0.028 & 35.20 $\pm$ 1.82 & 1.05 $\pm$ 0.03 \\
 & Heuristic & 0.620 $\pm$ 0.016 & 0.292 $\pm$ 0.018 & 37.57 $\pm$ 1.63 & 0.85 $\pm$ 0.03 \\ 
 \cmidrule(l){2-6}
 & \textbf{Avg.} & \textbf{0.631 $\pm$ 0.015} & \textbf{0.305 $\pm$ 0.023} & \textbf{37.85 $\pm$ 1.70} & \textbf{0.92 $\pm$ 0.02} \\ 
\midrule
\multicolumn{2}{c|}{\textbf{Average Relative Change}} & \cellcolor{myteal!15}\textbf{+7.57\%} & \cellcolor{myteal!30}\textbf{-22.88\%} & \cellcolor{myteal!20}\textbf{+15.86\%} & \cellcolor{myteal!25}\textbf{-20.31\%} \\ 
\bottomrule
\end{tabular}%
}
\end{small}
\end{table}

%%%%%%%%%%%%%%%%%%%%%%%%%%%%%%%%%%%%%%%%%%%%%%%%%%%%%%%%%%%%%%%%%%%%%%%%%%%%%%%
\vspace{-10pt}
\begin{figure}[htbp]
    \centering
    \begin{minipage}[t]{0.49\linewidth}
        \centering
        \includegraphics[width=0.6\linewidth]{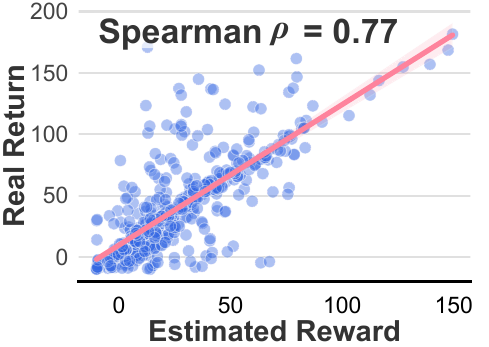}
        \caption{\textbf{Validation of proxy reward estimator.} Comparison of the estimated returns against the ground-truth. The strong Spearman correlation ($\rho = 0.77$) suggests that our rule-based proxy effectively preserves the preference ranking of adversarial candidates.}
        \label{fig:proxy_correlation}
    \end{minipage}
    \hfill
    \begin{minipage}[t]{0.49\linewidth}
        \centering
        \includegraphics[width=0.6\linewidth]{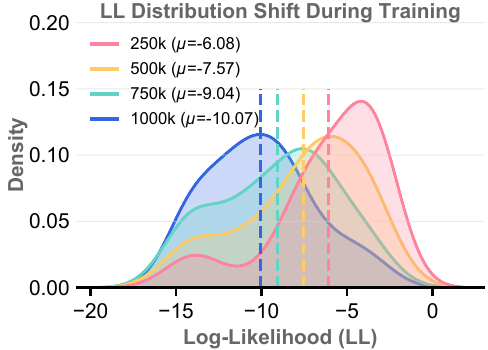}
        \caption{\textbf{Evolution of the adversary distribution}. Likelihood of adversarial trajectories at different training steps. As the ego improves, the distribution shifts towards lower values, suggesting that \texttt{ADV-0} actively identifies nonstationary failure boundary.}
        \label{fig:ll_shift}
    \end{minipage}
\end{figure}
\vspace{-10pt}

\vspace{-10pt}
\begin{table}[!htbp]
\centering
\caption{\textbf{Detailed breakdown of the trajectory-level reward model values}. Values represent the average accumulated discounted reward (over a 2.0s horizon with discount factor $\gamma=0.99$) of the \textit{best} planned trajectory selected by the model across all validation steps. \textbf{Reward components \& weights}: \textbf{Progress} ($w_\text{prog}=1.0$, longitudinal advance), \textbf{Collision} ($w_\text{coll}=20.0$, penalty for overlap with objects), \textbf{Off-road} ($w_\text{off}=5.0$, penalty for lane deviation), \textbf{Comfort} ($w_\text{acc}=0.1, w_\text{jerk}=0.1$, penalties for harsh dynamics), \textbf{Speed Efficiency} ($w_\text{eff}=0.2$, penalty for deviating from 10m/s). Note that the negative improvement in Speed Efficiency reflects a safety-efficiency trade-off, where the fine-tuned planner adopts a more conservative velocity profile to satisfy safety constraints.}
% \vspace{-5pt}
\label{tab:planner_breakdown}
\setlength{\tabcolsep}{6pt}
\begin{small}
\resizebox{0.8\textwidth}{!}{%
\begin{tabular}{l|ccccc|c}
\toprule
\multirow{2}{*}{\textbf{Method}} & \textbf{Progress} & \textbf{Collision} & \textbf{Off-road} & \textbf{Comfort} & \textbf{Speed Eff.} & \multirow{2}{*}{\textbf{Total Score}} \\
 & {( + )} & {( - )} & {( - )} & {( - )} & {( - )} & \\
\midrule
{PlanTF (Pretrained)} 
 & 14.17 $\pm$ 1.53 & -2.85 $\pm$ 0.92 & -0.12 $\pm$ 0.05 & -0.34 $\pm$ 0.08 & -0.54 $\pm$ 0.12 & \textbf{10.32 $\pm$ 1.79} \\
{PlanTF (Fine-tuned)} 
 & 16.45 $\pm$ 0.85 & -0.89 $\pm$ 0.45 & -0.05 $\pm$ 0.04 & -0.24 $\pm$ 0.05 & -0.64 $\pm$ 0.13 & \textbf{14.63 $\pm$ 0.97} \\
\midrule
\textbf{Improvement} & \cellcolor{myteal!10}\textbf{+16.09\%} & \cellcolor{myteal!30}\textbf{+68.77\%} & \cellcolor{myteal!25}\textbf{+58.33\%} & \cellcolor{myteal!15}\textbf{+29.41\%} & \cellcolor{myyellow!30}-18.52\% & \cellcolor{myteal!20}\textbf{+41.76\%} \\
\cmidrule[0.9pt]{1-7}
{SMART (Pretrained)} 
 & 13.20 $\pm$ 1.80 & -3.61 $\pm$ 1.11 & -0.25 $\pm$ 0.08 & -1.26 $\pm$ 0.25 & -0.60 $\pm$ 0.15 & \textbf{7.48 $\pm$ 2.14} \\
{SMART (Fine-tuned)}  
 & 15.10 $\pm$ 1.10 & -1.75 $\pm$ 0.37 & -0.17 $\pm$ 0.05 & -0.92 $\pm$ 0.19 & -0.65 $\pm$ 0.11 & \textbf{11.61 $\pm$ 1.18} \\
 \midrule
\textbf{Improvement} & \cellcolor{myteal!10}\textbf{+14.39\%} & \cellcolor{myteal!30}\textbf{+51.52\%} & \cellcolor{myteal!20}\textbf{+32.00\%} & \cellcolor{myteal!15}\textbf{+26.98\%} & \cellcolor{myyellow!20}-8.33\% & \cellcolor{myteal!20}\textbf{+55.21\%} \\
\bottomrule
\end{tabular}%
}
\end{small}
\end{table}

% \begin{figure}[!htbp]
%     \centering
%     \includegraphics[width=0.3\linewidth]{figs/reward_correlation_plot.pdf}
%     \caption{\textbf{Validation of proxy reward estimator.} Comparison of the estimated rewards against the ground-truth returns. The strong Spearman correlation ($\rho = 0.77$) suggests that our rule-based proxy effectively preserves the preference ranking of adversarial candidates.}
%     \label{fig:proxy_correlation}
% \end{figure}

% \begin{figure}[!htbp]
%     \centering
%     \includegraphics[width=0.3\linewidth]{figs/ll_shift_distribution.pdf}
%     \caption{\textbf{Evolution of the adversarial distribution during training}. The likelihood distribution of generated adversarial trajectories at different training steps (from 250k to 1000k). As the ego agent improves, the distribution mean progressively shifts towards lower values, suggesting that \texttt{ADV-0} actively explores the long tail to identify nonstationary failure boundary of the defender.}
%     \label{fig:ll_shift}
% \end{figure}

%%%%%%%%%%%%%%%%%%%%%%%%%%%%%%%%%%%%%%%%%%%%%%%%%%%%%%%%%%%%%%%%%%%%%%%%%%%%%%%

\begin{table}[htbp]
\centering
\caption{\textbf{Full results of quantitative evaluation on the unbiased long-tailed set mined from real-world data.} The benchmark consists of four long-tail scenario categories filtered by strict physical thresholds: \textit{Critical TTC} ($\min \text{TTC} < 0.4s$), \textit{Critical PET} ($\text{PET} < 1.0s$), \textit{Hard Dynamics} (Longitudinal Acc $< -4.0 m/s^2$ or $|\text{Jerk}| > 4.0 m/s^3$), and \textit{Rare Cluster} (topologically sparse trajectory clusters). \textbf{Reactive Traffic} denotes whether background vehicles utilize IDM/MOBIL policies to interact with the agent ($\checkmark$) or strictly follow logged trajectories ($\times$). Metrics assess \textbf{Safety Margin} (higher values indicate earlier risk detection), \textbf{Stability \& Comfort} (lower Jerk indicates smoother control), and \textbf{Defensive Driving} performance, quantified by Near-Miss Rate (hazardous proximity without collision) and RDP Violation (percentage of time requiring deceleration $> 6 m/s^2$ to avoid collision).}
% \vspace{-5pt}
\label{tab:long_tail}
\setlength{\tabcolsep}{2.5pt}
\renewcommand{\arraystretch}{1.1}
\begin{small}
\resizebox{0.95\textwidth}{!}{%
\begin{tabular}{llcccccccc}
\toprule
\multirow{2}[3]{*}{\rotatebox{90}{}} & \multirow{2}[3]{*}{\makecell{\textbf{Scenario}\\\textbf{Category}}} & \multirow{2}[3]{*}{\textbf{Percentage}} & \multirow{2}[3]{*}{\makecell{\textbf{Reactive}\\\textbf{Traffic}}} & \multicolumn{2}{c}{\textbf{Safety Margin}} & \multicolumn{2}{c}{\textbf{Stability \& Comfort}} & \multicolumn{2}{c}{\textbf{Defensive Driving}} \\
\cmidrule(lr){5-6} \cmidrule(lr){7-8} \cmidrule(lr){9-10}
 & & & & \makecell{\textbf{Avg Min-TTC}\\($\uparrow$)} & \makecell{\textbf{Avg Min-PET}\\($\uparrow$)} & \makecell{\textbf{Mean Abs Jerk}\\($\downarrow$)} & \makecell{\textbf{95\% Jerk}\\($\downarrow$)} & \makecell{\textbf{Near-Miss Rate}\\($\downarrow$)} & \makecell{\textbf{RDP Violation Rate}\\($\downarrow$)} \\
\midrule
\multirow{10}[6]{*}{\rotatebox{90}{ADV-0 (w/ IPL)}} 
 & Critical TTC & 7.40\% & \multirow{4}{*}{$\checkmark$} & $0.645 \pm 0.122$ & $1.450 \pm 0.50$ & $1.853 \pm 0.188$ & $5.850 \pm 1.252$ & $65.28\% \pm 1.45\%$ & $15.53\% \pm 4.16\%$ \\
 & Critical PET & 3.40\% & & $0.492 \pm 0.025$ & $1.150 \pm 0.20$ & $1.623 \pm 0.095$ & $5.157 \pm 0.653$ & $74.57\% \pm 5.52\%$ & $58.26\% \pm 3.29\%$ \\
 & Hard Dynamics & 3.20\% & & $0.885 \pm 0.157$ & N/A & $1.685 \pm 0.250$ & $4.555 \pm 0.951$ & $55.40\% \pm 8.26\%$ & $29.54\% \pm 4.80\%$ \\
 & Rare Cluster & 5.20\% & & $1.951 \pm 0.450$ & $0.850 \pm 0.25$ & $1.451 \pm 0.143$ & $4.650 \pm 0.684$ & $58.55\% \pm 2.12\%$ & $31.48\% \pm 2.81\%$ \\
\cmidrule(lr){2-10}
 & \multicolumn{3}{c}{\cellcolor{myblue!20}\textbf{Avg}} & \cellcolor{myblue!20}$\mathbf{0.993 \pm 0.189}$ & \cellcolor{myblue!20}$\mathbf{1.150 \pm 0.317}$ & \cellcolor{myblue!20}$\mathbf{1.653 \pm 0.169}$ & \cellcolor{myblue!20}$\mathbf{5.053 \pm 0.885}$ & \cellcolor{myblue!20}$\mathbf{63.45\% \pm 4.34\%}$ & \cellcolor{myblue!20}$\mathbf{33.70\% \pm 3.77\%}$ \\
\cmidrule(lr){2-10}
 & Critical TTC & 7.40\% & \multirow{4}{*}{$\times$} & $0.658 \pm 0.130$ & $1.410 \pm 0.48$ & $1.812 \pm 0.165$ & $5.784 \pm 1.102$ & $62.15\% \pm 1.32\%$ & $14.80\% \pm 3.83\%$ \\
 & Critical PET & 3.40\% & & $2.550 \pm 0.036$ & $1.080 \pm 0.22$ & $1.651 \pm 0.117$ & $5.923 \pm 0.823$ & $68.40\% \pm 4.86\%$ & $65.13\% \pm 2.56\%$ \\
 & Hard Dynamics & 3.20\% & & $0.850 \pm 0.140$ & N/A & $1.525 \pm 0.380$ & $4.421 \pm 1.259$ & $58.11\% \pm 7.52\%$ & $30.23\% \pm 3.50\%$ \\
 & Rare Cluster & 5.20\% & & $2.051 \pm 0.422$ & $0.820 \pm 0.24$ & $1.480 \pm 0.155$ & $4.580 \pm 0.626$ & $55.23\% \pm 2.05\%$ & $35.63\% \pm 1.51\%$ \\
\cmidrule(lr){2-10}
 & \multicolumn{3}{c}{\cellcolor{myblue!20}\textbf{Avg}} & \cellcolor{myblue!20}$\mathbf{1.527 \pm 0.182}$ & \cellcolor{myblue!20}$\mathbf{1.103 \pm 0.313}$ & \cellcolor{myblue!20}$\mathbf{1.617 \pm 0.204}$ & \cellcolor{myblue!20}$\mathbf{5.177 \pm 0.953}$ & \cellcolor{myblue!20}$\mathbf{60.97\% \pm 3.94\%}$ & \cellcolor{myblue!20}$\mathbf{36.45\% \pm 2.85\%}$ \\
\midrule
\multirow{10}[6]{*}{\rotatebox{90}{CAT}} 
 & Critical TTC & 7.40\% & \multirow{4}{*}{$\checkmark$} & $0.465 \pm 0.058$ & $1.350 \pm 0.45$ & $2.123 \pm 0.245$ & $6.726 \pm 1.064$ & $81.98\% \pm 3.12\%$ & $24.46\% \pm 3.47\%$ \\
 & Critical PET & 3.40\% & & $0.411 \pm 0.021$ & $0.950 \pm 0.25$ & $1.653 \pm 0.108$ & $4.706 \pm 0.405$ & $91.67\% \pm 3.61\%$ & $73.20\% \pm 8.84\%$ \\
 & Hard Dynamics & 3.20\% & & $0.573 \pm 0.137$ & N/A & $1.889 \pm 0.271$ & $4.985 \pm 0.706$ & $85.42\% \pm 9.55\%$ & $41.04\% \pm 7.71\%$ \\
 & Rare Cluster & 5.20\% & & $1.850 \pm 0.55$ & $0.750 \pm 0.25$ & $1.834 \pm 0.292$ & $5.271 \pm 1.155$ & $73.08\% \pm 7.69\%$ & $40.48\% \pm 4.46\%$ \\
\cmidrule(lr){2-10}
 & \multicolumn{3}{c}{\textbf{Avg}} & $\mathbf{0.825 \pm 0.192}$ & $\mathbf{1.017 \pm 0.317}$ & $\mathbf{1.875 \pm 0.229}$ & $\mathbf{5.422 \pm 0.833}$ & $\mathbf{83.04\% \pm 5.99\%}$ & $\mathbf{44.80\% \pm 6.12\%}$ \\
\cmidrule(lr){2-10}
 & Critical TTC & 7.40\% & \multirow{4}{*}{$\times$} & $0.469 \pm 0.058$ & $1.310 \pm 0.40$ & $2.168 \pm 0.206$ & $6.871 \pm 1.149$ & $81.08\% \pm 4.68\%$ & $24.82\% \pm 3.54\%$ \\
 & Critical PET & 3.40\% & & $2.850 \pm 1.10$ & $0.910 \pm 0.30$ & $1.818 \pm 0.085$ & $5.876 \pm 0.869$ & $75.00\% \pm 12.50\%$ & $87.97\% \pm 4.17\%$ \\
 & Hard Dynamics & 3.20\% & & $0.559 \pm 0.091$ & N/A & $1.960 \pm 0.299$ & $5.232 \pm 0.886$ & $81.25\% \pm 6.25\%$ & $44.03\% \pm 5.88\%$ \\
 & Rare Cluster & 5.20\% & & $1.780 \pm 0.60$ & $0.720 \pm 0.28$ & $1.969 \pm 0.182$ & $6.028 \pm 0.820$ & $64.10\% \pm 8.01\%$ & $47.96\% \pm 5.38\%$ \\
\cmidrule(lr){2-10}
 & \multicolumn{3}{c}{\textbf{Avg}} & $\mathbf{1.415 \pm 0.462}$ & $\mathbf{0.980 \pm 0.327}$ & $\mathbf{1.979 \pm 0.193}$ & $\mathbf{6.002 \pm 0.931}$ & $\mathbf{75.36\% \pm 7.86\%}$ & $\mathbf{51.20\% \pm 4.74\%}$ \\
\midrule
\multirow{10}[6]{*}{\rotatebox{90}{Heuristic}} 
 & Critical TTC & 7.40\% & \multirow{4}{*}{$\checkmark$} & $0.513 \pm 0.066$ & $1.550 \pm 0.60$ & $2.252 \pm 0.101$ & $6.849 \pm 0.698$ & $81.98\% \pm 1.56\%$ & $25.73\% \pm 1.89\%$ \\
 & Critical PET & 3.40\% & & $0.404 \pm 0.039$ & $1.100 \pm 0.50$ & $1.957 \pm 0.316$ & $6.204 \pm 0.577$ & $85.42\% \pm 3.61\%$ & $81.71\% \pm 6.07\%$ \\
 & Hard Dynamics & 3.20\% & & $0.938 \pm 0.244$ & N/A & $2.151 \pm 0.308$ & $6.353 \pm 1.384$ & $62.50\% \pm 16.54\%$ & $41.41\% \pm 5.35\%$ \\
 & Rare Cluster & 5.20\% & & $1.600 \pm 0.45$ & $0.800 \pm 0.30$ & $2.155 \pm 0.224$ & $6.761 \pm 1.107$ & $67.95\% \pm 2.22\%$ & $46.31\% \pm 1.31\%$ \\
\cmidrule(lr){2-10}
 & \multicolumn{3}{c}{\textbf{Avg}} & $\mathbf{0.864 \pm 0.200}$ & $\mathbf{1.150 \pm 0.467}$ & $\mathbf{2.129 \pm 0.237}$ & $\mathbf{6.542 \pm 0.942}$ & $\mathbf{74.46\% \pm 5.98\%}$ & $\mathbf{48.79\% \pm 3.65\%}$ \\
\cmidrule(lr){2-10}
 & Critical TTC & 7.40\% & \multirow{4}{*}{$\times$} & $0.503 \pm 0.061$ & $1.480 \pm 0.55$ & $2.229 \pm 0.054$ & $6.945 \pm 0.891$ & $79.28\% \pm 4.13\%$ & $26.07\% \pm 2.60\%$ \\
 & Critical PET & 3.40\% & & $3.550 \pm 1.50$ & $1.050 \pm 0.45$ & $2.103 \pm 0.171$ & $6.449 \pm 0.485$ & $75.00\% \pm 12.50\%$ & $88.02\% \pm 2.03\%$ \\
 & Hard Dynamics & 3.20\% & & $0.833 \pm 0.190$ & N/A & $2.098 \pm 0.358$ & $6.444 \pm 1.220$ & $64.58\% \pm 7.22\%$ & $43.71\% \pm 0.52\%$ \\
 & Rare Cluster & 5.20\% & & $1.520 \pm 0.40$ & $0.750 \pm 0.35$ & $2.366 \pm 0.288$ & $7.273 \pm 1.668$ & $56.41\% \pm 12.36\%$ & $50.61\% \pm 1.95\%$ \\
\cmidrule(lr){2-10}
 & \multicolumn{3}{c}{\textbf{Avg}} & $\mathbf{1.601 \pm 0.538}$ & $\mathbf{1.093 \pm 0.450}$ & $\mathbf{2.199 \pm 0.218}$ & $\mathbf{6.778 \pm 1.066}$ & $\mathbf{68.82\% \pm 9.05\%}$ & $\mathbf{52.10\% \pm 1.77\%}$ \\
\midrule
\multirow{10}[6]{*}{\rotatebox{90}{Replay}} 
 & Critical TTC & 7.40\% & \multirow{4}{*}{$\checkmark$} & $0.355 \pm 0.045$ & $0.950 \pm 0.35$ & $2.855 \pm 0.359$ & $8.553 \pm 2.101$ & $94.28\% \pm 2.51\%$ & $55.69\% \pm 6.27\%$ \\
 & Critical PET & 3.40\% & & $0.282 \pm 0.015$ & $0.650 \pm 0.15$ & $2.555 \pm 0.427$ & $7.857 \pm 1.259$ & $96.86\% \pm 1.56\%$ & $92.54\% \pm 1.85\%$ \\
 & Hard Dynamics & 3.20\% & & $0.453 \pm 0.124$ & N/A & $2.951 \pm 0.550$ & $8.200 \pm 1.655$ & $85.49\% \pm 9.55\%$ & $68.23\% \pm 8.46\%$ \\
 & Rare Cluster & 5.20\% & & $1.257 \pm 0.555$ & $0.450 \pm 0.15$ & $2.753 \pm 0.387$ & $8.450 \pm 1.858$ & $81.50\% \pm 3.59\%$ & $62.47\% \pm 4.50\%$ \\
\cmidrule(lr){2-10}
 & \multicolumn{3}{c}{\textbf{Avg}} & $\mathbf{0.587 \pm 0.185}$ & $\mathbf{0.683 \pm 0.217}$ & $\mathbf{2.779 \pm 0.431}$ & $\mathbf{8.265 \pm 1.718}$ & $\mathbf{89.53\% \pm 4.30\%}$ & $\mathbf{69.73\% \pm 5.27\%}$ \\
\cmidrule(lr){2-10}
 & Critical TTC & 7.40\% & \multirow{4}{*}{$\times$} & $0.389 \pm 0.050$ & $1.050 \pm 0.40$ & $2.653 \pm 0.281$ & $8.159 \pm 1.806$ & $91.51\% \pm 2.11\%$ & $52.45\% \pm 5.58\%$ \\
 & Critical PET & 3.40\% & & $1.855 \pm 1.204$ & $0.720 \pm 0.18$ & $2.481 \pm 0.357$ & $7.620 \pm 0.956$ & $88.56\% \pm 4.25\%$ & $90.18\% \pm 2.11\%$ \\
 & Hard Dynamics & 3.20\% & & $0.481 \pm 0.112$ & N/A & $2.825 \pm 0.485$ & $7.950 \pm 1.458$ & $82.15\% \pm 8.81\%$ & $65.57\% \pm 7.25\%$ \\
 & Rare Cluster & 5.20\% & & $1.188 \pm 0.483$ & $0.420 \pm 0.12$ & $2.922 \pm 0.410$ & $8.688 \pm 1.928$ & $78.29\% \pm 3.16\%$ & $65.82\% \pm 3.80\%$ \\
\cmidrule(lr){2-10}
 & \multicolumn{3}{c}{\textbf{Avg}} & $\mathbf{0.978 \pm 0.462}$ & $\mathbf{0.730 \pm 0.233}$ & $\mathbf{2.720 \pm 0.383}$ & $\mathbf{8.104 \pm 1.537}$ & $\mathbf{85.13\% \pm 4.58\%}$ & $\mathbf{68.50\% \pm 4.68\%}$ \\
\bottomrule
\end{tabular}%
}
\end{small}
\end{table}

%%%%%%%%%%%%%%%%%%%%%%%%%%%%%%%%%%%%%%%%%%%%%%%%%%%%%%%%%%%%%%%%%%%%%%%%%%%%%%%
% \clearpage
% \vspace{-30pt}
\begin{figure}[t]
  \begin{center}
    \centerline{\includegraphics[width=0.6\columnwidth]{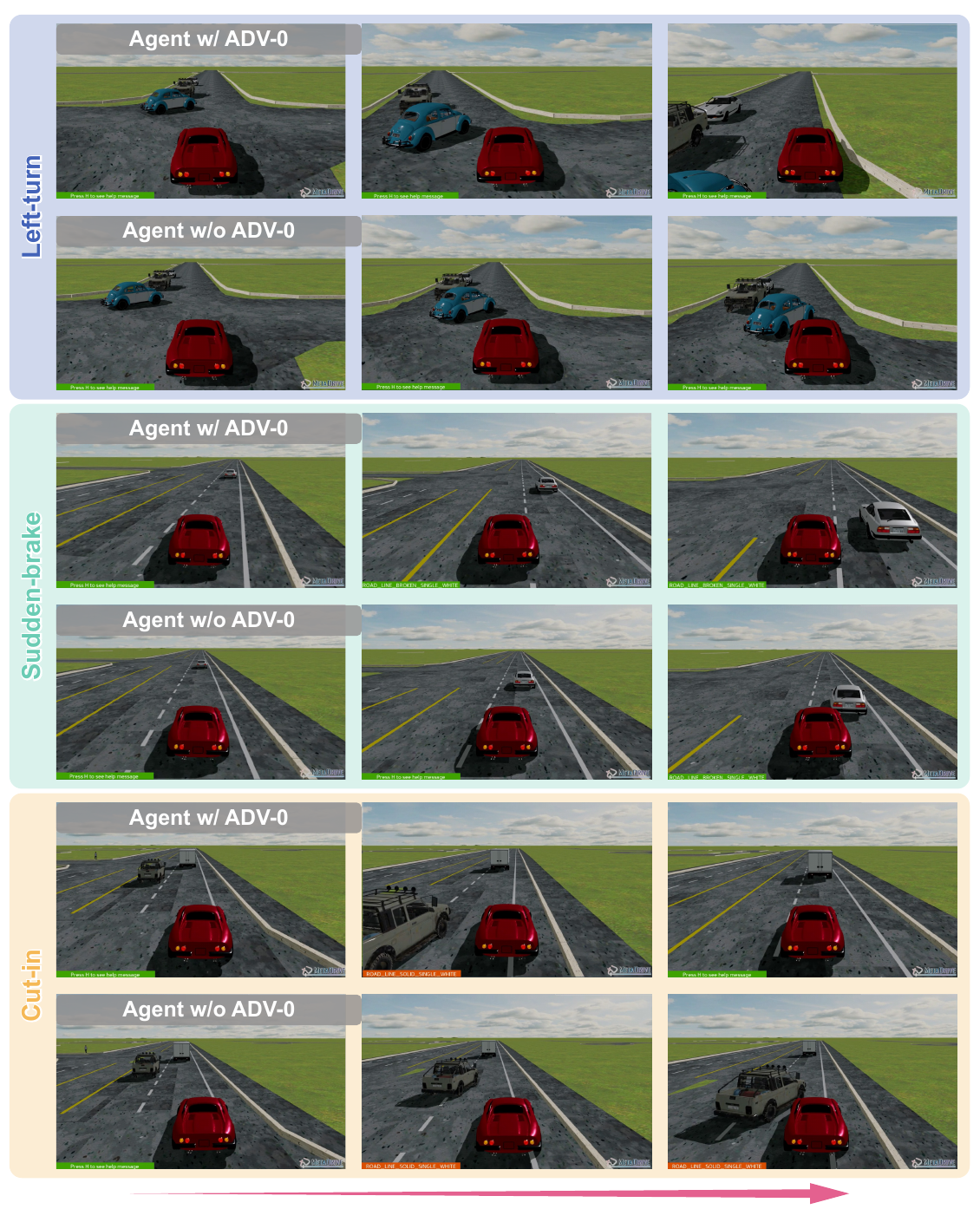}}
    % \vspace{-5pt}
    % \caption{Examples of improved safe driving ability after being trained with \texttt{ADV-0}.}
    \caption{\textbf{Qualitative comparison of improved safe driving ability after being trained with \texttt{ADV-0}.} We showcase three typical safety-critical scenarios generated by the adversary: Left-turn, Sudden-brake, and Cut-in. In each column, the bottom row (Agent w/o \texttt{ADV-0}) shows the baseline agent failing to anticipate the aggressive behavior of the background traffic, resulting in collisions. In contrast, the top row (Agent w/ \texttt{ADV-0}) demonstrates that the agent trained with our framework learns robust defensive behaviors, such as yielding at intersections or decelerating in time, thereby successfully avoiding accidents.}
    \label{fig:case-appendix}
  \end{center}
\end{figure}
% \vspace{-100pt}

% \vspace{-100pt}
\begin{figure}[b]
  \begin{center}
    \centerline{\includegraphics[width=0.9\columnwidth]{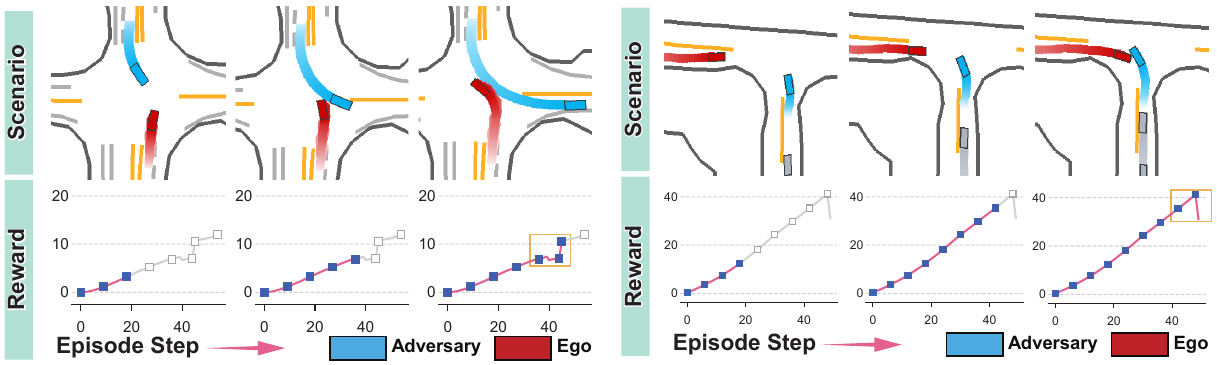}}
    % \vspace{-5pt}
    \caption{\textbf{Additional qualitative examples of reward-reduced adversarial scenarios from} \texttt{ADV-0}. In the first case (left), the adversary interrupts the straight-going ego, forcing it to deviate from the lane centerline to avoid a crash; this maneuver halts the ego's progress, causing the cumulative reward to stagnate. In the second case (right), a collision occurs near the end of the episode, triggering a sharp drop in the cumulative reward due to the safety penalty.}
    \label{fig:case-reward-appendix}
  \end{center}
\end{figure}

%%%%%%%%%%%%%%%%%%%%%%%%%%%%%%%%%%%%%%%%%%%%%%%%%%%%%%%%%%%%%%%
\clearpage
\begingroup
\setlength{\intextsep}{3pt plus 1pt minus 1pt}
\subsection{Evaluation with a Reactive Ego Policy}\label{app:reactive_ego}

We evaluate candidate scoring at 250k, 500k, 750k, and 1M training steps across five random seeds. At each checkpoint, we sample 32 candidates from the current generator for each of 20 held-out contexts and execute every candidate against the current reactive ego in the simulator, producing 12,800 oracle rollouts for individual candidates. \Cref{tab:reactive_ego}(a) compares candidate rankings and selections with these oracle returns. Pairwise accuracy uses pairs with a difference in oracle returns greater than 5; regret is the return of the selected candidate minus the lowest oracle return. \Cref{tab:reactive_ego}(b) reports the selected candidates' closed-loop outcomes. The realism penalty is computed for each trajectory as described in \cref{app:evaluation}.

\begin{table}[!htbp]

\centering
\caption{Evaluation with a reactive ego policy. Results report mean $\pm$ standard deviation over five seeds. In panel (a), Top-1 and Oracle in top-4 measure whether the candidate with the lowest oracle return is selected or included among the four candidates ranked highest, respectively. Panel (b) reports the outcomes of the selected attacks.}
\label{tab:reactive_ego}
\fontsize{8.5}{10}\selectfont
\setlength{\tabcolsep}{3.5pt}
\renewcommand{\arraystretch}{1.0}
\begin{tabular}{lccccc}
\toprule
\multicolumn{6}{l}{\textbf{(a) Ranking and selection against the reactive oracle}}\\
\midrule
Estimator & \makecell{Spearman\\$\rho\,\uparrow$} & \makecell{Pairwise\\accuracy $\uparrow$} & Top-1 $\uparrow$ & \makecell{Oracle in\\top-4 $\uparrow$} & Regret $\downarrow$\\
\midrule
Replay-only passive & $0.56\pm0.05$ & $0.69\pm0.01$ & $0.35\pm0.05$ & $0.62\pm0.04$ & $5.7\pm0.8$\\
Experience & $0.64\pm0.05$ & $0.73\pm0.03$ & $0.44\pm0.05$ & $0.70\pm0.04$ & $4.4\pm0.7$\\
RewardModel & $0.71\pm0.04$ & $0.76\pm0.03$ & $0.51\pm0.04$ & $0.77\pm0.04$ & $3.4\pm0.6$\\
\textbf{Current proxy} & $0.77\pm0.03$ & $0.80\pm0.03$ & $0.61\pm0.05$ & $0.84\pm0.03$ & $2.3\pm0.5$\\
Reactive oracle rollout & $1.00$ & $1.00$ & $1.00$ & $1.00$ & $0.0$\\
\bottomrule
\end{tabular}

\smallskip
\begin{tabular}{lcccc}
\toprule
\multicolumn{5}{l}{\textbf{(b) Closed-loop outcomes of selected attacks}}\\
\midrule
Selection & \makecell{True\\return $\downarrow$} & \makecell{Reactive\\crash $\uparrow$} & \makecell{Non-collision\\failure $\uparrow$} & \makecell{Realism\\penalty $\downarrow$}\\
\midrule
Random prior & $43.2\pm1.4$ & $0.31\pm0.04$ & $0.07\pm0.01$ & $9.8\pm0.4$\\
Collision-only & $34.6\pm1.8$ & $0.57\pm0.03$ & $0.08\pm0.06$ & $12.0\pm0.6$\\
\textbf{Proxy top-1} & $29.9\pm1.7$ & $0.51\pm0.03$ & $0.21\pm0.02$ & $11.4\pm0.5$\\
Reactive oracle top-1 & $27.6\pm1.6$ & $0.54\pm0.02$ & $0.23\pm0.03$ & $11.5\pm0.6$\\
\bottomrule
\end{tabular}
\end{table}

The proxy achieves 0.80 pairwise accuracy and 0.84 recall of the oracle choice among the top four candidates. The selected attacks yield an ego return of 29.9 versus 27.6 for oracle selection.

\subsection{Comparisons of Preference Learning Methods}\label{app:preference_baselines}

We compare five methods using the TD3 setting and five paired random seeds. Static Return-DPO constructs preference pairs once using the initial defender, receives the same generator optimization budget as \texttt{ADV-0}, and keeps the resulting generator fixed during 1M steps of ego training. AILP-Traj adapts adversarial imitation learning with preferences \citep{taranovic2023adversarial} to traffic trajectory generation: context $X$ is the state, a complete traffic future $Y$ is treated as a single action, WOMD futures provide demonstrations, and samples from the current generator provide imperfect demonstrations. The current defender's returns determine the preference ordering, and the discriminator combines demonstration matching with preference learning. The generator is updated through policy gradient over complete trajectories under the same KL constraint to $\mathcal{G}_{\mathrm{ref}}$.

\begin{table}[!htbp]

\centering
\caption{Comparisons of preference learning methods in the TD3 setting. Results report mean performance over five random seeds.}
\label{tab:preference_baselines}
\small
\setlength{\tabcolsep}{4pt}
\renewcommand{\arraystretch}{1.05}
\begin{tabular}{llcccc}
\toprule
Method & Preference mechanism & RC $\uparrow$ & Crash $\downarrow$ & Reward $\uparrow$ & Cost $\downarrow$\\
\midrule
Replay & None & 0.672 & 0.344 & 42.66 & 0.600\\
\texttt{ADV-0} w/o IPL & Return sampling only & 0.690 & 0.314 & 44.35 & 0.590\\
Static Return-DPO & Fixed offline pairs & 0.718 & 0.378 & 43.14 & 0.597\\
AILP-Traj & Demo + preference AIL & 0.707 & 0.291 & 45.47 & 0.571\\
\texttt{ADV-0} w/ IPL & Online pairwise & 0.743 & 0.255 & 48.58 & 0.520\\
\bottomrule
\end{tabular}
\end{table}

RC denotes route completion. \texttt{ADV-0} w/ IPL performs best on all four metrics (\cref{tab:preference_baselines}).

\clearpage
\endgroup
\section{Detailed Experimental Setups}\label{appendix:setups}
In this section, we provide a comprehensive description of the experimental environment, datasets, baseline methods, model implementations, and hyperparameter configurations used in our study. Our experimental design follows the protocols established in prior works \citep{zhang2023cat,nie2025steerable,stoler2025seal} to ensure a fair and rigorous comparison.

\subsection{Environment and Dataset}\label{appendix: }

\paragraph{Waymo Open Motion Dataset (WOMD).}
We utilize the WOMD as the source of real-world traffic scenarios. 
WOMD is a large-scale dataset containing diverse and complex urban driving environments captured in various conditions. Each scenario spans 9 seconds and is sampled at 10 Hz, capturing complex interactions between vehicles, pedestrians, and cyclists. Following the standard practices in safety-critical scenario generation, we filter and select a subset of 500 scenarios that involve interactive and complex behaviors.

\paragraph{MetaDrive simulator.}
All experiments are conducted within the MetaDrive simulator \citep{li2022metadrive}, a lightweight and efficient platform that supports importing real-world data for closed-loop simulation.
MetaDrive constructs the static map environment and replays the background traffic trajectories based on the WOMD logs. 
The simulation runs at a frequency of 10 Hz. The observation space consists of the ego vehicle's kinematic state (velocity, steering, heading), navigation information (relative distance and direction to references), and surrounding information (surrounding traffic, road boundaries, and road lines) encoded as a vector by a simulated 2D LiDAR with 30 lasers and a 50-meter detection range.
The action space consists of low-level continuous control signals including steering, brake, and throttle.
The long-tail evaluation therefore uses real-world logged scenarios in a closed-loop simulation environment, rather than on-road adversarial deployment.
\paragraph{Compute resources.}
All experiments are conducted on a computing platform equipped with two NVIDIA RTX A6000 GPUs. The computational overhead of the experiments is reported in the corresponding experimental analysis.

\paragraph{Reward definition.}
The ground-truth reward function for the ego agent is designed to balance safety and progression. It is composed of a dense driving reward and sparse terminal penalties. Formally, the reward function $R_t$ at step $t$ is defined as:
\begin{equation}\label{eq:metadrive}
    R_t = R_{\text{driving}} + 0.1v_t + R_{\text{success}} - P_{\text{collision}} - P_{\text{offroad}}
\end{equation}
where $R_{\text{driving}} = d_{t} - d_{t-1}$ represents the longitudinal progress along the reference route, incentivizing the agent to move toward the destination. $R_{\text{success}} = +10$ is a sparse reward granted upon reaching the destination. Safety penalties are applied for terminal failures: $P_{\text{collision}} = 10$ for collisions with vehicles or objects, and $P_{\text{offroad}} = 10$ for driving out of the drivable road boundaries. Additionally, a small speed reward of $0.1 \times v_t$ is added to encourage movement. The episode terminates if the agent succeeds, crashes, or leaves the drivable area. For Lagrangian-based algorithms (e.g., PPO-Lag), we define a binary cost function $C_t$ which equals 1 if a safety violation occurs and 0 otherwise.

\subsection{Detailed Implementation}\label{app:implementation}

The overall closed-loop training procedure is summarized in \cref{alg:adv0}.

\begin{algorithm}[t]
\begin{small}
   \caption{Closed-Loop Min-Max Policy Optimization}
   \label{alg:adv0}
\begin{algorithmic}[1]
    \STATE {\bfseries Input:} Initial ego policy $\pi_\theta$, Pretrained adversary prior $\mathcal{G}_{\text{ref}}$.
    % Ego history buffer $\mathcal{H}_{\text{ego}}$, RL \texttt{Algorithm}.
   \STATE {\bfseries Hyperparameters:} Temperature $\tau$, Reward filters $\delta,\xi$, Learning rates $\alpha, \eta$, Frequency $N_{\text{freq}}$, IPL batch size $M$.
   \STATE {\bfseries Initialize:} Adversary $\mathcal{G}_\psi \leftarrow \mathcal{G}_{\text{ref}}$, Ego history buffer $\mathcal{H}_{\text{ego}} \leftarrow \emptyset$, Replay buffer $\mathcal{D}$ (Off-policy) or Batch buffer $\mathcal{B}$ (On-policy).

   \FOR{timestep $t = 1$ {\bfseries to} $T_{\text{max}}$}
   
   \STATE \textcolor{gray}{\emph{// --- Phase 1: Adversary Update (Inner Loop) ---}}
   \IF{$t \mod N_{\text{freq}} == 0$}
   \FOR{iteration $k = 1$ {\bfseries to} $K_{\text{IPL}}$}
       \STATE Sample context batch $\{X_m\}_{m=1}^M$.
       % \STATE Generate $\{Y^\text{Adv}_{m,k}\} \sim \mathcal{G}_\psi(\cdot|X_m)$ $\forall{X_m}$.
       \STATE Generate candidates $\{Y^\text{Adv}_{m,j}\}_{j=1}^K \sim \mathcal{G}_\psi(\cdot|X_m), \forall X_m$.
       \STATE Calculate $\{\hat{J}(Y^\text{Adv}_{m,j}, \pi_\theta)\}_{j=1}^K$ using $\mathcal{H}_{\text{ego}}[X_m]$ (Eq.~\ref{eq:proxy_reward}).
       \STATE Construct preference pairs $\mathcal{D}_{\text{pref}}$ based on all $\hat{J}$.
       \STATE Update $\psi$ via IPL on $\mathcal{D}_{\text{pref}}$ (Eq.~\ref{eq:ipl_loss}).
   \ENDFOR
   \ENDIF
   
   \STATE \textcolor{gray}{\emph{// --- Phase 2: Ego Update (Outer Loop) ---}}
   \STATE Sample new scenario context $X$.
   
   % \STATE \textbf{Adversarial Sampling:} 
   \STATE Generate candidates $\{Y_k\} \sim \mathcal{G}_\psi(\cdot|X)$.
   \STATE Select $Y^{\text{adv}}$ via softmax sampling (Eq.~\ref{eq:sampling}).
   
   % \STATE \textbf{Interaction:} 
   \STATE Rollout $\pi_\theta$ in environment $\mathcal{P}_{\psi}$ to get $Y^\text{Ego}$.
   \STATE Update history: $\mathcal{H}_{\text{ego}}[X] \leftarrow \text{FIFO}(\mathcal{H}_{\text{ego}} \cup \{Y^\text{Ego}\})$.
   
   % \STATE \textbf{Policy Optimization:}
   \IF{\texttt{Algorithm} is On-Policy (e.g., PPO)}
       \STATE Store $Y^\text{Ego}$ in Batch Buffer $\mathcal{B}$.
       \IF{$\mathcal{B}$ is full}
           \STATE Update $\theta$ via Eq.~\ref{eq:min_max} several steps on $\mathcal{B}$, then clear $\mathcal{B}$.
       \ENDIF
   \ELSIF{\texttt{Algorithm} is Off-Policy (e.g., SAC)}
       \STATE Store transitions from $Y^\text{Ego}$ into Replay Buffer $\mathcal{D}$.
       \STATE Sample mini-batch from $\mathcal{D}$ and update $\theta$ one step (Eq.~\ref{eq:min_max}).
   \ENDIF
   
   \ENDFOR
\end{algorithmic}
\end{small}
\end{algorithm}

\subsubsection{Proxy reward estimator}\label{subsubsec:proxy_reward}

To efficiently estimate the ego's expected return $\hat{J}(Y^\text{Adv}, \pi_\theta)$ without executing computationally expensive closed-loop simulations during the inner loop, we implement a vectorized rule-based proxy evaluator. This module calculates the geometric interaction between a candidate adversarial trajectory $Y^\text{Adv}$ and a sampled ego response $Y^\text{Ego}$ from the cached history buffer. 
Let $Y^\text{Ego} = \{p^\text{ego}_t, \psi^\text{ego}_t\}_{t=1}^T$ and $Y^\text{Adv} = \{p^\text{adv}_t, \psi^\text{adv}_t\}_{t=1}^T$ denote the sequences of position and yaw for the ego and adversary, respectively. The proxy reward is calculated by mimicking the scheme in MetaDrive (Eq.~\ref{eq:metadrive}):

\paragraph{Collision detection.} We approximate the vehicle geometry using Oriented Bounding Boxes (OBB) defined by the center position, yaw, length $l$, and width $w$. For each timestep $t$, we compute the four corner coordinates of both vehicles. We employ the Separating Axis Theorem (SAT) to determine if the two OBBs overlap. A collision penalty $r_{\text{crash}}$ is applied if an overlap is detected at any timestep, and the evaluation terminates early.

\paragraph{Route progress and termination.} We map the ego's position to the Frenet frame of the reference lane to obtain the longitudinal coordinate $s_t$ and lateral deviation $d_t$. The evaluation terminates if:
(1) \textit{Success:} the longitudinal progress $s_t / L_{\text{total}} > 0.95$, granting a reward $r_{\text{success}}$.
(2) \textit{Off-road:} the lateral deviation $|d_t| > 10.0$ meters, applying a penalty $r_{\text{offroad}}$.

\paragraph{Dense reward.} If no termination condition is met, a dense driving reward is accumulated based on the incremental longitudinal progress $\Delta s_t = s_t - s_{t-1}$. The total estimated return is the sum of step-wise rewards and terminal bonuses/penalties:
\begin{equation}
    \mathcal{R}_{\text{proxy}} = \sum_{t=1}^{T_{\text{end}}} \left( \lambda_{\text{drive}} \cdot \Delta s_t \right) + \mathbb{I}_{\text{crash}} \cdot r_{\text{crash}} + \mathbb{I}_{\text{success}} \cdot r_{\text{success}} + \mathbb{I}_{\text{offroad}} \cdot r_{\text{offroad}}.
\end{equation}
Once a terminal condition is met, the summation stops, and the cumulative value is returned. This geometric calculation is fully vectorized across the batch of candidate trajectories, allowing for rapid evaluation of potential attacks.
In our experiments, we set $r_{\text{success}}=10$, $r_{\text{crash}}=-10$, $r_{\text{offroad}}=-10$, and $\lambda_{\text{drive}}=1.0$.

\paragraph{Baseline schemes.}
For the ablation study in \cref{fig:reward_model}, we compare against:
(1) \textit{Experience:} We query the ego's replay buffer. If the exact scenario context exists, we use the recorded return; otherwise, we retrieve the return of the nearest neighbor scenario based on trajectory similarity.
(2) \textit{RewardModel:} We train a separate learnable reward model $\mathcal{M}(X, Y^\text{Ego}, Y^\text{Adv})$ via supervised regression on the historical interaction dataset to predict the scalar return.
(3) \textit{GTReward:} We execute the physics engine in a parallel process to roll out the interaction between $\pi_\theta$ and the specific $Y^\text{Adv}$ to obtain the exact return.
While our results suggest that rule-based approach is effective, future work could explore integrating value function approximations (e.g., Q-networks from Actor-Critic architecture) to estimate returns for more complex reward functions.

%%%%%%%%%%%%%%%
\paragraph{Context-aware buffer.} 
As noted in \cref{subsec:sampling}, the validity of the rule-based proxy estimator relies on the geometric consistency between the adversarial candidate $Y^{\text{Adv}}$ and the ego response $Y^{\text{Ego}}$. Since different scenarios possess distinct map topologies and coordinate systems, using a global history buffer would lead to physically meaningless calculations. 
To address this, we implement the ego history buffer $\mathcal{H}_{\text{ego}}$ as a context-indexed dictionary, mapping each unique scenario ID to a First-In-First-Out (FIFO) queue of its historical ego trajectories: $\mathcal{H}_{\text{ego}}(X) = \{Y^{\text{Ego}}_{i}|X\}_{i=1}^N$. This ensures that the geometric interaction $R_{\text{proxy}}(Y^{\text{Ego}}, Y^{\text{Adv}})$ is always computed between trajectories residing in the same spatial environment.

A key challenge raised by the dynamic training process is handling newly sampled contexts $X_{\text{new}}$ that have not yet been interacted with by the current ego policy, resulting in an empty buffer $\mathcal{H}_{\text{ego}}(X_{\text{new}}) = \emptyset$. To address this, we employ a warm-up rollout before the inner-loop adversarial update begins for a sampled batch of contexts:
(1) We check the buffer status for each context $X$ in the batch.
(2) If $\mathcal{H}_{\text{ego}}(X_{\text{new}}) = \emptyset$, we execute a single inference rollout of the current ego policy $\pi_\theta$ in the simulator. Importantly, this rollout is conducted against the non-adversarial replay log.
(3) The resulting trajectory $Y^{\text{Ego}}_{\text{ref}}$ represents the ego's baseline behavior in the absence of attacks and is added to $\mathcal{H}_{\text{ego}}(X)$.
This mechanism ensures that the proxy estimator always has a valid reference to approximate the ego's vulnerability, even for unseen scenarios. 
In addition, the proxy estimator always calculates geometric interactions between trajectories sharing the same spatial topology.
As the training progresses, new interactions under adversarial perturbations are generated and added to the buffer, gradually shifting the distribution in $\mathcal{H}_{\text{ego}}(X)$ from naturalistic responses to defensive responses against attacks.

\begingroup
\paragraph{Computational complexity.}
Let $M$ denote the number of contexts in an update, $K$ the candidates per context, $N$ the cached ego responses, and $T$ the trajectory horizon. Vectorized proxy scoring has time complexity $O(MKNT)$. The FIFO history requires $O(NT)$ memory per active context, and generated candidates require $O(MKT)$ temporary memory. We use $K=32$ and $N=5$, which remain fixed throughout training.

\par\endgroup

\subsubsection{Evaluation and benchmark}\label{app:evaluation}

\paragraph{Adversarial scenario generation.}
We follow the safety-critical scenario generation evaluation protocol used in \citet{zhang2023cat}.
For each scenario, the evaluation presented in \cref{tab:scenario_gen} and \cref{tab:scenario_gen_adv0} follows a two-stage process: (1) The environment is reset with a fixed seed and the ego agent first interacts with the log-replayed environment to generate a reference trajectory; (2) The adversarial generator conditions on the context and the ego's reference trajectory to generate adversarial trajectories for one selected adversary (which is marked as \texttt{Object-of-Interests} in WOMD). This trajectory is set as a fixed plan for the adversary traffic in the simulator. 
For \texttt{ADV-0}, we employ a worst-case sampling strategy for evaluation, setting the sampling temperature $\tau \to 0$ (Eq.~\ref{eq:sampling}) to select the trajectory with the highest estimated adversarial utility from $K=32$ candidates. We test against three kinds of driving policies: a \textit{Replay} policy that follows ground truth logs from WOMD, an \textit{IDM} policy representing reactive rule-based drivers, and \textit{RL agents} trained via standard PPO on replay logs. These systems under test execute their control loop in the environment.
The primary metrics are Collision Rate (CR), defined as the percentage of episodes where the ego collides with the adversary, and Ego's Return (ER), the cumulative reward achieved by the ego. 
For the baselines adversarial generators, we utilize their official implementations but adapt them for our evaluation environment to ensure a fair comparison. All of them follow the same evaluation procedure. For the realism penalty metric in \cref{fig:real_pen}, we adopt the trajectory-level measure from \citet{nie2025steerable}, which discourages trajectories that
are physically implausible or exhibit unnatural driving behavior.

\paragraph{Performance validation of learned agents.}
To rigorously evaluate the generalizability of the learned AD policies and compare different adversarial learning methods, we established a cross-validation protocol where each trained agent is tested against multiple distinct scenario generators. 
For the evaluation, we utilized a held-out test set of 100 WOMD scenarios that were not seen during training. 
We compared five types of agents: those adversarially trained by \texttt{ADV-0} (with and without IPL), \texttt{CAT}, \texttt{Heuristic}, and a baseline trained solely on \texttt{Replay} data. All these agents are trained using 400 WOMD scenarios.
Each agent was evaluated in five distinct environments with the held-out test set: \texttt{Replay}, \texttt{ADV-0}, \texttt{CAT}, \texttt{SAGE}, and \texttt{Heuristic}.
Note that since \texttt{SAGE} introduces an additional scenario difficulty training curriculum, we exclude it for training methods.
For each evaluation run, we used agents saved at the best validation checkpoints and recorded four metrics: Route Completion (RC), Crash Rate (percentage of episodes ending in collision), Reward (cumulative environmental reward), and Cost (safety violation penalty). 
Each method uses five independent random seeds. Unless specified otherwise, entries for an individual method, RL algorithm, and evaluation environment report the mean and standard deviation across five seeds. \Cref{tab:agent_average_all} averages the six algorithm means and the standard deviations across seeds for each algorithm. The Avg. row of \cref{tab:scenario_gen_adv0} applies the same aggregation across the six fine-tuned generator variants, and \cref{tab:scenario_gen} reports the corresponding \texttt{ADV-0} macro means. In Tables~\ref{tab:agent_average_all} and \ref{tab:agent_grpo}--\ref{tab:agent_td3}, Avg. rows average the means and standard deviations across the five evaluation environments.
For the cross-validation in \cref{tab:cross_val_adv0}, we compare the adversary enhanced with IPL against the pretrained prior equipped with energy-based sampling (Eq.~\ref{eq:sampling}). Results are macro means across the six RL algorithms $\pm$ the mean of their standard deviations over five seeds.

\paragraph{Evaluation in long-tailed scenarios.} Generated scenarios from adversarial models can be biased and cause a sim-to-real gap for AD policies.
To ensure an unbiased evaluation of policy robustness, we construct a curated evaluation set by mining additional 500 held-out scene segments from the WOMD. These scenarios are not shown in the training set.
We employ strict physical thresholds to identify and categorize rare, safety-critical events: (1) \textit{Critical TTC}: scenarios containing frames where the TTC between the ego and any object drops below $0.4s$; (2) \textit{Critical PET}: scenarios with a Post-Encroachment Time (PET) lower than 1.0s, indicating high-risk intersection crossings; (3) \textit{Hard Dynamics}: scenarios involving aggressive behaviors, defined by longitudinal deceleration exceeding $-4.0 m/s^2$ or absolute jerk exceeding $4.0 m/s^3$; and (4) \textit{Rare Cluster}: scenarios belonging to the two lowest-density clusters identified via K-Means clustering ($k=10$) on trajectory features for all interacting objects, including curvature, velocity profiles, and displacement. 
During evaluation, we reproduce these scenarios in the simulator and utilize two traffic modes: a non-reactive mode where background vehicles follow logged trajectories, and a reactive mode where vehicles are controlled by IDM and MOBIL policies to simulate human-like interaction with the agent. 
In addition to safety margin and stability metrics, we also report \textit{Near-Miss Rate}, defined as the percentage of episodes where TTC $< 1.0s$ or distance $< 0.5m$ without collision, and \textit{RDP Violation}, which measures the frequency of violating the Responsibility-Sensitive Safety (RSS) Danger Priority safety distances.

%%%%%%%%%%%%%%%%%%%%%%%%%%%%%%%%%%%%%%%%%%%%%%

\subsubsection{Adaptability to different RL algorithms}
\label{subsubsec:adaptability}

This section elaborates on the implementation details regarding the integration of \texttt{ADV-0} with various RL algorithms, as mentioned in Section~\ref{subsec:framework}. We specifically address the synchronization mechanisms and the specialized credit assignment strategy developed for critic-free architectures to ensure stable convergence in safety-critical tasks.

\paragraph{RL algorithms.}
Our framework is designed to be algorithm-agnostic, treating the adversarial generator as a dynamic component of the environment dynamics $\mathcal{P}_\psi$. Consequently, the ego agent perceives the generated adversarial trajectories simply as state transitions, allowing \texttt{ADV-0} to support both on-policy and off-policy algorithms. In our experiments, we instantiate the ego policy using six distinct algorithms, covering both on-policy and off-policy paradigms, as well as Lagrangian variants for constrained optimization: GRPO, PPO, SAC, TD3, PPO-Lag, and SAC-Lag. The primary distinction in implementation lies in the data collection and update scheduling.
For \textit{on-policy} methods (e.g., PPO, GRPO), the training alternates strictly between the adversary and the defender. In the outer loop, we fix the adversary $\psi$ and collect a batch of trajectories $\mathcal{B}$ using the current ego policy $\pi_\theta$. The policy is updated using this batch, which is then discarded to ensure that the policy gradient is estimated using the data distribution induced by the current adversary.
Conversely, for \textit{off-policy} methods (e.g., SAC, TD3), we maintain a replay buffer $\mathcal{D}$. While the adversary $\psi$ evolves periodically, older transitions in $\mathcal{D}$ technically become off-dynamics data. To mitigate the impact of non-stationarity, we employ a sliding window approach where the replay buffer has a limited capacity, ensuring that the value function estimation relies predominantly on recent interactions with the current or near-current adversary. The adversary update frequency $N_{\text{freq}}$ is tuned such that the off-policy agent has sufficient steps to adapt to the current risk distribution before the adversary shifts its strategy. This historical diversity acts as a natural form of domain randomization, preventing the ego from overfitting to specific attack patterns.

\paragraph{Credit assignment in critic-free methods.}
While critic-based methods (e.g., PPO) rely on a value function $V(s)$ to reduce variance, critic-free methods like GRPO typically utilize sequence-level outcome supervision. In the training of LLMs, it is common to assign the final reward of a completed sequence to all tokens. However, we identify that \textit{directly applying this sequence-level supervision to AD tasks leads to severe credit assignment issues, particularly when addressing the long-tail distribution}. 
Our experiments revealed that applying standard outcome supervision leads to training instability and policy collapse after certain steps. This occurs because safety-critical failures (e.g., collisions) often happen at the very end ($t=T$) of a long-horizon episode. Assigning a low return to the entire trajectory incorrectly penalizes the correct driving behaviors exhibited in the early stages of the episode ($t \ll T$), resulting in high-variance gradients that disrupt the fine-tuning phase.
On the other hand, implementing standard process supervision (e.g., via Monte Carlo value estimation across rollouts \citep{guo2025deepseek}) would require the physical simulator to support resetting to arbitrary intermediate states to perform multiple forward rollouts from every timestep. In complex high-fidelity simulators, this requirement introduces substantial engineering complexities regarding state serialization and incurs prohibitive computational overhead.

To resolve this, we propose a \textit{step-aligned group advantage estimator} that provides dense step-level supervision without requiring a learned critic or simulator state resets.
Specifically, for each update step, we sample a scenario context $X$ and generate a group of $G$ independent episodes ($G=6$ in our experiments) starting from the exact same initial state (achieved via scenario seeding) but with different stochastic action realizations.
Let $\tau_i = \{(s_t, a_t, r_t)\}_{t=0}^{T_i}$ denote the $i$-th transition in the group. We implement the following modifications to the advantage estimation:
\begin{enumerate}
    \item \textbf{Calculating returns-to-go:} Instead of the total episode return, we calculate the discounted return-to-go $R_{t,i} = \sum_{k=t}^{T_i}\gamma^{k-t}r_{k,i}$ for each step $t$ in transition $i$. This ensures that an action is only evaluated based on its consequences.
    \item \textbf{Step-aligned group normalization:} We compute the advantage $A_{t,i}$ by normalizing $R_{t,i}$ against the returns of other trajectories in the same group at the same time step $t$, which uses the peer group as a dynamic baseline:
    \begin{equation}\label{eq:grpo}
        A_{t,i} = \frac{R_{t,i} - \mu_t}{\sigma_t + \epsilon}, \quad \text{where } \mu_t = \frac{1}{G}\sum_{j=1}^G R_{t,j}, \quad \sigma_t = \sqrt{\frac{1}{G}\sum_{j=1}^G (R_{t,j} - \mu_t)^2}.
    \end{equation}
    \item \textbf{Baseline padding:} Since episodes have varying lengths (e.g., due to early termination from crashes), the group size at step $t$ could decrease. To maintain a low-variance baseline, we apply zero-padding to terminated trajectories. If trajectory $j$ ends at $T_j < t$, we set $R_{t,j} = 0$. This ensures the baseline $\mu_t$ is always computed over the full group size $G$, correctly reflecting that survival yields higher future returns than termination.
    
    \item \textbf{Advantage clipping:} To prevent single outliers such as rare failures from dominating the gradient and destabilizing the policy, we clip the calculated advantages as: $\hat{A}_{t,i} = \text{clip}(A_{t,i}, -C, C)$, where $C=5.0$.
\end{enumerate}
Then Eq.~\ref{eq:grpo} is integrated into the standard GRPO loss and optimized via mini-batch gradient descent.
This modification provides an efficient and low-variance gradient signal that correctly attributes failure to specific actions leading up to it, without the instability observed in outcome supervision or expensive state resetting.

%%%%%%%%%%%%%%%%%%%%%%%%%%%%%%%%%%%%%%%%%%%%%%

\subsubsection{Application to learning-based motion planning models}
\label{app:planner_implementation}

In this section, we detail the implementation of applying \texttt{ADV-0} to fine-tune trajectory planning models. Unlike standard end-to-end RL policies that output control actions directly, motion planners output future trajectories which are then executed by a low-level controller. This introduces challenges regarding re-planning and reward attribution. To address this, we decouple the planning evaluation from environmental execution, inspired by \citet{chen2025rift}. Note that this framework can be applied to any RL algorithm, and we adopt GRPO as a demonstration.

\paragraph{Model architectures.}
To demonstrate the versatility of \texttt{ADV-0}, we apply it to two representative categories of state-of-the-art motion planning models:

\begin{enumerate}
    \item \textbf{Autoregressive generation (SMART~\citep{wu2024smart}):} 
    SMART formulates motion generation as a next-token prediction task, analogous to LLMs. It discretizes vectorized map data and continuous agent trajectories into sequence tokens, utilizing a decoder-only Transformer to model spatial-temporal dependencies. By autoregressively predicting the next motion token, the model effectively captures complex multi-agent interactions and has demonstrated potential for motion generation and planning.

    \item \textbf{Multimodal Scoring (PlanTF~\citep{cheng2024rethinking}):} 
    PlanTF is a Transformer-based imitation learning planner designed to address the shortcut learning phenomenon often observed in history-dependent models. It encodes the ego vehicle's current kinematic states, map polylines, and surrounding agents using a specialized attention-based state dropout encoder to mitigate compounding errors. This architecture allows the planner to generate robust closed-loop trajectories by focusing on the causal factors of the current scene rather than overfitting to historical observations. PlanTF has achieved state-of-the-art in several popular motion planning benchmarks.
\end{enumerate}

Following the standard pretrain-then-finetune practice, we first apply imitation learning to train supervised policies by behavior cloning.
To deploy these models in closed-loop simulation, we implement a wrapper policy that executes a re-planning cycle every $N=10$ simulation steps (1.0s). At each cycle, the planner receives the current observation and generates a trajectory. A PID controller then tracks this trajectory to produce steering and acceleration commands for the underlying physics engine. We employ an advanced PID controller with a dynamic lookahead distance to ensure smooth tracking of the planned path.

\paragraph{Fine-tuning planners via \texttt{ADV-0.}}
Directly applying standard RL algorithms to fine-tune trajectory planners is inefficient due to the sparsity of rewards relative to the high-dimensional output space. In addition, the low-level controller executed during a re-planning horizon increases the difficulty of reward credit assignment.
To address this, we decouple planning evaluation from execution and implement a state-wise reward model (SWRM) to provide dense supervision directly on the planned trajectories, following \citet{chen2025rift}.
We employ GRPO to fine-tune the planners. The process at each re-planning step $t$ is as follows:

\begin{enumerate}
    \item \textbf{Generation:} The planner generates a group of $K$ candidate trajectories $\{T_1, \dots, T_K\}$ conditioned on the current state $s_t$. For PlanTF, these are the multimodal outputs; for SMART, we sample $K$ sequences via temperature sampling.
    \item \textbf{Evaluation:} Instead of rolling out $K$ trajectories in the simulator, we evaluate them immediately using SWRM. SWRM calculates an instant reward $r_k$ for each $T_k$ based on geometric and kinematic properties over a horizon $H=2.0s$:
    \begin{equation}
        r(T_k) = R_{\mathrm{prog}}(T_k) - P_{\mathrm{coll}}(T_k) - P_{\mathrm{off}}(T_k) - P_{\mathrm{comf}}(T_k) - P_{\mathrm{eff}}(T_k),
    \end{equation}
    where the terms are the weighted reward and penalties defined in \cref{tab:planner_breakdown}, accumulated over the 2.0s horizon with $\gamma=0.99$. The comfort term includes acceleration and jerk penalties, and the speed term penalizes deviation from 10m/s.
    \item \textbf{Optimization:} We compute the advantage for each candidate as $A_k = (r_k - \bar{r}) / \sigma_r$, where $\bar{r}$ and $\sigma_r$ are the mean and standard deviation of rewards within the group. The policy is updated to maximize the likelihood of high-advantage trajectories using the GRPO objective:
    \begin{equation}
    \mathcal{L}_{\text{GRPO}} = -\frac{1}{K} \sum_{k=1}^K \min \left( \frac{\pi_\theta(T_k|s_t)}{\pi_{\text{old}}(T_k|s_t)} A_k, \text{clip}\left(\frac{\pi_\theta(T_k|s_t)}{\pi_{\text{old}}(T_k|s_t)}, 1-\epsilon, 1+\epsilon\right) A_k \right).
\end{equation}
    For PlanTF, we fine-tune the trajectory scoring head and decoder; for SMART, we fine-tune the token prediction logits.
    \item \textbf{Execution:} The trajectory with the highest SWRM score is selected for execution by the PID controller to advance the environment to the next re-planning step.
\end{enumerate}

This approach allows the planner to learn from the adversarial scenarios generated by \texttt{ADV-0} by explicitly penalizing trajectories that the SWRM identifies as risky, without requiring dense environmental feedback.

\paragraph{Adversarial interaction.}
The \texttt{ADV-0} adversary operates in the inner loop as described in the main paper. The adversary generator $\mathcal{G}_\psi$ creates challenging scenarios based on the planner's executed history, and is further updated via IPL. The planner is then fine-tuned via GRPO to propose safer trajectories in response to these generated risks.

\subsection{Baselines}

We compare \texttt{ADV-0} against a comprehensive set of baselines, categorized into adversarial scenario generators (methods that create the environment) and closed-loop adversarial training frameworks (methods that train the ego policy).

\vspace{-5pt}

\paragraph{Backbone model.}
Consistent with prior works \citep{zhang2023cat,nie2025steerable,stoler2025seal}, we employ DenseTNT \citep{gu2021densetnt} as the backbone motion prediction model for the adversarial generator. DenseTNT is an anchor-free, goal-based motion forecasting model capable of generating multimodal distributions of future trajectories, which is known for its high performance on the WOMD benchmark. We initialize the generator using the publicly available pretrained checkpoint, ensuring a fair comparison of the generation capabilities.

\vspace{-5pt}
\paragraph{Adversarial generators.}
To evaluate the effectiveness of the generated scenarios, we compare \texttt{ADV-0} against a comprehensive set of adversarial generation methods, covering optimization-based, learning-based, and sampling-based paradigms:
\begin{itemize}
    \item \textbf{Heuristic}~\citep{zhang2023cat}: A hand-crafted baseline that modifies the trajectory of the background vehicle to intercept the ego vehicle's path using Bezier curve fitting. It heuristically generates aggressive cut-ins or emergency braking maneuvers based on the ego vehicle's position. This serves as an oracle method representing worst-case physical attacks.
    
    \item \textbf{CAT}~\citep{zhang2023cat}: A state-of-the-art sampling-based approach that generates adversarial trajectories by resampling from the DenseTNT traffic prior. It selects trajectories that maximize the posterior probability of collision with the ego vehicle's planned path.
    
    \item \textbf{KING}~\citep{hanselmann2022king}: A gradient-based approach that perturbs adversarial trajectories by backpropagating through a differentiable kinematic bicycle model to minimize the distance to the ego vehicle.
    
    \item \textbf{AdvTrajOpt}~\citep{zhang2022adversarial}: An optimization-based approach that formulates adversarial generation as a trajectory optimization problem. It employs Projected Gradient Descent (PGD) to iteratively modify trajectory waypoints to induce collisions.
    
    \item \textbf{SEAL}~\citep{stoler2025seal}: A skill-enabled adversary that combines a learned objective function with a reactive policy. It utilizes a scoring network to predict collision criticality and ego behavior deviation.
    
    \item \textbf{GOOSE}~\citep{ransiek2024goose}: A goal-conditioned RL framework. The adversary is modeled as an RL agent that learns to manipulate the control points of Non-Uniform Rational B-Splines (NURBS) to construct safety-critical trajectories.
    
    \item \textbf{SAGE}~\citep{nie2025steerable}: A recent preference alignment framework that fine-tunes motion generation models using pairs of trajectories. It learns to balance adversariality and realism, allowing for test-time steerability via weight interpolation between adversarial and realistic expert models.
\end{itemize}

\vspace{-5pt}

\paragraph{Adversarial training frameworks.}
To demonstrate the effectiveness of our closed-loop training pipeline, we compare \texttt{ADV-0} with the following training paradigms. All methods use the same outer-loop ego policy training structure but differ in how the training environment is generated and how the inner and outer loops are integrated:
\begin{itemize}
    \item \textbf{Replay (w/o Adversary):} The ego agent is trained purely on the original log-replay scenarios from WOMD without any adversarial modification. This serves as a lower bound for performance.
    \item \textbf{Heuristic training:} The ego agent is trained against the Heuristic rule-based generator described above.
    \item \textbf{CAT:} The state-of-the-art closed-loop training framework where the ego agent is trained against the CAT generator. The generator selects adversarial trajectories based on collision probability against the latest ego policy but does not update its policy via preference learning during the training loop.
    \item \textbf{ADV-0 w/o IPL:} An ablation variant of \texttt{ADV-0} where IPL is removed and the adversary is not fine-tuned. It relies solely on energy-based sampling from the pretrained backbone. This isolates the contribution of the evolving adversary.
\end{itemize}
For fair comparison, all adversarial training methods (CAT, \texttt{ADV-0}, etc.) utilize the same curriculum learning schedule regarding the frequency and intensity of adversarial encounters.
Note that we explicitly exclude comparisons with standard adversarial RL frameworks, such as RARL \citep{pinto2017robust,ma2018improved} or observation-based perturbation methods \citep{tessler2019action,zhang2020robust} due to two key factors: 
(1) They primarily focus on perturbations to observations or state vectors. In contrast, ADV-0 targets behavioral robustness by altering the transition dynamics via trajectory generation.  
(2) Standard adversarial RL models the adversary as an agent with a low-dimensional action space. Our setting involves noisy real-world traffic data and the adversary outputs high-dimensional continuous trajectories. 
Training a standard RL adversary from scratch to generate effective trajectories in this noisy environment is computationally intractable and fails to converge in our preliminary experiments.

\vspace{-5pt}
\subsection{Hyperparameters}

We provide the detailed hyperparameters used in our experiments to facilitate reproducibility. Table~\ref{tab:rl_hyperparams_all} lists the parameters for the various RL algorithms used to train the ego agent. Table~\ref{tab:adv0_hyperparams} details the hyperparameters for the \texttt{ADV-0} framework, including the IPL fine-tuning process and the min-max training schedule.

\vspace{-5pt}
\begin{table}[thb]
\centering
\caption{Hyperparameters for different RL algorithms used in the experiments.}
\vspace{-5pt}
\label{tab:rl_hyperparams_all}
\begin{small}
% --- Table 1: TD3 ---
\begin{minipage}[t]{0.3\textwidth}
\centering
\caption{TD3}
% \vspace{-5pt}
\label{tab:hp_td3}
\begin{tabular}{ll}
\toprule
\textbf{Hyper-parameter} & \textbf{Value} \\
\midrule
Discount Factor $\gamma$ & 0.99 \\
Batch Size & 256 \\
Actor Learning Rate & 3e-4 \\
Critic Learning Rate & 3e-4 \\
Target Update $\tau$ & 0.005 \\
Policy Delay & 2 \\
Exploration Noise & 0.1 \\
Policy Noise & 0.2 \\
Noise Clip & 0.5 \\
\bottomrule
\end{tabular}
\end{minipage}
\hfill
% --- Table 2: SAC & SAC-Lag ---
\begin{minipage}[t]{0.3\textwidth}
\centering
\caption{SAC \& SAC-Lag}
% \vspace{-5pt}
\label{tab:hp_sac}
\setlength{\tabcolsep}{3pt}
\begin{tabular}{ll}
\toprule
\textbf{Hyper-parameter} & \textbf{Value} \\
\midrule
Discount Factor $\gamma$ & 0.99 \\
Batch Size & 256 \\
Learning Rate & 3e-4 \\
Target Update $\tau$ & 0.005 \\
Entropy $\alpha$ & 0.2 (Auto) \\
Cost Coefficient & 0.5 \\
\midrule
\textit{SAC-Lag Specific} & \\
Cost Limit & 0.3 \\
Lagrangian LR & 5e-2 \\
\bottomrule
\end{tabular}
\end{minipage}
\hfill
% --- Table 3: PPO, PPO-Lag & GRPO ---
\begin{minipage}[t]{0.34\textwidth}
\centering
\caption{PPO, PPO-Lag \& GRPO}
% \vspace{-5pt}
\label{tab:hp_ppo_grpo}
\begin{tabular}{ll}
\toprule
\textbf{Hyper-parameter} & \textbf{Value} \\
\midrule
Discount Factor $\gamma$ & 0.99 \\
Batch Size & 256 \\
Learning Rate & 3e-5 \\
Update Timestep & 4096 \\
Epochs per Update & 10 \\
Clip Ratio & 0.2 \\
GAE Lambda $\lambda$ & 0.95 \\
Entropy Coefficient & 0.01 \\
Value Coefficient & 0.5 \\
\midrule
\textit{Algorithm Specific} & \\
Cost Limit (PPO-Lag) & 0.4 \\
Lagrangian LR (Lag) & 5e-2 \\
Group Size (GRPO) & 6 \\
KL Beta (GRPO) & 0.001 \\
\bottomrule
\end{tabular}
\end{minipage}
\end{small}
\end{table}

\begin{table}[th]
\caption{Hyperparameters for \texttt{ADV-0} Framework and IPL Fine-tuning.}
% \vspace{-5pt}
\label{tab:adv0_hyperparams}
\centering
\setlength{\tabcolsep}{3pt}
\renewcommand{\arraystretch}{1.1}
\begin{small}
\resizebox{0.7\textwidth}{!}{%
\begin{tabular}{llc}
\toprule
\textbf{Module} & \textbf{Parameter} & \textbf{Value} \\
\midrule
{Backbone (DenseTNT)} & Hidden Size & 128 \\
& Sub-graph Depth & 3 \\
& Global-graph Depth & 1 \\
& Trajectory Modes ($K$) & 32 \\
& NMS Threshold & 7.2  \\
\midrule
{IPL Fine-tuning} & Learning Rate & 5e-6 \\
& Temperature ($\tau$) & 0.05 \\
& Optimizer & AdamW \\
& Scheduler & CosineAnnealing \\
& Gradient Accumulation Steps & 16 \\
& Pairs per Scenario & 8 \\
& Reward Margin & 5.0 \\
& Spatial Diversity Threshold & 2.0 m \\
\midrule
{Min-Max Schedule (RARL)} 
& Adversary Update Frequency & Every 5 Ego Updates \\
& Adversary Training Iterations & 5 Epochs per Block \\
& Adversary Training Batch Size & 32 Scenarios \\
& Adversarial Sampling Temperature & 0.1 \\
& Max Training Timesteps & $1 \times 10^6$ \\
& Opponent Trajectory Candidates & 32 \\
& Ego History Buffer Length  & 5 \\
& Min Probability (Curriculum) & 0.1 \\
\bottomrule
\end{tabular}}
\end{small}
\end{table}

%%%%%%%%%%%%%%%%%%%%%%%%%%%%%%%%%%%%%%%%%%%%%%%%%%%%%%%%%%%%%%

% \clearpage
% \section{Discussions}

% To make \texttt{ADV-0} more intuitive and theoretically grounded, we discuss through the following \textit{``How to Understand"} questions.

% \subsection{How to understand the ego history buffer?}

% \subsection{How to understand the benefit of temperature sampling?}

\section*{Broader Impacts}
This paper presents work whose goal is to advance the field of Machine Learning, specifically in the domain of safety-critical autonomous systems. Our research focuses on identifying system vulnerabilities and improving robustness against rare, long-tail events, which is a prerequisite for the safe large-scale deployment of autonomous vehicles. By providing a rigorous framework for generating and mitigating high-risk scenarios, this work contributes to reducing potential accidents and enhancing public trust in automation technologies. While adversarial generation techniques could theoretically be repurposed to identify vulnerabilities for malicious intent, our framework is designed as a defensive mechanism to patch these flaws before deployment. Beyond this potential dual-use concern, we do not identify specific negative ethical consequences or societal risks associated with this research, as the adversarial generation is strictly confined to simulation environments for system validation and improvement.

\clearpage
\section*{NeurIPS Paper Checklist}

\begin{enumerate}

\item {\bf Claims}
    \item[] Question: Do the main claims made in the abstract and introduction accurately reflect the paper's contributions and scope?
    \item[] Answer: \answerYes{}.
    \item[] Justification: The abstract and \cref{sec:formulation}--\cref{sec:method} state the scope, assumptions, and contributions. The theoretical and empirical support is provided in \cref{ssec:theory} and \cref{sec:theory_appendix}.

\item {\bf Limitations}
    \item[] Question: Does the paper discuss the limitations of the work performed by the authors?
    \item[] Answer: \answerYes{}.
    \item[] Justification: The conclusion discusses limitations, including reliance on high-fidelity simulators for online RL training and the challenge of extending the approach to vision-based inputs.

\item {\bf Theory assumptions and proofs}
    \item[] Question: For each theoretical result, does the paper provide the full set of assumptions and a complete (and correct) proof?
    \item[] Answer: \answerYes{}.
    \item[] Justification: The main paper summarizes the theoretical results in \cref{ssec:theory}. Complete derivations, assumptions, and proofs are provided in \cref{sec:theory_appendix}.

    \item {\bf Experimental result reproducibility}
    \item[] Question: Does the paper fully disclose all the information needed to reproduce the main experimental results of the paper to the extent that it affects the main claims and/or conclusions of the paper (regardless of whether the code and data are provided or not)?
    \item[] Answer: \answerYes{}.
    \item[] Justification: \cref{appendix:setups} describes the dataset, simulator, evaluation protocol, baselines, implementation details, and hyperparameters used for the experiments.

\item {\bf Open access to data and code}
    \item[] Question: Does the paper provide open access to the data and code, with sufficient instructions to faithfully reproduce the main experimental results, as described in supplemental material?
    \item[] Answer: \answerNo{}
    \item[] Justification: The code is not publicly available. \Cref{appendix:setups} describes the dataset, implementation, training procedures, and evaluation protocol.

\item {\bf Experimental setting/details}
    \item[] Question: Does the paper specify all the training and test details (e.g., data splits, hyperparameters, how they were chosen, type of optimizer) necessary to understand the results?
    \item[] Answer: \answerYes{}.
    \item[] Justification: The paper specifies the WOMD scenario splits, MetaDrive environment, benchmark protocol, RL algorithms, and hyperparameters in \cref{appendix:setups}.

\item {\bf Experiment statistical significance}
    \item[] Question: Does the paper report error bars suitably and correctly defined or other appropriate information about the statistical significance of the experiments?
    \item[] Answer: \answerYes{}.
    \item[] Justification: The table captions and \cref{app:evaluation} define standard deviations computed over five seeds and the macro averages of means and standard deviations. \Cref{tab:reactive_ego} reports mean $\pm$ standard deviation over five seeds; \cref{tab:preference_baselines} reports means over five seeds.

\item {\bf Experiments compute resources}
    \item[] Question: For each experiment, does the paper provide sufficient information on the computer resources (type of compute workers, memory, time of execution) needed to reproduce the experiments?
    \item[] Answer: \answerYes{}.
    \item[] Justification: \cref{appendix:setups} states that the experiments are conducted on a computing platform with two NVIDIA RTX A6000 GPUs, and the paper reports the corresponding computational overhead in the experimental analysis.
    
\item {\bf Code of ethics}
    \item[] Question: Does the research conducted in the paper conform, in every respect, with the NeurIPS Code of Ethics \url{https://neurips.cc/public/EthicsGuidelines}?
    \item[] Answer: \answerYes{}.
    \item[] Justification: The work is evaluated in simulation and is framed as a defensive robustness method for autonomous systems. The broader impacts discussion describes safety motivations and possible misuse considerations.


\item {\bf Broader impacts}
    \item[] Question: Does the paper discuss both potential positive societal impacts and negative societal impacts of the work performed?
    \item[] Answer: \answerYes{}.
    \item[] Justification: The broader impacts section describes positive safety implications and notes the possible dual-use risk of adversarial generation, while emphasizing the simulation-only defensive setting.

    
\item {\bf Safeguards}
    \item[] Question: Does the paper describe safeguards that have been put in place for responsible release of data or models that have a high risk for misuse (e.g., pre-trained language models, image generators, or scraped datasets)?
    \item[] Answer: \answerNA{}.
    \item[] Justification: The manuscript does not release a high-risk model or scraped dataset. The method is evaluated within simulation for robustness analysis.

\item {\bf Licenses for existing assets}
    \item[] Question: Are the creators or original owners of assets (e.g., code, data, models), used in the paper, properly credited and are the license and terms of use explicitly mentioned and properly respected?
    \item[] Answer: \answerNo{}.
    \item[] Justification: The paper cites the main existing assets and datasets, such as WOMD and MetaDrive, but the current manuscript does not explicitly list their licenses or terms of use.

\item {\bf New assets}
    \item[] Question: Are new assets introduced in the paper well documented and is the documentation provided alongside the assets?
    \item[] Answer: \answerNA{}.
    \item[] Justification: The paper introduces a method and experimental framework but does not release a new dataset, benchmark, or model asset in the current submission package.

\item {\bf Crowdsourcing and research with human subjects}
    \item[] Question: For crowdsourcing experiments and research with human subjects, does the paper include the full text of instructions given to participants and screenshots, if applicable, as well as details about compensation (if any)? 
    \item[] Answer: \answerNA{}.
    \item[] Justification: The research does not involve crowdsourcing or human-subject experiments.

\item {\bf Institutional review board (IRB) approvals or equivalent for research with human subjects}
    \item[] Question: Does the paper describe potential risks incurred by study participants, whether such risks were disclosed to the subjects, and whether Institutional Review Board (IRB) approvals (or an equivalent approval/review based on the requirements of your country or institution) were obtained?
    \item[] Answer: \answerNA{}.
    \item[] Justification: The research does not involve human subjects, so IRB approval is not applicable.

\item {\bf Declaration of LLM usage}
    \item[] Question: Does the paper describe the usage of LLMs if it is an important, original, or non-standard component of the core methods in this research? Note that if the LLM is used only for writing, editing, or formatting purposes and does \emph{not} impact the core methodology, scientific rigor, or originality of the research, declaration is not required.
    %this research? 
    \item[] Answer: \answerNA{}.
    \item[] Justification: LLMs are not used as an important, original, or non-standard component of the proposed method.

\end{enumerate}

\end{document}